%% file: main.tex
\documentclass{article}

\usepackage{microtype}
\usepackage{graphicx,xcolor, colortbl, tcolorbox,algorithm,algorithmic}
\usepackage{wrapfig}
\usepackage{booktabs,paralist,amsmath,amsthm,amssymb} % 
\usepackage{multirow,multicol}
\usepackage{scalerel, amssymb}
\usepackage{caption, subcaption}
\usepackage{makecell}
\usepackage{tikz}
\usepackage{dsfont}
\usepackage{textcomp}
\usepackage[export]{adjustbox}

\usetikzlibrary{calc}

\usepackage{hyperref}

\usepackage{iclr2022/iclr2022_conference}

\newtheorem{theorem}{Theorem}
\newtheorem{proposition}{Proposition}[section]

\newtheorem{remark}[theorem]{Remark}
\newtheorem{definition}{Definition}[section]

 \definecolor{azure}{rgb}{0.0, 0.5, 1.0}
\newtcolorbox{mybox}[1][]
{
  colframe = #1!25,
  colback  = #1!10,
  coltitle = #1!20!black,  
}

\def\msquare{\mathord{\scalerel*{\Box}{gX}}}

\title{Forward Operator Estimation in Generative Models with Kernel Transfer Operators}

\author{Zhichun Huang \\
Carnegie Mellon University\\
Pittsburgh, PA 15213, USA \\
\texttt{zhichunh@cs.cmu.edu} \\
\And
Rudrasis Chakraborty \\
Butlr \\
Burlingame, CA 94002, USA \\
\texttt{rudra@butlr.io}\\
\And
Vikas Singh \\
University of Wisconsin-Madison\\
Madison, WI 53706, USA\\
\texttt{vsingh@biostat.wisc.edu}\\
}

\iclrfinalcopy
\begin{document}

\maketitle

\input{abstract}
\input{intro}
\input{prelim}

\input{method}

\input{results}

\input{related_works}

\bibliography{ref}
\bibliographystyle{iclr2022/iclr2022_conference.bst}

\input{appendix}
\end{document}

%% file: abstract.tex
\begin{abstract}
% Flow-based generative models refer to deep generative models with 
% tractable likelihoods, and 
% offer several attractive properties including 
% efficient density estimation and sampling. 
% Despite many advantages, 
% current formulations (e.g., normalizing flow) often have an expensive memory/runtime footprint,
% which hinders their use in a number of applications. 
% In this paper, we consider the setting where we have access to an autoencoder, which is
% suitably effective for the dataset of interest. 
% Under some mild conditions,
% we show that we can calculate a mapping to a RKHS which subsequently enables deploying 
% mature ideas from the kernel methods literature for flow-based generative models. 
% Specifically, we can \textit{explicitly} map the RKHS distribution (i.e., 
% approximate the flow) to match or align with  
% a template/well-characterized distribution, via kernel transfer operators. 
% This leads to a direct and resource efficient approximation avoiding iterative optimization. 
% We empirically show that this simple idea yields competitive results on popular datasets such as CelebA,
% as well as promising results on a public 3D brain imaging dataset where the sample sizes are much smaller. 

Generative models which use explicit density modeling (e.g., variational autoencoders, flow-based generative models) involve finding a mapping from a known distribution, e.g. Gaussian, to the unknown input distribution. This often requires searching over a class of non-linear functions (e.g., representable by a deep neural network). While effective in practice, the associated runtime/memory costs can increase rapidly, usually as a function of the performance desired in an application. We propose a much cheaper (and simpler) strategy 
to estimate this mapping based on adapting known results in kernel transfer operators. We show that our formulation enables highly efficient distribution approximation and sampling, and offers surprisingly good empirical performance that compares favorably with powerful baselines, but with significant runtime savings. We show that the algorithm also performs well in small sample size settings (in brain imaging). 

%of arbitrarily shaped distributions in $\mathbf{R}^n$ and $\mathbf{S}^{n - 1}$. We propose to solve the problem of learning the non-linear density transfer operator in the input space by alternatively computing the {\it closed-form} optimal linear operator in the reproducing-kernel Hilbert space (RKHS). Through a comprehensive empirical evaluation, we further show that the proposed method yields comparable or better image generation samples on several computer vision datasets compared with popular VAE variants and other density estimation methods, while having a significant advantage in learning time. {\color{red} WRITE ONE SENTENCE ABOUT MEDICAL EXPTS}
\end{abstract}

%% file: intro.tex
\newcommand{\var}[1]{\uppercase{#1}}
\newcommand{\spc}[1]{\mathcal{#1}}
\makeatletter
\newcommand*{\rsimdots}{%
  \mathrel{%
    \mathpalette\@rsimdots{}% Adopt to math style size via \mathpalette
  }%
}
\newcommand*{\@rsimdots}[2]{%
  % #1: math style
  % #2: unused
  \sbox0{$#1\sim\m@th$}%
  \sbox2{$#1\vcenter{}$}% \ht2 is height of the math axis
  \dimen@=.75\ht2\relax % distance dot to math axis
  \sbox2{$#1\cdot\m@th$}% single vertically centered dot
  \sbox2{% two dots above and below the math axis
    \rlap{\raisebox{\dimen@}{\copy2}}%
    \raisebox{-\dimen@}{\copy2}%
  }%
  % Rotate the two dots.
  \sbox2{$#1\rotatebox[origin=c]{-45}{\copy2}$}%
  % Combine symbol
  \rlap{\hbox to \wd0{\hss\copy2\hss}}%
  \copy0 %
}
\makeatother

\section{Introduction}
Generative modeling, in its unconditional form, refers to the problem of estimating the data generating distribution: given \textit{i.i.d.} samples $\mathbf{X}$ with an unknown distribution $P_{\var{x}}$, a generative model seeks to 
find a parametric distribution that closely resembles $P_{\var{x}}$. In modern deep generative models, 
we often approach this problem via a {\em latent variable} -- i.e., we assume that there is some variable $\var{z} \in \spc{Z}$ 
associated with the observed data $\var{x} \in \spc{X}$ that 

follows a {\em known} distribution $P_{\var{z}}$ (also referred to as the \textit{prior} in generative models). Thus, we can learn a mapping $f: \spc{Z} \to \spc{X}$ such that the distribution after transformation, denoted by $P_{f(\var{z})}$, aligns well with the data generating distribution $P_{\var{x}}$. Therefore, sampling from $P_{\var{X}}$ becomes convenient since $P_{\var{Z}}$ can be efficiently sampled. Frequently, $f$ is parameterized by deep neural networks and optimized with stochastic gradient descent (SGD).
% While the mechanics of how each component in the architecture is implemented varies 
% from one model to the other, if the overall workflow functions as intended, it is quite convenient and effective,  
% because drawing samples from 
% the known density $p_{\theta}$
% is much easier.

Existing generative modeling methods variously optimize the transformation $f$, most commonly modeling it as a Maximum Likelihood Estimation (MLE) or distribution matching problem. For instance, given data $\mathbf{X} = \{ \mathbf{x}_{1}, \dots, \mathbf{x}_{n}\}$, a variational autoencoder (VAE) \citep{kingma2013auto} first constructs $\var{Z}$ through the approximate posterior $q_{\var{z}|\var{X}}$ and maximizes a lower bound of likelihood $p_{f(\var{Z})}(\mathbf{X})$. Generative adversarial networks (GANs) \citep{goodfellow2014gan} relies on a simultaneously learned discriminator such that samples of $P_{f(\var{Z})}$ are indistinguishable from $\mathbf{X}$. Results in \citep{arjovsky2017wasserstein, li2017mmdgan} suggest that GANs minimize the distributional discrepancies between $P_{f{(\var{Z})}}$ and $P_{\var{X}}$. Flow-based generative models optimize $p_{f(\var{Z})}(\mathbf{X})$ explicitly through the {\it change of variable rule} and efficiently calculating the Jacobian determinant of the inverse mapping $f^{-1}$. 

% Consider generative models with an explicit decoder or generator structure, for instance, generative adversarial networks (GANs) \citep{goodfellow2014gan} or variational autoencoders (VAEs) \citep{kingma2013auto}. The parameterized empirical density $p_\tau(\var{X})$ can be written as $\int p_\tau(\var{x}|\var{z})p(\var{z}) dz$ where $\var{z}$ is the latent variable with a suitable prior and the conditional density $p_\tau(\var{x}|\var{z})$ is usually modeled with a multi-layer perceptron or a convolutional neural network. Such a decoder shapes (or non-linearly transforms) a simple prior distribution $p_z$ into the distribution $p_{\var{x}}$. Training a generative model of this form involves numerical optimization of an appropriate loss function, e.g., associated with the likelihood \citep{kingma2013auto}.
% % During training, we use a decoder distribution 
% % $q_{\tau'}(z|x)$ and an encoder distribution 
% % $p_\tau(x|z)$ to train the parameters $\tau$ and $\tau'$ by maximizing the lower bound of likelihood.
% In flow based generative models, we learn a parametric invertible mapping $f_{\tau}$ by directly maximizing the log likelihood of $p_{\tau}(\var{x})$ using the \textit{change of variable} rule. The conditional density, $p_{\tau}(\var{x}|\var{z})$, therefore, is characterized by the inverse mapping $f_{\tau}^{-1}$. 
% The ``flow'' nomenclature is used because of the transformation of a (standard normal) distribution to the empirical distribution 
% of the samples composed by invertible transformations.
%
In all examples above, the architecture or objective notwithstanding, the common goal is to find a suitable function $f$ that reduces the difference between $P_{f(\var{Z})}$ and $P_{\var{X}}$. 
%
% \revise{Eric: take 1 paragraph...use a fig below, describe the mechanics of VAE. Don't yet "call" it a forward operator yet.}
%
% \revise{Eric: take 1 paragraph...copy/paste from previous submission, describe the mechanics of FLOW/GLOW. Don't yet call it a forward operator yet.}
%
% \revise{Now, draw attention to figure and the two previous paragraphs..and say that density approximation is fundamentally dependennt on forward operator....talk about yes, it works..but it is a lot of work. Copy/paste lines from difficulties in GLOW, yadda yadda...}
%
%G a generative learning algorithm 
Thus, a key component in many deep generative models 
is to learn a \textit{forward operator} as defined below.
\begin{definition}[Forward operator]
\label{def:forward}
A forward operator $f^\star \in \mathcal{C}: \spc{Z} \to \spc{X}$ is defined to be a mapping associated with some latent variable ${\var{z}} \sim P_{\var{z}}$ such that $f^\star = \arg \min_{f \in \mathcal{C}} d(P_{f(\var{z})}, P_{\var{x}})$ for some function class $\mathcal{C}$ and a distance measure $d(\cdot,\cdot)$.
\end{definition}

%One benefit is that 
%the log-determinant of the Jacobian for the %sequence of invertible transformations can be %efficiently computed or approximated. This %enables optimization of the maximum likelihood %as well as efficient density estimation and %sampling.

%\begin{comment}
%\revise{Eric: you can use this more %holistic view here..}
%In the generative model setting, $f^*$ %is sometimes referred to as the %\textit{optimal generator}. And in %practice, $p_\var{z}$ is often chosen %to be certain known parametric %distribution (e.g. Gaussian), and %therefore samples of $p_{\var{x}}$ can %be approximately generated by applying %$f^*$ to samples of $p_\var{z}$, which %can be easily produced. 
%\end{comment}

%  The state-of-the-art generative models can be broadly categorized into three sub-classes: \begin{inparaenum}[(a)] \item Generative adversarial networks (GANs) \citeppending{}, \item Variational autoencoders (VAE) \citeppending{}, \item Flow-based generative models \citeppending{}. \end{inparaenum} GANs  adopt a generator-critic scheme and use a surrogate loss to enforce $f(Z)$ to be \textit{in-distribution}. On the other hand, VAEs \citeppending{} pose generation as a variational inference problem and optimize the lower bound of the marginal likelihood $p_f(\var{x})$. Flow-based generative models \citeppending{} matches the posterior to the data distribution through explicitly maximization of $p_f(\var{x})$ using the \textit{change of variables}. 

{\bf Motivation:} The specifics 
of the forward operator may differ from case to case. But its properties and how it is estimated 
numerically greatly influences the empirical performance of the model. 
For instance, mode collapse issues in GANs are well known and solutions continue to emerge \citep{srivastava2017veegan}. To learn the forward operator, VAEs use an approximate posterior $q_{\var{Z}|\var{X}}$ that may 
sometimes fail to align with the prior \citep{Kingma2016improved, dai2018diagnosing}. 
%and therefore creating a discrepancy between the learned operator (from %approximated posterior) and the target operator (from the prior). 
Flow-based generative models enable direct access to the posterior likelihood, yet in order to tractably evaluate the Jacobian of the transformation during training, one must either restrict the expressiveness at each layer \citep{dinh2017RealNVP, Kingma2018Glow} or use more involved solutions \citep{Chen2018neuralode}. 
Of course, solutions 
to mitigate these weaknesses \citep{ho2019flow++} remains an active area of research.

%Based on these known pitfalls, we can summarize that a better learning algorithm %for 
The starting point of our work is to evaluate the extent to which 
we can {\em radically} simplify the forward operator in deep generative models. Consider some desirable properties of a hypothetical forward operator (in Def. \eqref{def:forward}): \begin{inparaenum}[\bfseries (a)]\item Upon  convergence, the learned operator $f^\star$ minimizes the distance between $P_{\var{x}}$ and $P_{f(\var{z})}$ over all possible operators of a certain class. \item The training directly learns the mapping from the prior distribution $P_{\var{Z}}$, rather than a variational approximation. \item The forward operator $f^\star$ can be efficiently learned and sample generation is also efficient.  
\end{inparaenum} 
It would appear that 
these criteria violate the ``no free lunch rule'', and 
some compromise must be involved. Our goal is 
to investigate this trade-off: 
which design 
choices can make 
this approach work? 
Specifically, a well studied construct in dynamical systems, namely the Perron-Frobenius operator \citep{lemmens2012nonlinear}, suggests an alternative \textit{linear} route to model the forward operator. Here, 
we show that 
if we are willing to give up on a few features  
in existing models -- this may be acceptable depending on the downstream use case --  
then, the forward operator in generative models can be efficiently approximated as the estimation of a {\it closed-form linear operator} in the reproducing kernel Hilbert space (RKHS). With simple adjustments of existing results, we identify a novel way to replace the expensive training for generative tasks with a simple principled kernel approach.
% Surprisingly, with only minor adjustments, a direct instantiation of this idea works. 

\textbf{Contributions.} Our results are largely based on results in kernel methods and dynamical systems, but we demonstrate their relevance in generative modeling and complement recent ideas that emphasize links between deep generative models and dynamical systems. Our contributions are 
\begin{inparaenum}[\bfseries (a)] \item We propose a \textit{non-parametric} method for transferring a known prior density {\it linearly} in RKHS to an unknown data density -- equivalent to learning a nonlinear forward operator in the input space. When compared to its functionally-analogous module used in other deep generative methods, our method 
avoids multiple expensive training steps yielding significant efficiency gains; \item We evaluate this idea in multiple scenarios and show competitive generation performance and efficiency benefits with pre-trained autoencoders on popular image datasets including MNIST, CIFAR-10, CelebA and FFHQ; \item As a special use case, we demonstrate the advantages over other methods in limited data settings. 
\end{inparaenum}

%% file: prelim.tex
\section{Preliminaries}

We briefly introduce reproducing kernel Hilbert space (RKHS) and kernel embedding of probability distributions, concepts 
we will use frequently. 

%which are the building blocks of this paper.
%\edit{We can define feature maps earlier here}
\begin{definition}[RKHS \citep{aronszajn1950theory}]
\label{def:rkhs}
For a set $\mathcal{X}$, let $\mathcal{H}$ be a set of functions $g:\mathcal{X}\rightarrow \mathbf{R}$. Then, $\mathcal{H}$ is a reproducing kernel Hilbert space (RKHS) with a product $\langle \cdot, \cdot \rangle_{\mathcal{H}}$ if there exists a function $k:\mathcal{X}\times \mathcal{X}\rightarrow \mathbf{R}$ (called a reproducing kernel) such that \begin{inparaenum}[(i)] \item $\forall x \in \mathcal{X}, g \in \mathcal{H}, g(x) = \langle g, k(x, \cdot)\rangle_{\mathcal{H}}$; \item $\mathcal{H} = cl({\text{span}}(\left\{k(x, \cdot), x \in \mathcal{X}\right\}))$, where $cl(\cdot)$ is the set closure. \end{inparaenum} 
\end{definition}

\begin{wraptable}{r}{0.45\textwidth}
\centering % to have the caption near the table
{\small
\begin{tabular}{p{.45cm} p{.45cm} p{3.65cm} }
\specialrule{1pt}{1pt}{0pt}
\rowcolor{azure!20}
\multicolumn{2}{c}{Notations} &  Meaning \\ \specialrule{1pt}{0pt}{1pt}
$\var{z}$ & $\var{x}$  & Random variable\\
$\mathbf{Z}$ & $\mathbf{X}$ & Data samples\\
$\spc{Z}$ & $\spc{X}$  & Domain \\
$P_{\var{z}}$ & $P_{\var{x}}$ & Distribution\\
$p_{\var{z}}$ & $p_{\var{x}}$  & Density function \\
$k$ & $l$  & Kernel function \\
$\mathcal{H}$ & $\mathcal{G}$  & RKHS\\
$\phi(\cdot)$ & $\psi(\cdot)$  & Feature mapping\\
$\mathcal{E}_k$ & $\mathcal{E}_l$  & Mean embedding operator\\
$\mu_{\var{z}}$ & $\mu_{\var{x}}$  & Kernel mean embedding\\
\bottomrule
\end{tabular}
\vspace{-1em}
\caption{\label{tab:notations} Notations  used in this paper.}
}
\end{wraptable}

The function $\phi(x) = k(x, \cdot) : \mathcal{X} \to \mathcal{H}$ is referred to as the \textit{feature mapping} of the induced RKHS $\mathcal{H}$. A useful identity derived from feature mappings is the {\it kernel mean embedding}: it defines a mapping from a probablity measure in $\mathcal{X}$ to an element in the RKHS.
% The use of kernel mean embeddings has recently regained much attention in the machine learning community for tasks associated with distribution matching.

\begin{definition}[Kernel Mean Embedding \citep{smola2007ahilbertspace}]
\label{def:me}
Given a probability measure $p$ on $\mathcal{X}$ with an associated RKHS $\mathcal{H}$ equipped with
a reproducing kernel $k$ such that $\sup_{x \in \mathcal{X}} k(x, x) < \infty$, the kernel mean embedding of $p$ in RKHS $\mathcal{H}$, denoted by $\mu_{p} \in \mathcal{H}$, is defined as $\mu_{p} = E_p[\phi(x)]= \int k(x, \cdot) p(x)dx$, and the mean embedding operator $\mathcal{E}: L^1(\mathcal{X}) \to \mathcal{H}$ is defined as $\mu_p = \mathcal{E}p$. 
\end{definition}

\begin{tcolorbox}[bottom=0mm,top=0mm]
\begin{remark}
For characteristic kernels, the operator $\mathcal{E}$ is injective. Thus, two distributions $(p, q)$ in $\mathcal{X}$ are identical \textit{iff} $\mathcal{E}p = \mathcal{E}q$.
\end{remark}
\end{tcolorbox}
 This property allows using of  
{\it Maximum Mean Discrepancy (MMD)}
for distribution matching \citep{gretton2012kernel, li2017mmdgan} and is 
common, see \citep{Muandet2017kernel,zhou2018pnas}. For a finite number of
samples $\left\{\mathbf{x}_i\right\}_{i=1}^n$ drawn from the probability measure $p$,
an unbiased empirical estimate of $\mu_{\mathcal{H}}$ is $\hat{\mu}_{\mathcal{H}} = \tfrac{1}{n} \sum_{i=1}^{n}k(\mathbf{x}_i, \cdot)$ such that $\lim_{n \to \infty} \tfrac{1}{n} \sum_{i=1}^{n}k(\mathbf{x}_i, \cdot) = \mu_{\mathcal{H}}$.

Next, we review the covariance/cross covariance operators, two widely-used identities in kernel methods \citep{fukumizu2013kernel,song2013kernel} and building blocks of our approach.
\begin{definition}[Covariance/Cross-covariance Operator]
    Let $X, Z$ be random variables defined on $\mathcal{X} \times \mathcal{Z}$ with joint distribution $P_{X, Z}$ and marginal distributions $P_{X}$, $P_{Z}$. Let $(l, \phi, \mathcal{H})$ and $(k, \psi, \mathcal{G})$ be two sets of (a) bounded kernel, (b) their corresponding feature map, and (c) their induced RKHS, respectively. The (uncentered) covariance operator $\mathcal{C}_{ZZ}: \mathcal{H} \to \mathcal{H}$ and cross-covariance operator $\mathcal{C}_{XZ}: \mathcal{H} \to \mathcal{G}$ are defined as 
    \begin{equation}
        \mathcal{C}_{ZZ} \triangleq \mathbb{E}_{z \sim P_Z}[\phi(z) \otimes \phi(z)]\qquad \mathcal{C}_{XZ} \triangleq \mathbb{E}_{(x, z) \sim P_{X, Z}}[\psi(x) \otimes \phi(z)]
    \end{equation}
    where $\otimes$ is the outer product operator.
\end{definition}

%% file: method.tex
\section{Simplifying the estimation of the forward operator}

\paragraph{Forward operator as a dynamical system:} 
The dynamical system view of generative models has been described by others \citep{Chen2018neuralode, grathwohl2018scalable, behrmann2019invertible}. 
These strategies model the evolution of latent variables in a residual neural network in 
terms of its dynamics over continuous or discrete time $t$, and consider the output function $f$ as the evaluation function at a predetermined boundary condition $t = t_1$. Specifically, given an input (i.e., initial condition) $z(t_0)$, $f$ is defined as
\begin{equation}
    f(z(t_0)) = z(t_0) + \int_{t_0}^{t_1} \Delta_t(z(t)) dt
\end{equation}
where $\Delta_t$ is a time-dependent neural network function and $z(t)$ is the intermediate solution at $t$. This view of generative models is not limited to specific methods or model archetypes, but generally useful, for example, by viewing the outputs of each hidden layer as evaluations in discrete-time dynamics. After applying $f$ on a random variable $Z$, the marginal density of the output over any subspace $\Lambda \subseteq \mathcal{X}$ can be expressed as
\begin{align}
\label{eq0}
\int_\Lambda p_{f(Z)}(x) dx = \int_{z \in {f}^{-1}(\Lambda)} p_{\var{z}}(z)dz
\end{align}
If there exists some neural network instance $\Delta^\star_t$ such that the corresponding output function $f^\star$ satisfies $P_X = P_{f^\star(Z)}$, by Def. \ref{def:forward}, $f^\star$ is a forward operator. Let $\mathbf{X}$ be a set of \textit{i.i.d.} samples drawn from $P_X$. In typical generative learning, either maximizing the likelihood $\tfrac{1}{|\mathbf{X}|}\sum_{x \in \mathbf{X}} p_{f(Z)}(x)$ or minimizing the distributional divergence $d(P_{f(Z)}, P_{\mathbf{X}})$ requires evaluating and differentiating through $f$ or $f^{-1}$ many times.

\paragraph{Towards a one-step estimation of forward operator:} Since $f$ and ${f}^{-1}$ in \eqref{eq0} will be highly nonlinear in practice, evaluating and computing the gradients can be expensive. Nevertheless, the dynamical systems literature suggests a {\em linear} extension of $f^*$, namely the \textit{Perron-Frobenius} operator or transfer operator, that conveniently transfers $p_{\var{z}}$ to $p_{\var{x}}$.
\begin{definition}[Perron-Frobenius operator \citep{mayer1980ruelle}]
\label{def:pf}
Given a dynamical system $f: \spc{X} \to \spc{X}$, the Perron-Frobenius (PF) operator $\mathcal{P}: L^1(\mathcal{X}) \rightarrow L^1(\mathcal{X})$ is an \textit{infinite-dimensional linear} operator defined as $\int_\Lambda (\mathcal{P}p_{\var{z}})(x) dx = \int_{z \in {f^{-1}(\Lambda)}}p_{\var{z}}(z)dz$ for all $\Lambda \subseteq \mathcal{X}.$
\end{definition}

Although in Def. \ref{def:pf}, the PF operator $\mathcal{P}$ is defined for self-maps, it is trivial to extend $\mathcal{P}$ to mappings $f: \mathcal{Z} \to \mathcal{X}$ by restricting the RHS integral $\int_{z \in {f^{-1}(\Lambda)}}p_{\var{z}}(z)dz$ to $\mathcal{Z}$.

It can be seen that, for the forward operator $f^*$,  the corresponding PF operator $\mathcal{P}$ satisfies
\begin{align}
\label{eq1}
p_{\var{x}} = \mathcal{P}p_{\var{z}}. 
\end{align}
If $\mathcal{P}$ can be efficiently computed, transferring the tractable density $p_{\var{z}}$ to the target density $p_{\var{x}}$ can be accomplished simply by applying $\mathcal{P}$. 
%\revise{if we have figs for GLOW/VAE etc, refer to them %here}. 
However, since $\mathcal{P}$ is an infinite-dimensional operator on $L^1(\mathcal{X})$, it is impractical to instantiate it explicitly and exactly. 
Nonetheless, there exist several methods for estimating the Perron-Frobenius operator, including Ulam's method \citep{ulam1960collection} and the Extended Dynamical Mode Decomposition (EDMD) \citep{williams2015data}.
Both strategies project $\mathcal{P}$ onto a finite number of hand-crafted basis functions -- this may suffice 
in many settings but may fall short in modeling highly complex dynamics. 

% Thus, a natural solution is to search for a vector space of infinite dimension to model any complex dynamics. Next, we extend the PF operator in RKHS, denoted by embedded PF operator which will be expressive enough for complex dynamics. 

\paragraph{Kernel-embedded form of PF operator:} A natural 
extension of PF operator 
%over the previously mentioned methods 
is to represent $\mathcal{P}$ by an infinite set of functions \citep{klus2020eigendecompositions}, e.g., projecting it onto the bases of an RKHS via the \textit{kernel trick}. There, for a characteristic kernel $l$, the $\textit{kernel mean embedding}$ uniquely identifies an element $\mu_{\var{x}} = \mathcal{E}_{l}p_{\var{x}} \in \spc{G}$ for any $p_{\var{x}} \in L^1(\spc{X})$. Thus, to approximate $\mathcal{P}$, we may alternatively solve for the dynamics from $p_{\var{z}}$ to $p_{\var{x}}$ in their {\em embedded} form. Using Tab. \ref{tab:notations} notations, we have the following linear operator that defines the dynamics between two embedded densities.

% Let $(k, \mathcal{H}, \mathcal{E}_{k})$ and $(l, \mathcal{G}, \mathcal{E}_{l})$ be two sets of positive definite kernels, their induced RKHS and the corresponding mean embedding operator (as defined in Tab. \ref{tab:notations}). We can apply $\mathcal{E}_{l}$ on both sides on \eqref{eq1} to get
% \begin{align}
% \label{eq2}
% \mathcal{E}_{l}p_{\var{x}} = \mathcal{E}_{l}\mathcal{P}p_{\var{z}}.
% \end{align}

\begin{definition}[Kernel-embedded Perron-Frobenius operator \citep{klus2020eigendecompositions}]
\label{def:kpf}Given $p_{\var{z}} \in L^1(\mathcal{X})$ and $p_{\var{x}} \in L^1(\mathcal{X})$. Denote $k$ as the \textbf{input kernel} and $l$ as the \textbf{output kernel}. Let $\mu_{\var{x}} = \mathcal{E}_l p_{\var{x}}$ and $\mu_{\var{z}} = \mathcal{E}_k p_{\var{z}}$ be their corresponding mean kernel embeddings. The kernel-embedded Perron-Frobenius (kPF) operator, denoted
by $\mathcal{P}_{\mathcal{E}}:\mathcal{H}\rightarrow \mathcal{G}$, is defined as
%\boxed{
\begin{align}
    \mathcal{P}_\mathcal{E} = \mathcal{C}_{\var{x}\var{z}}\mathcal{C}_{\var{z}\var{z}}^{-1}
\end{align}
\end{definition}

\begin{proposition}[\citet{song2013kernel}] \label{prop:kpf}
With the above definition, $\mathcal{P}_\mathcal{E}$ satisfies
\begin{align}
   \mu_{\var{x}} = \mathcal{P}_\mathcal{E}\mu_{\var{z}}
\end{align}
under the conditions:   \begin{inparaenum}[\bfseries (i)] \item $\mathcal{C}_{\var{z}\var{z}}$ is injective \item $\mu_{t} \in \text{range}(\mathcal{C}_{\var{z}\var{z}})$ \item $\mathbb{E}[g(\var{x})|\var{z} = \cdot] \in \mathcal{H}$ for any $g \in G$.\end{inparaenum}
%}
\end{proposition}
The last two assumptions can sometimes be difficult to satisfy for certain RKHS (see Theorem 2 of \citet{fukumizu2013kernel}).  In such cases, a relaxed solution can be constructed by replacing $\mathcal{C}^{-1}_{ZZ}$ by a regularized inverse $(\mathcal{C}_{ZZ} + \lambda I)^{-1}$ or a Moore-Penrose pseudoinverse $\mathcal{C}^\dagger_{ZZ}$.
% \begin{tcolorbox}
% \begin{remark}
% Note that $\mathcal{P}_\mathcal{E}$ here essentially has the same form as the \textit{conditional mean embedding} $\mathcal{U}_{\var{x}|\var{z}}$ in \citep{song2013kernel}.
% \end{remark} 
% \end{tcolorbox}
% While the first assumption in Prop. \ref{prop:kpf} can be easily satisfied with topological $\spc{X}$, continuous $k$, and fully supported $p_{\var{z}}$, the last two assumptions often need to be relaxed by taking the regularized inverse $(C_{\var{z}\var{z}} + \lambda n I)^{-1}$ (see Theorem 2 of \citep{fukumizu2013kernel}).

\begin{figure}[t]
    \centering
    \vspace{-1em}
    \includegraphics[width=0.65\textwidth, clip, trim={0 2cm 5cm 0}]{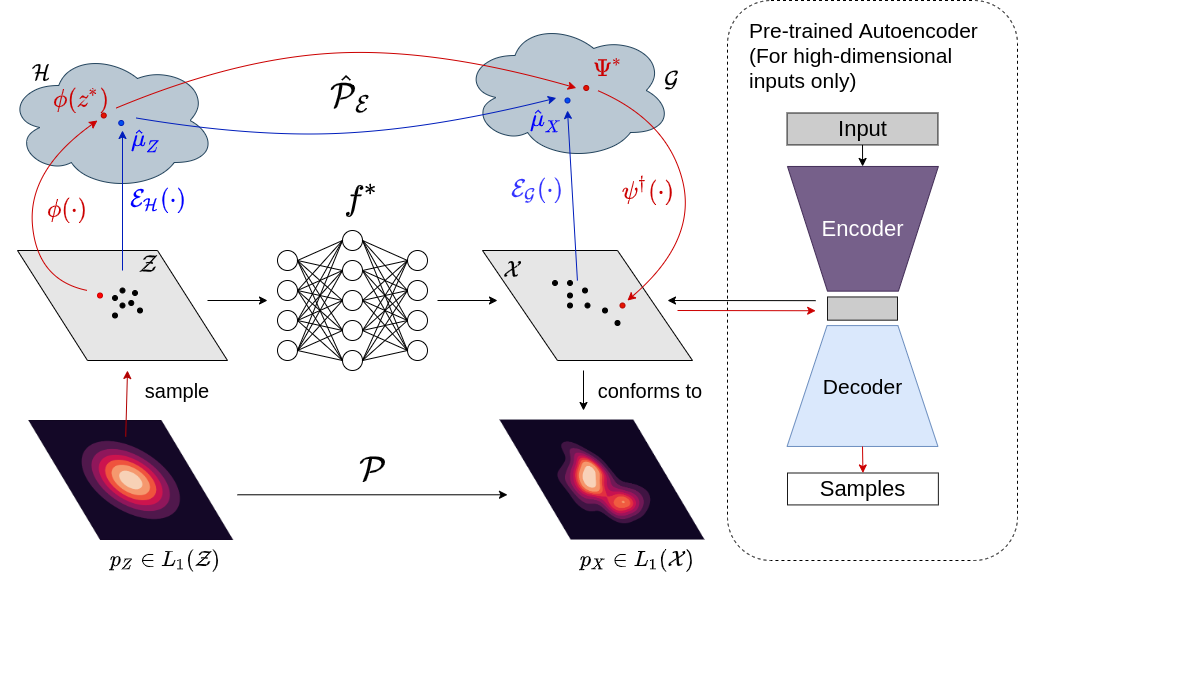}
    \vspace{-8pt}
    \caption{Summary of our framework. The proposed operator is estimated to transfer RKHS embedded densities from one to anothter (blue arrows). Generating new samples (red arrows) involves embedding the prior samples $z^* \sim P_Z$, applying the operator $\Psi^* = \hat{\mathcal{P}}_\mathcal{E}\phi(z^*)$, and finding preimages $\psi^\dagger(\Psi^*)$. The pre-trained autoencoder projects the data onto a \textit{smooth} latent space and is only required when generating high-dimensional data such as images.}
    % \vspace{-1em}
    \label{fig:kpf_workflow}
\end{figure}

The following proposition shows commutativity between the (kernel-embedded) PF operator and the mean embedding operator, showing its equivalence to $\mathcal{P}$ when $l$ is characteristic.

\begin{proposition}[\citep{klus2020eigendecompositions}]
  \label{prop:klus}
   With the above notations, $\mathcal{E}_l \circ \mathcal{P} = \mathcal{P}_{\mathcal{E}} \circ \mathcal{E}_k$.
\end{proposition}

\paragraph{Transferring embedded densities with the kPF operator:} The kPF operator is a powerful tool that allows transferring embedded densities in RKHS. The main steps are:
\begin{tcolorbox}[top=0mm,bottom=0mm]
\begin{compactenum}[\bfseries (1)] \item Use mean embedding operator $\mathcal{E}_{l}$ on $p_{\var{z}}$. Let us denote it by $\mu_{\var{z}}$.  \item  Transfer $\mu_{\var{z}}$ using kPF operator $\mathcal{P}_\mathcal{E}$ to get the mean embedded $p_{\var{x}}$, given by $\mu_{\var{x}}$.\end{compactenum}
\end{tcolorbox}
Of course, in practice with finite data, $\{\mathbf{x}_i\}_{i \in [n]} \sim P_{\var{x}}$ and $\{\mathbf{z}_i\}_{i \in [n]} \sim P_{\var{x}}$, $\mathcal{P}_\mathcal{E}$ must be estimated empirically (see \citet{klus2020eigendecompositions} for an error analysis).
\begin{align}
\hat{\mathcal{P}}_\mathcal{E} = \hat{\mathcal{C}}_{\var{x}\var{z}}(\hat{\mathcal{C}}_{\var{z}\var{z}})^{-1} \nonumber \approx \Psi(\Phi^T\Phi + \lambda n I)^{-1}\Phi^T \approx \Psi(\Phi^T\Phi)^\dagger\Phi^T
\end{align}
where $\Phi = [k(\mathbf{z}_1, \cdot), \cdots, k(\mathbf{z}_n, \cdot)],~
 \Psi = [l(\mathbf{x}_1,\cdot), \cdots, l(\mathbf{x}_n, \cdot)]$ are simply the corresponding feature matrices for samples of $P_{\var{x}}$ and $P_{\var{z}}$, and $\lambda$ is a small penalty term.

\paragraph{Learning kPF for unconditional generative modeling:} Some generative modeling methods such as VAEs and 
flow-based formulations explicitly model the latent variable $Z$ as conditionally dependent on the data variable $X$. This allows deriving/optimizing the likelihood $p_{f(Z)}(X)$. 
This is desirable but may not be essential in all 
applications. 
To learn a kPF, however, $X$ and $Z$ can be independent RVs. While it may not be immediately obvious why we could assume this independence, we can observe the following property for the empirical kPF operator, assuming that the empirical covariance operator $\hat{\mathcal{C}}_{ZZ}$ is non-singular:
%
% In unconditional generative modeling, often we assume a joint distribution of the data variable $X$ and a latent variable $Z$, such that the marginal $p_{Z}$ conforms to a known simple distribution (e.g., Gaussian). However, since the joint distribution $p(X, Z)$ is typically unknown, sampling jointly in general is not possible. Existing generative methods that learn with explicit densities often use additional mechanisms to capture the joint (e.g. approximate posterior in VAE). Nevertheless, to learn a kPF operator, we may simply create an known-distributed latent variable $Z$ that contains no information of $X$, in other words, $X$ and $Z$ are independent RVs. While it may not be immediately obvious why we could assume this independence, we have the following property for the empirical kPF operator assuming that the empirical covariance operator $\hat{\mathcal{C}}_{ZZ}$ is non-singular:
%
\begin{align}\label{eq:emp_mean_emb}
    \hat{\mathcal{P}}_{\mathcal{E}}\hat{\mu}_{Z} &= \hat{\mathcal{C}}_{XZ}\hat{\mathcal{C}}_{ZZ}^{-1}\hat{\mu}_{Z} = \underbrace{\Psi \Phi^\top}_{\hat{\mathcal{C}}_{XZ}} (\underbrace{\Phi \Phi^\top}_{\hat{\mathcal{C}}_{ZZ}})^{-1} \Phi \mathds{1}_{n}
    = {\Psi (\Phi^\top \Phi)^{-1} \Phi^\top \Phi} \mathds{1}_{n} = \Psi \mathds{1}_{n} = \hat{\mu}_{X}
\end{align}
Suppose that $\{\mathbf{x}_i\}_{i \in [n]}$ and $\{\mathbf{z}_j\}_{i \in [n]}$ are independently sampled from the marginals $P_X$ and $P_Z$. It is easy to verify that (\ref{eq:emp_mean_emb}) holds for any pairing $\{(\mathbf{x}_i, \mathbf{z}_j)\}_{(i, j) \in [n] \times [n]}$. However, instantiating the RVs in this way rules out the use of kPF for certain downstream tasks such as controlled generation or mode detection, since $\var{Z}$ does not contain information regarding $\var{X}$. Nevertheless, if sampling is our only goal, then this instantiation of kPF will suffice.
\paragraph{Mapping $Z$ to $\mathcal{G}$:} Now, since  $\mathcal{P_\mathcal{E}}$ is a {\em deterministic} linear operator, we can easily set up a scheme to map samples of $\var{z}$ to elements of $\mathcal{G}$ where the expectation of the mapped samples equals $\mu_{\var{x}}$

% Now, we can easily set up a strategy to map samples of $\var{z}$ to elements of $\mathcal{G}$ by utilizing the fact that $\mathcal{P_\mathcal{E}}$ is a deterministic linear operator and show that the mapped density asymptotically converges to $p_{\var{x}}$. 

Define $\phi(z) = k(z, \cdot)$ and $\psi(x) = l(x, \cdot)$ as feature maps of kernels $k$ and $l$. We can rewrite $\mu_{\var{x}}$  as 
\begin{align}
{{\mu}_{\var{x}}} = \mathcal{P}_\mathcal{E}\mathcal{E}_k {p_{\var{z}}}  
                = \mathcal{P}_\mathcal{E}E_{\var{z}}[\phi(\var{z})]
                = E_{\var{z}}[\mathcal{P}_\mathcal{E}(\phi(\var{Z}))]
                = E_{\var{Z}}[\psi\left(\psi^{-1}\left(\mathcal{P}_\mathcal{E} \phi\left(\var{z}\right)\right)\right)]
\label{eq:approx_sample}
\end{align}
Here $\psi^{-1}$ is the inverse or the \textit{preimage map} of $\psi$. Such an inverse, in general, may not exist \citep{kwok2004pre,honeine2011preimage}. We will discuss a procedure to approximate $\psi^{-1}$ in \S\ref{sec:preimage}. In what follows, we will temporarily assume that an exact preimage map exists and is tractable to compute. 

Define $\Psi^* = \hat{\mathcal{P}}_\mathcal{E} \phi(Z)$ as the \textit{transferred sample} in $\mathcal{G}$ using the empirical embedded PF operator $\hat{\mathcal{P}}_\mathcal{E}$. Then the next result shows that asymptotically the transferred samples converge in distribution to the target distribution. 

\begin{proposition}
  \label{prop:convergence} As $n \to \infty$, $\psi^{-1}\left(\Psi^*\right) \stackrel{\mathcal{D}}{\to} P_{\var{x}}$. That is, the preimage of the transferred sample approximately conforms to  $P_{\var{X}}$ under previous assumptions when $n$ is large.
\end{proposition}
\begin{proof}
Since $\hat{\mathcal{P}}_\mathcal{E} \overset{\text{asymp.}}{\to} \mathcal{P}$, the proof immediately follows from \eqref{eq:approx_sample}.
\end{proof}

\section{Sample generation using the Kernel transfer operator}\label{sec:gen_sample}
At this point, the transferred sample  $\Psi^*$, obtained by the kPF operator, remains an element of RKHS $\spc{G}$. To translate the samples back to the input space, we must find the  preimage $x^*$ such that $\psi(x^*) = \Psi^*$.

\subsection{Solving for an approximate preimage}\label{sec:preimage}

% \begin{wrapfigure}{r}{0.425\textwidth}
% \vspace{-18pt}
%     \centering
%     \includegraphics[width=0.23\linewidth]{imgs/data_2spirals.png}
%     \includegraphics[width=0.23\linewidth]{imgs/data_8gaussians.png}
%     \includegraphics[width=0.23\linewidth]{imgs/data_checkerboard.png}
%     \includegraphics[width=0.23\linewidth]{imgs/data_rings.png}\\
%     \includegraphics[width=0.23\linewidth]{imgs/samples_2spirals_mds.png}
%     \includegraphics[width=0.23\linewidth]{imgs/samples_8gaussians_mds.png}
%     \includegraphics[width=0.23\linewidth]{imgs/samples_checkerboard_mds.png}
%     \includegraphics[width=0.23\linewidth]{imgs/samples_rings_mds.png}\\
%     \includegraphics[width=0.23\linewidth]{imgs/samples_2spirals_wm.png}
%     \includegraphics[width=0.23\linewidth]{imgs/samples_8gaussians_wm.png}
%     \includegraphics[width=0.23\linewidth]{imgs/samples_checkerboard_wm.png}
%     \includegraphics[width=0.23\linewidth]{imgs/samples_rings_wm.png}\\
%     \caption{Sample generation results. \textbf{Top:} Data samples \textbf{Middle:} Samples by MDS-based preimage \textbf{Bottom:} Samples by weighted Fr\'{e}chet mean.}
%     \label{fig:samples}
%     \vspace{-1em}
% \end{wrapfigure}

Solving the preimage in kernel-based methods is known to be ill-posed \citep{Mika1999kpca} because the mapping $\psi(\cdot)$ is not necessarily surjective, i.e., a unique preimage $\mathbf{x}^* = \psi^{-1}(\Psi^*), \Psi^* \in \mathcal{H}$ may not exist. Often, an approximate preimage $\psi^{\dagger}(\mathbf{X}, \Psi^*) \approx \psi^{-1}(\Psi^*)$ is constructed instead based on relational properties among the training data in the RKHS. We consider two options in our framework \textbf{(1)} MDS-based method \citep{kwok2004pre,honeine2011preimage}, 
\begin{equation}
\psi_{\textrm{MDS}}^{\dagger}(\mathbf{X}, \Psi^*) = \tfrac{1}{2}(\mathbf{X}'\mathbf{X}'^\top)^{-1}\mathbf{X}'(\textrm{diag}(\mathbf{X}'^\top \mathbf{X}') - \mathbf{d}^\top), \textrm{where}\;\forall i \in [\gamma], \mathbf{d}_i = \Vert l(\mathbf{x}'_i, \cdot) - \Psi^* \Vert_{\mathcal{G}}
\end{equation}
which optimally preserves the distances in RKHS to the preimages in the input space, and \textbf{(2)} weighted Fr\'{e}chet mean \citep{friedman2001elements}, which in Euclidean space takes the form
\begin{equation}
\label{eq:wfm_euc}
\psi_{\textrm{wFM}}^{\dagger}(\mathbf{X}, \Psi^*) = \psi_{\textrm{wFM}}^{\dagger}(\mathbf{X}'; \mathbf{s}) = \mathbf{X}'\mathbf{s}/\Vert \mathbf{s} \Vert_1, \textrm{where}\;\forall i \in [\gamma],\mathbf{s}_i = \langle l(\mathbf{x}'_i, \cdot), \Psi^* \rangle
\end{equation}

\begin{wraptable}{r}{0.5\textwidth}
    \vspace{-2em}
    \begin{minipage}{0.5\textwidth}
        \begin{algorithm}[H]
        \footnotesize
        \caption{Sample Generation from kPF}
        \label{alg:gen_algo}
        \begin{algorithmic}[1]
        %   \begin{flushleft}
        
          \STATE {\bfseries Input:} 
          %\begin{compactitem}
            Training data $\mathbf{X} = \{\mathbf{x}_1, \dots, \mathbf{x}_n\}$, Optional autoencoder $(E, D)$, input/output kernels ($k, l$), neighborhood size $\gamma$\\
          \STATE {\bfseries Training}
          \STATE \quad $\mathbf{X} = \left(E\left(\mathbf{x}_1\right), \dots, E\left(\mathbf{x}_n\right)\right)$ if $E$ is provided
          \STATE \quad Sample $\{\mathbf{z}_i\}_{i \in [n]} \sim P_Z$ independently \\
          \STATE \quad Construct $L, K \in R^{n \times n}$ s.t.\\
                 \qquad $L_{ij} = l(\mathbf{x}_i, \mathbf{x}_j), K_{ij} = k\left(\mathbf{z}_i, \mathbf{z}_j\right)$\\
          \STATE \quad $K_{\textrm{inv}} = (K + \lambda n I)^{-1}$ or $K^{\dagger}$\\
          \STATE {\bfseries Inference}
          \STATE \quad Generate new prior sample $\mathbf{z}^* \sim P_Z$\\
          \STATE \quad$\mathbf{s} = L\cdot K_{\textrm{inv}}[k(\mathbf{z}_1, \mathbf{z}^*) \dots k(\mathbf{z}_n, \mathbf{z}^*)]^\top$\\
          \STATE \quad $ind = \textrm{argsort}(\mathbf{s})[-\gamma:]$\\
        
          \STATE \quad$\mathbf{x}^* = \psi_{\textrm{wFM}}^{\dagger}\left(\mathbf{X}[ind]; \mathbf{s}[ind] \right)$.
          \STATE {\bfseries Output} $D(\mathbf{x}^*)$ if $D$ is provided else $\mathbf{x}^*$
        \end{algorithmic}
        \end{algorithm}
    \end{minipage}
    \vspace{-1em}
\end{wraptable}
where $\mathbf{X}'$ a neighborhood of $\gamma$ training samples based on pairwise distance or similarity in RKHS, following \citep{kwok2004pre}. The weighted Fr\'{e}chet mean preimage uses the inner product weights $\langle \Psi^*, \psi(\mathbf{x}_i) \rangle$ as measures of similarities to interpolate training samples. On the toy data (as in Fig. \ref{fig:density}), weighted Fr\'{e}chet mean produces fewer samples that deviate from the true distribution and is easier to compute. Based on this observation, we use the weighted Fr\'{e}chet mean as the preimage module for all experiments that requires samples, while acknowledging that other preimage methods can also be substituted in.

% \begin{wrapfigure}{r}{0.5\textwidth}

% \end{wrapfigure}

With all the ingredients in hand, we now present an algorithm for sample generation using the kPF operator in Alg. \ref{alg:gen_algo}. The idea is simple yet powerful: at training time, we construct the empirical kPF operator $\hat{\mathcal{P}}_{\mathcal{E}}$ using the training data $\{\mathbf{x}_i\}_{i \in [s]}$ and samples of the known prior $\{\mathbf{z}_i\}_{\i \in [n]}$. At test time, we will transfer new points sampled from $P_Z$ to feature maps in $\mathcal{H}$, and construct their preimages as the generated output samples.

\subsection{Image generation} Image generation is a common application for generative models \citep{goodfellow2014gan,dinh2017RealNVP}. 
While our proposal is not image specific, constructing sample preimages in a high dimensional space with limited training samples can be challenging, 
since the space of images is usually not dense in a reasonably sized neighborhood. However, empirically images often lie near a low dimensional manifold in the ambient space \citep{Seung2000manifold}, and one may utilize an autoencoder (AE) $(E, D)$ to embed the images onto a latent space that represents coordinates on a learned manifold. If the learned manifold lies close to the true manifold, we can learn densities on the manifold directly \citep{dai2018diagnosing}. 

Therefore, for image generation tasks, the training data is first projected onto the latent space of a pretrained AE. Then, the operator will be constructed using the projected latent representations, and samples will be mapped back to image space with the decoder of AE. Our setup can be viewed analogously to other generative methods based on so called ``\textit{ex-post}'' density estimation of latent variables \citep{Ghosh2020From}. We also restrict the AE latent space to a hypersphere $\mathbf{S}^{n- 1}$ to ensure that \begin{inparaenum}[\bfseries (a)]
\item $k(\cdot, \cdot)$ and $l(\cdot, \cdot)$ are bounded and \item the space is geodesically convex and complete, which is required by the preimage computation. To compute the weighted Fr\'{e}chet mean on a hypersphere, we adopt the recursive algorithm in \citet{chakraborty2015recursive} (see appendix \ref{appdx:wfm_sphere} for details).
\end{inparaenum}

\begin{figure}[t]
    \centering
    \begin{tabular}{c|c}
        \setlength{\tabcolsep}{1pt}
         \begin{tabular}{m{1cm} m{1cm} m{1cm} m{1cm} m{1cm}}
            Data &
            \includegraphics[width=\linewidth]{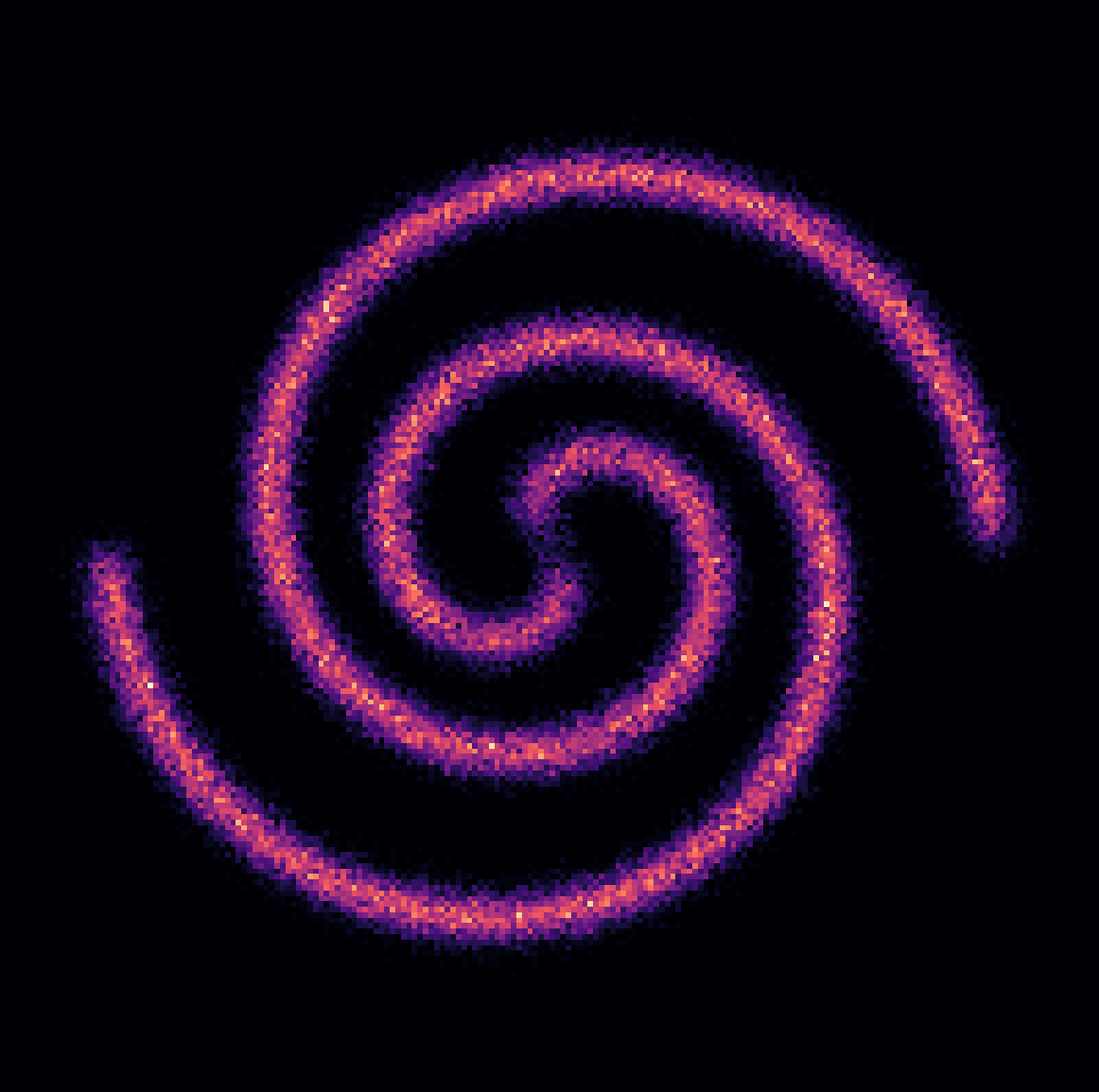}&
            \includegraphics[width=\linewidth]{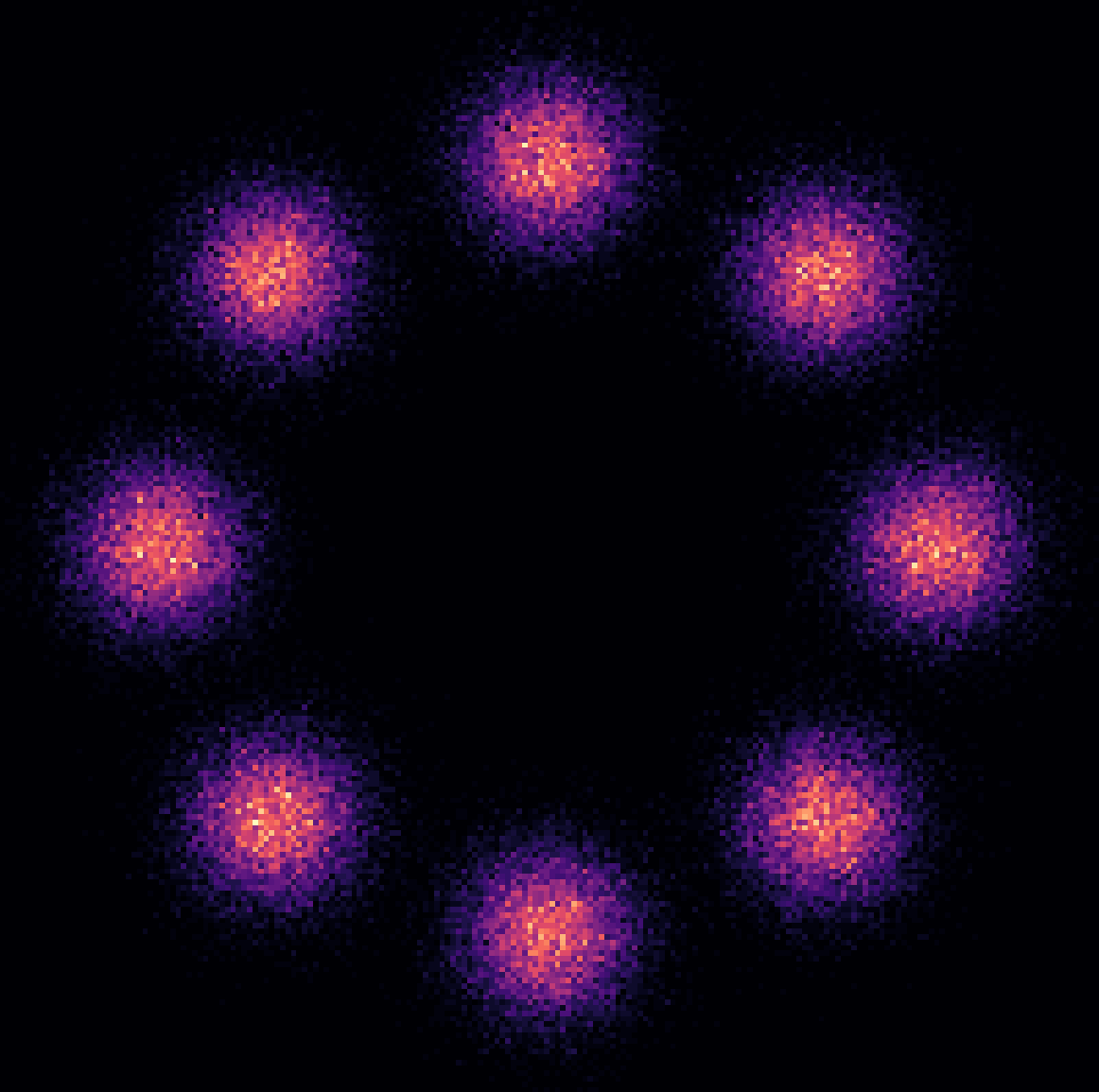}&
            \includegraphics[width=\linewidth]{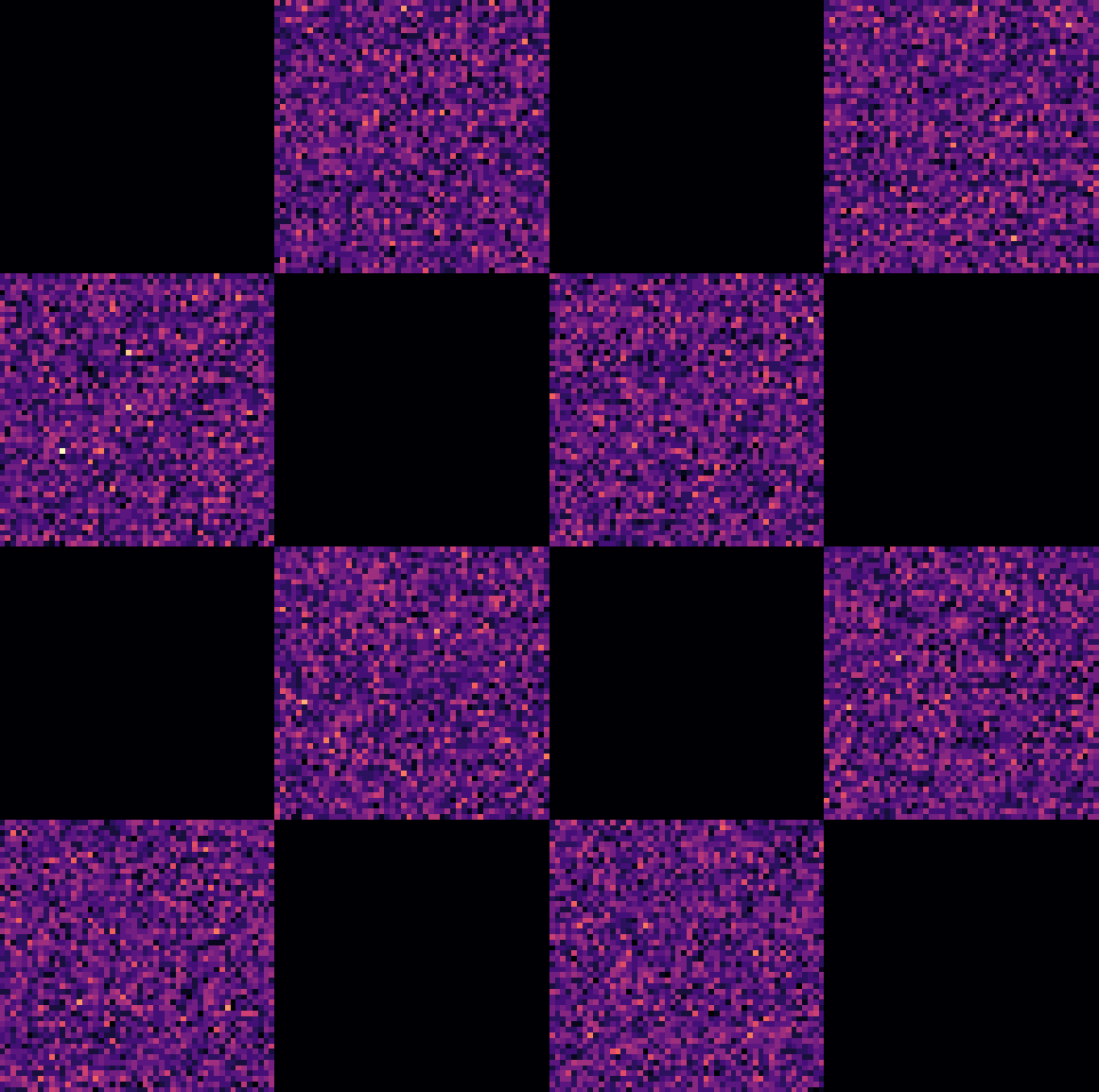}&
            \includegraphics[width=\linewidth]{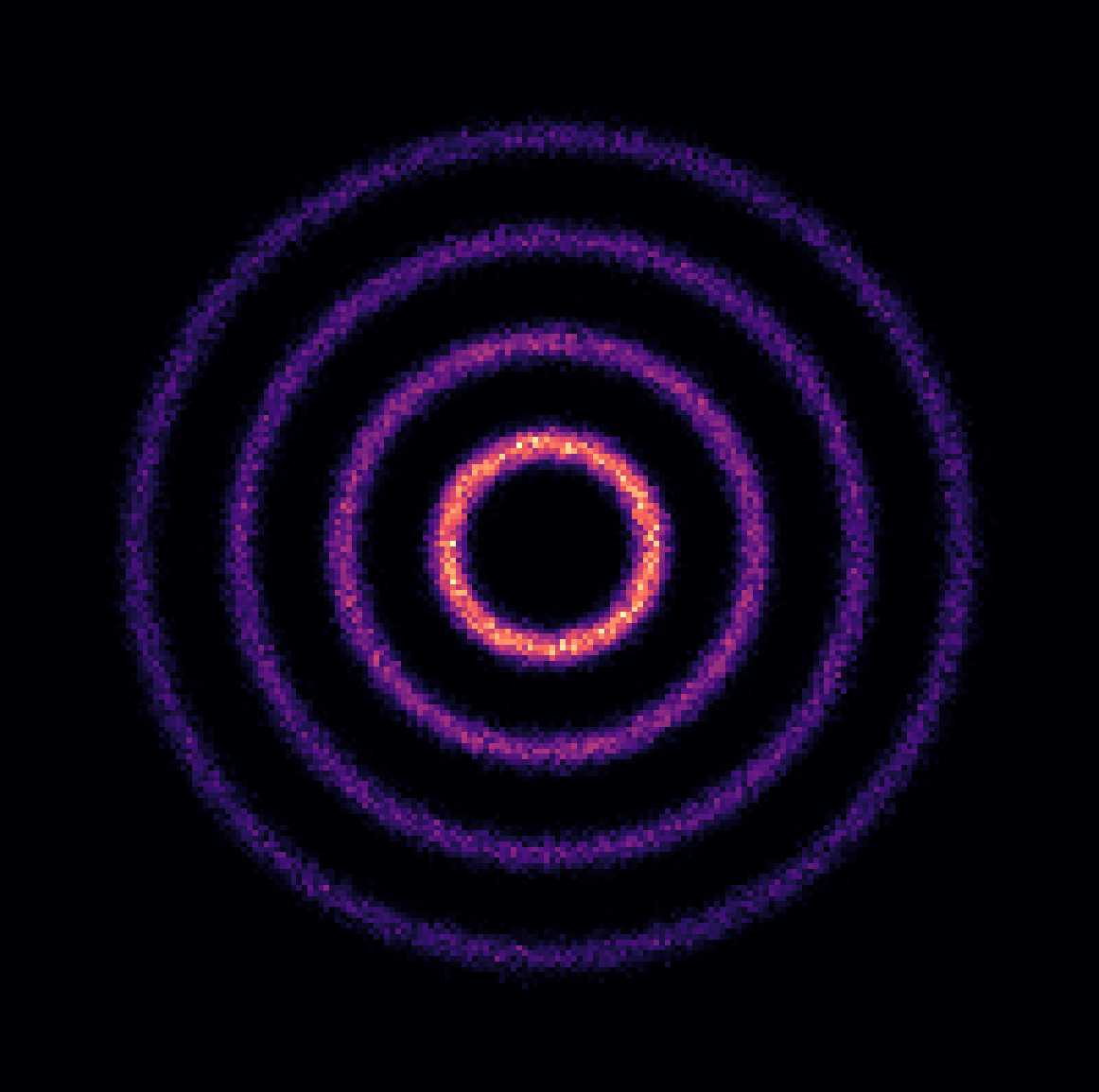}\\
            GMM &
            \includegraphics[width=\linewidth]{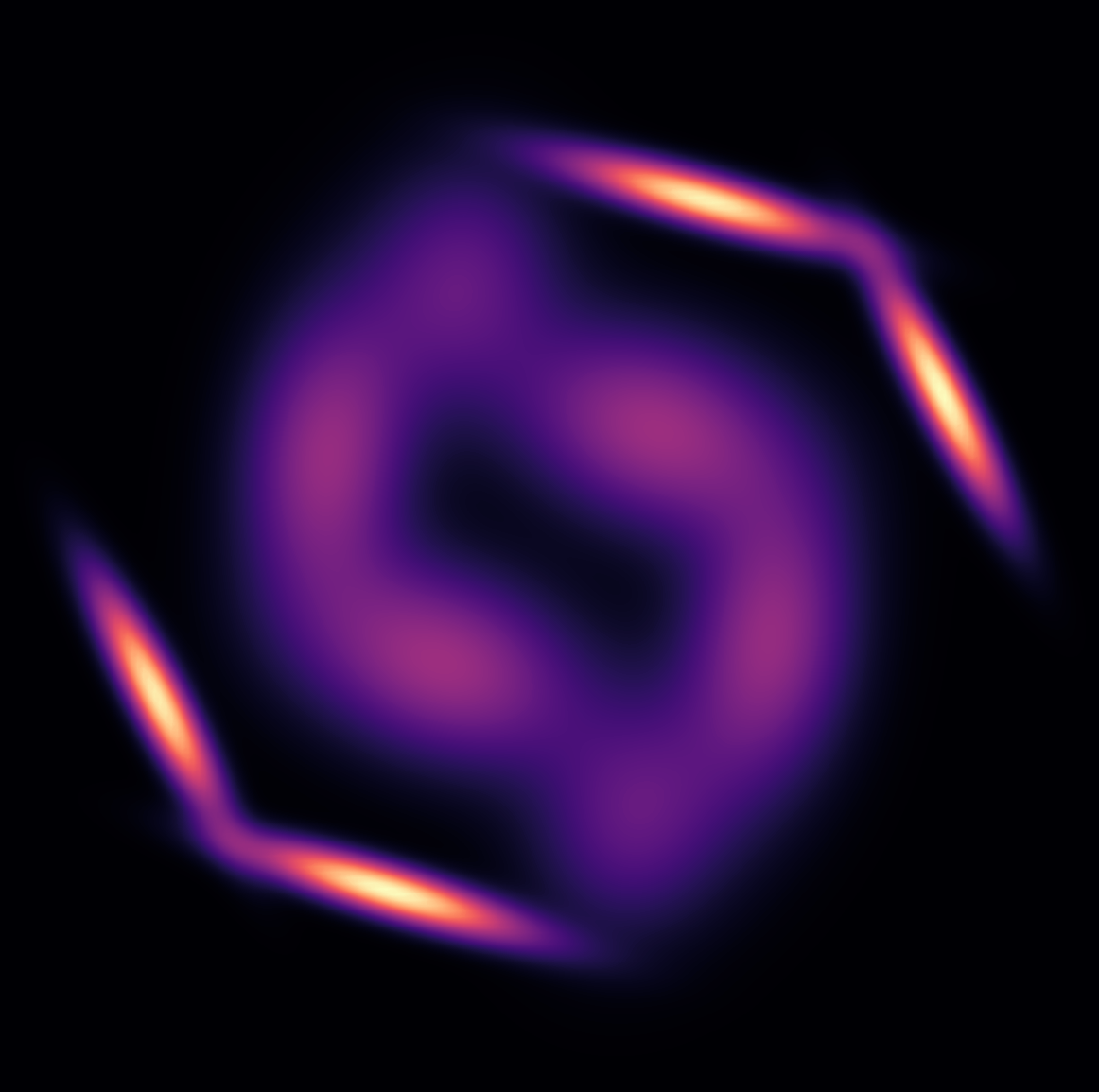}&
            \includegraphics[width=\linewidth]{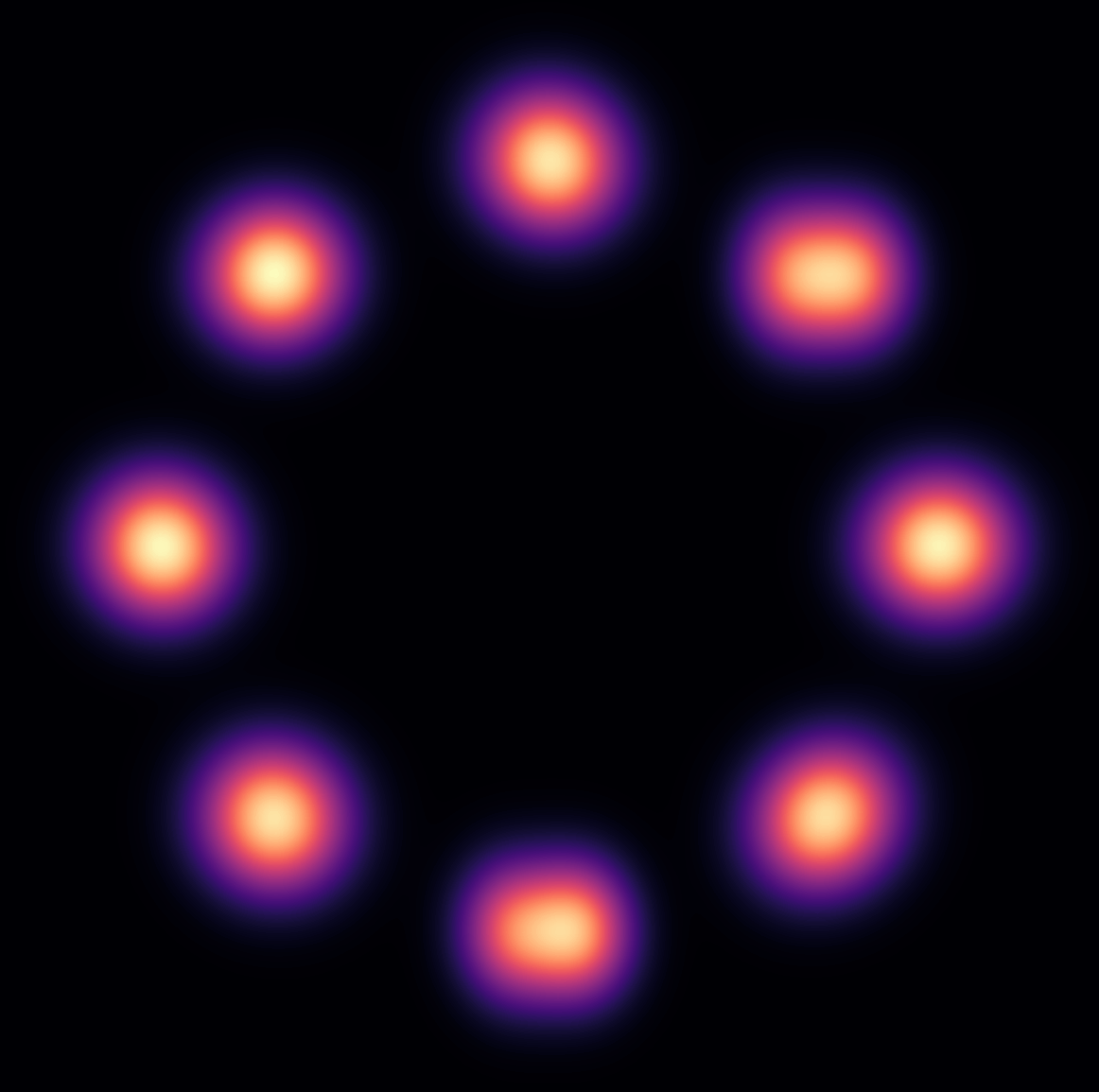}&
            \includegraphics[width=\linewidth]{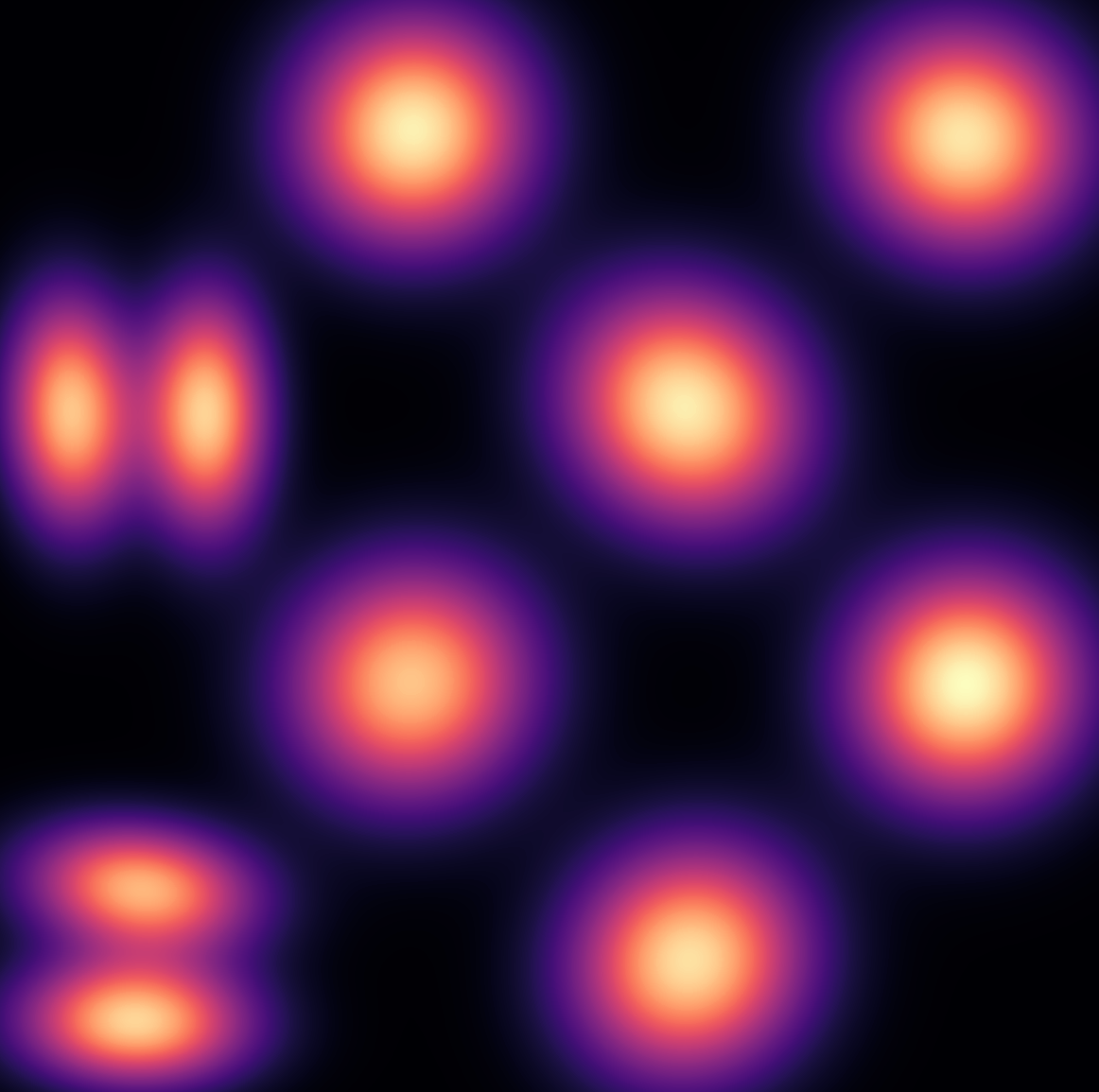}&
            \includegraphics[width=\linewidth]{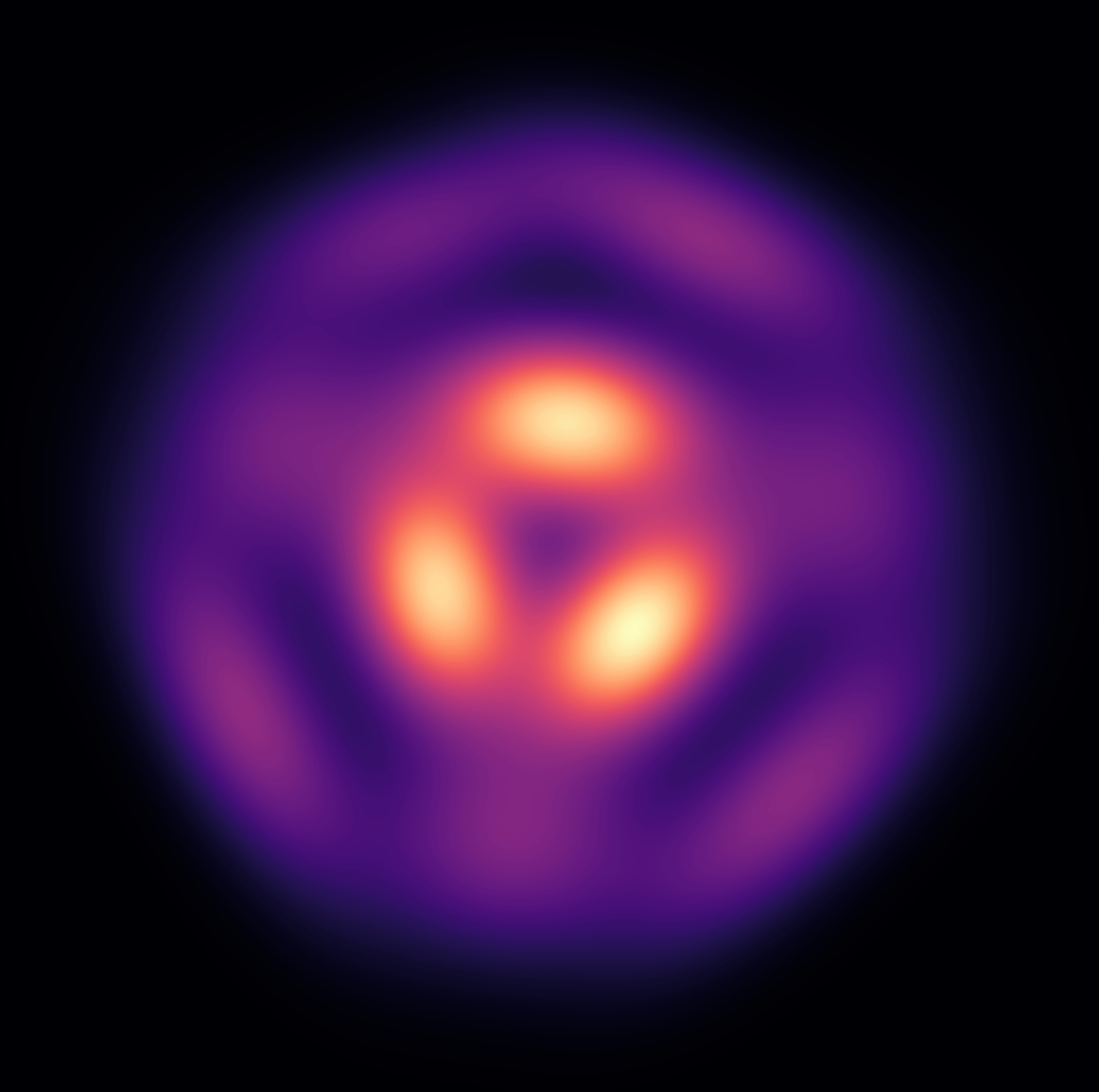}\\
            Glow &
            \includegraphics[width=\linewidth]{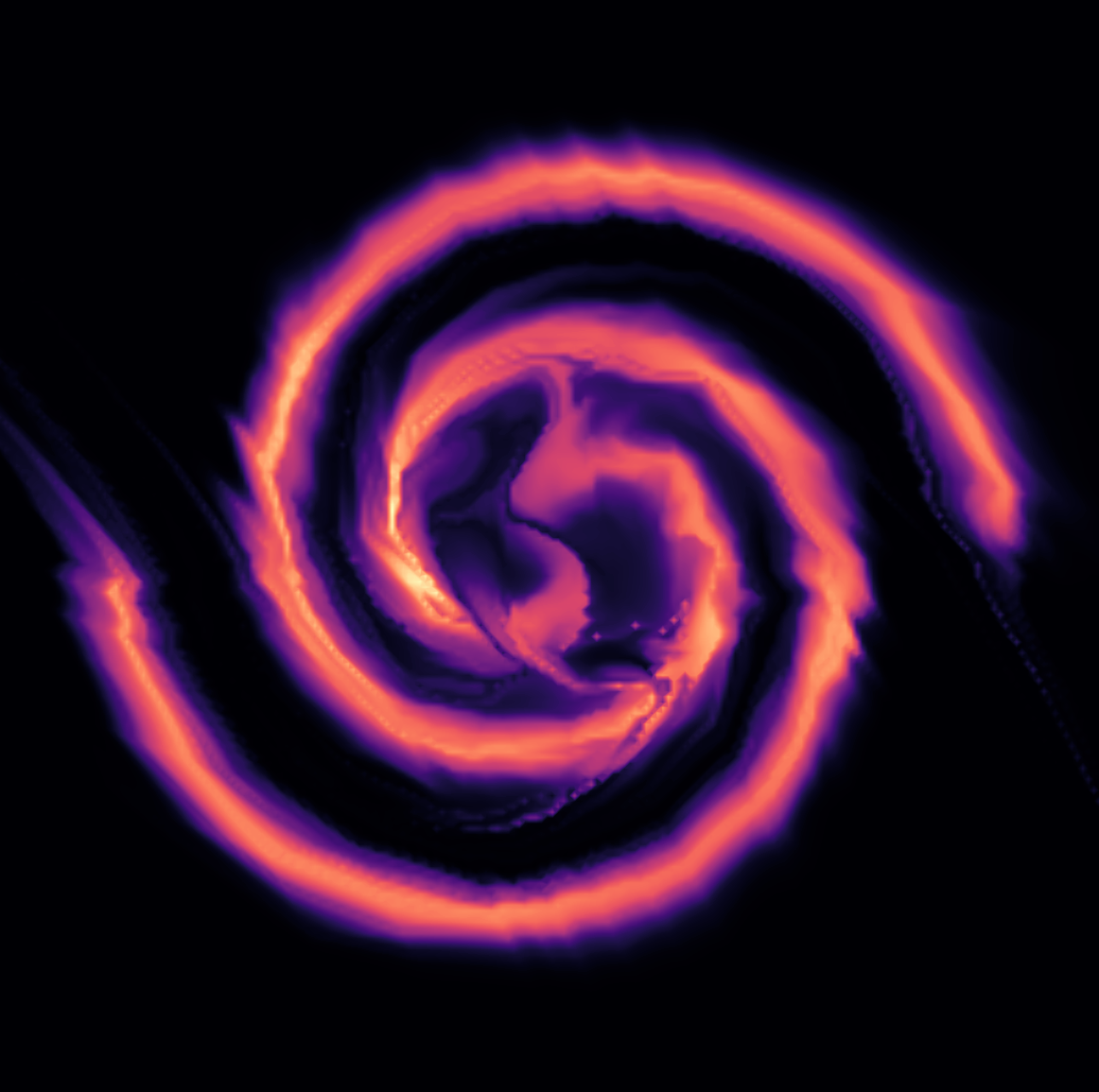}&
            \includegraphics[width=\linewidth]{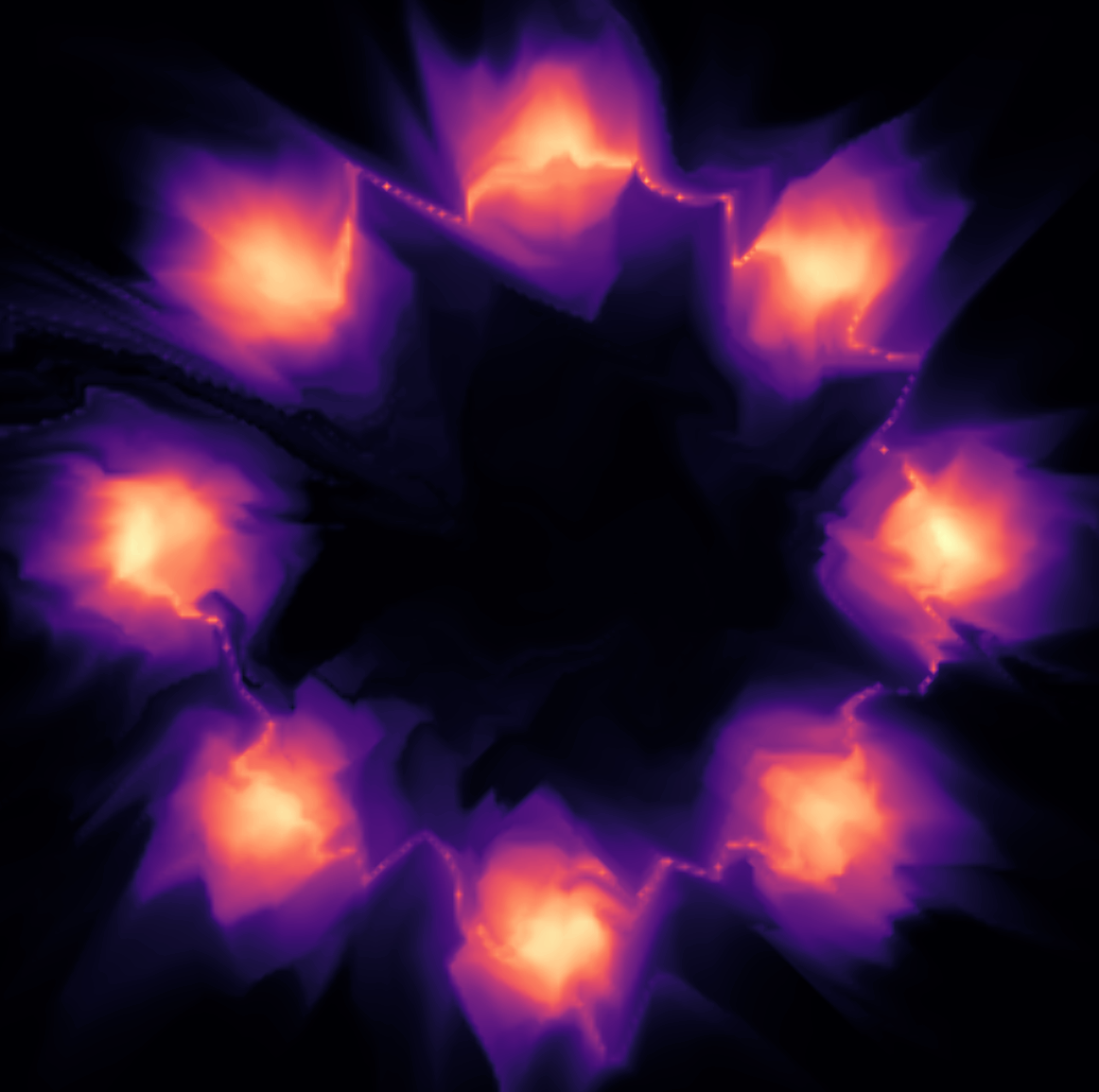}&
            \includegraphics[width=\linewidth]{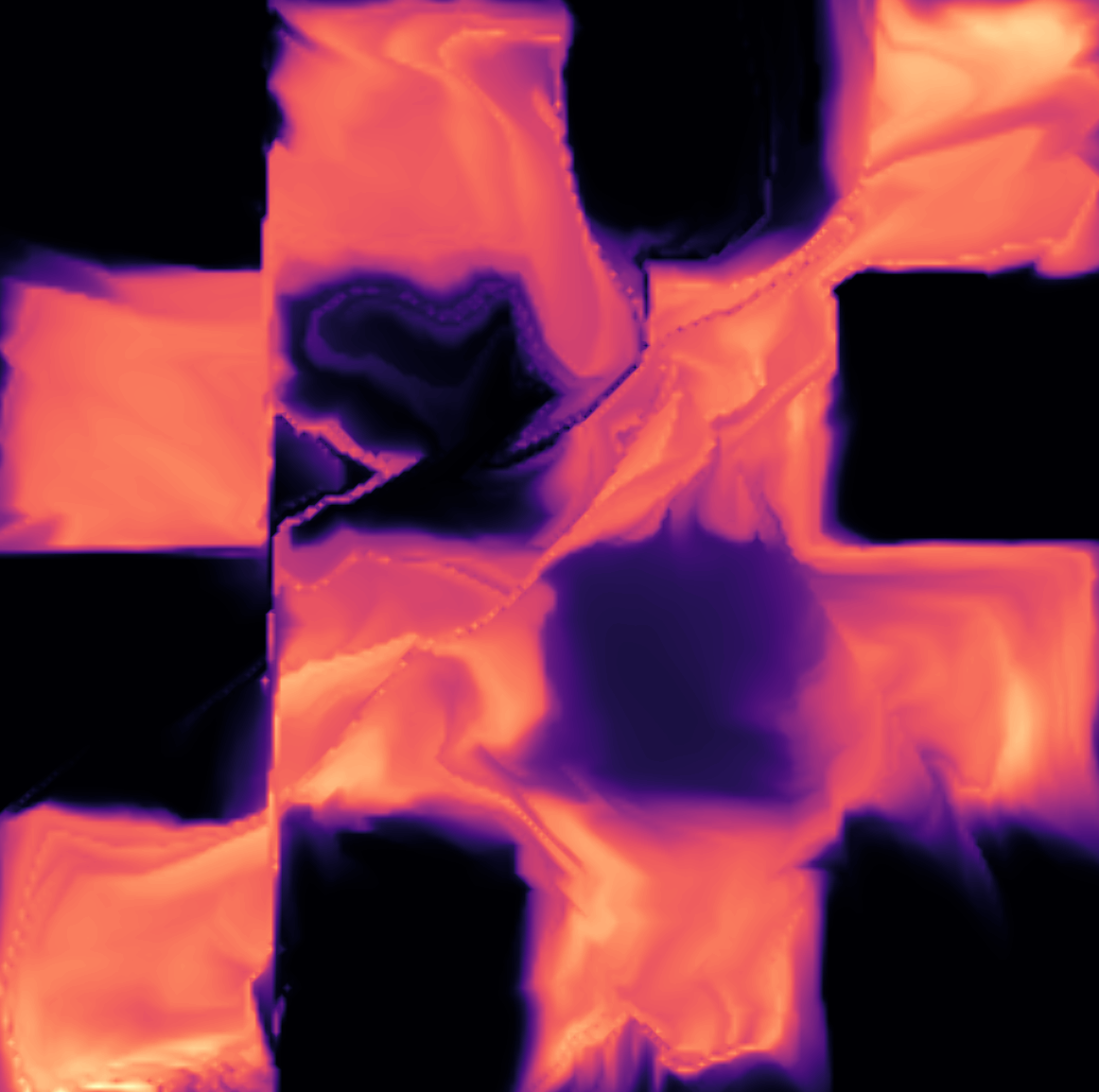}&
            \includegraphics[width=\linewidth]{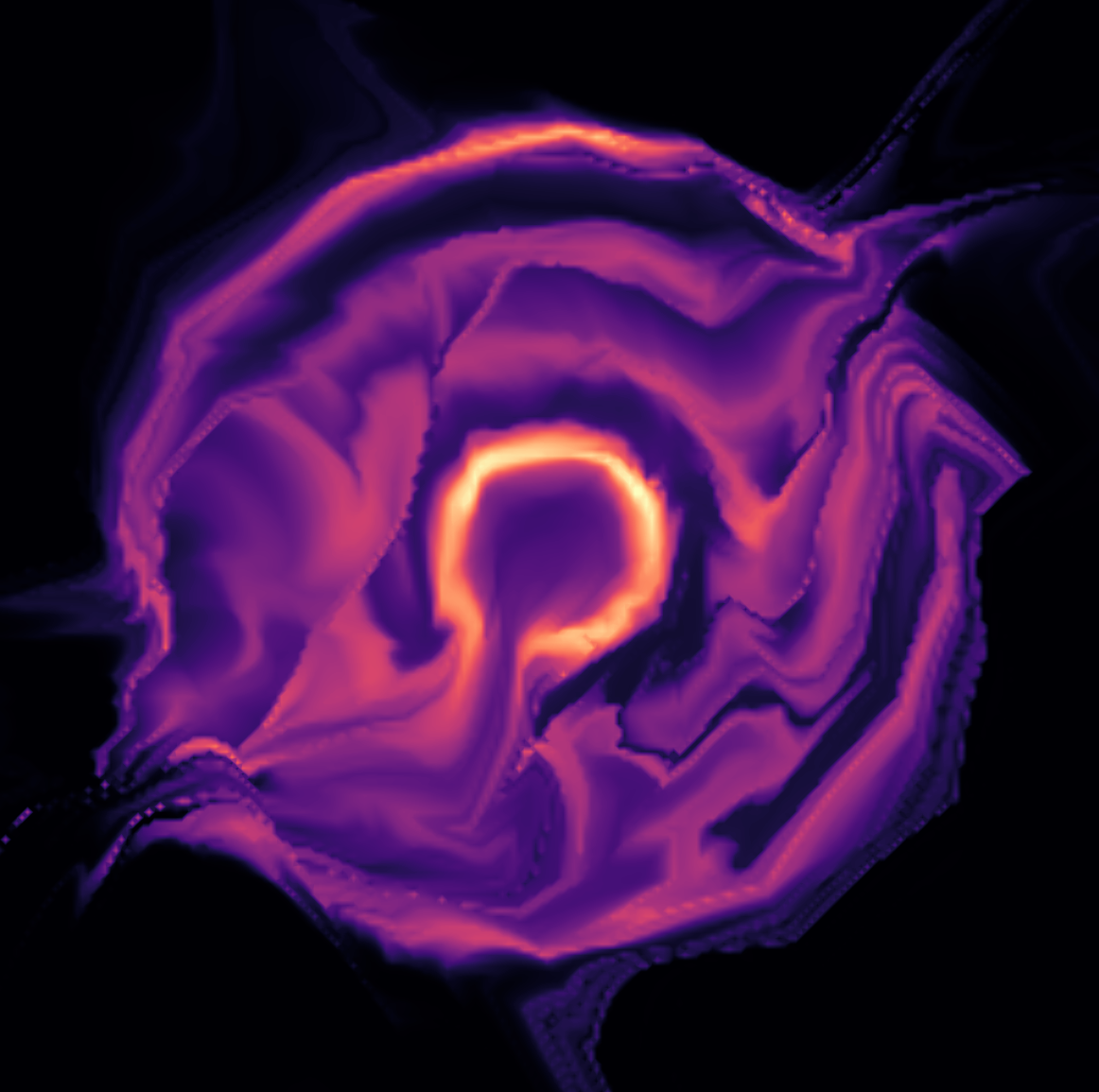}\\
            kPF &
            \includegraphics[width=\linewidth]{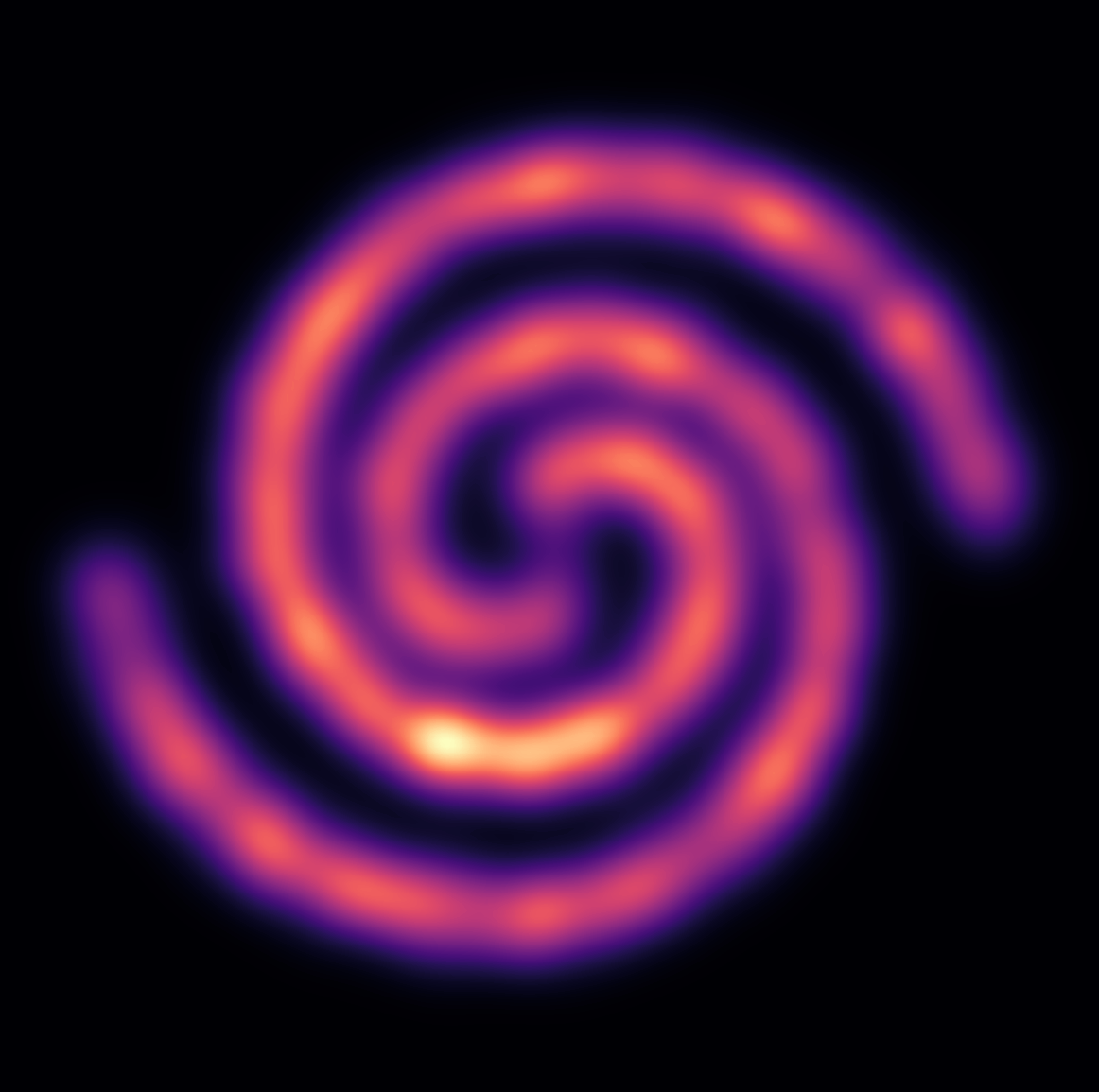}&
            \includegraphics[width=\linewidth]{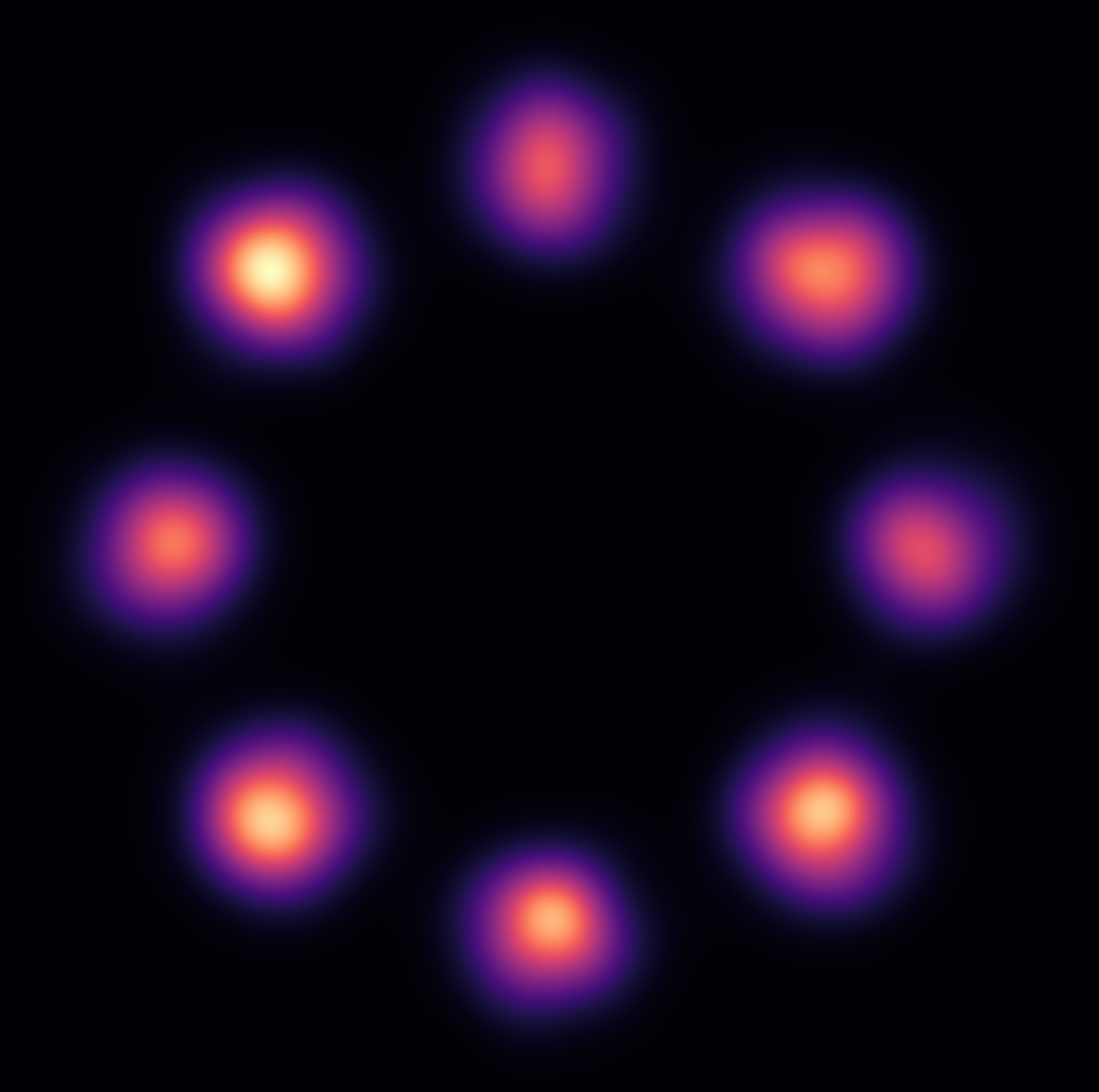}&
            \includegraphics[width=\linewidth]{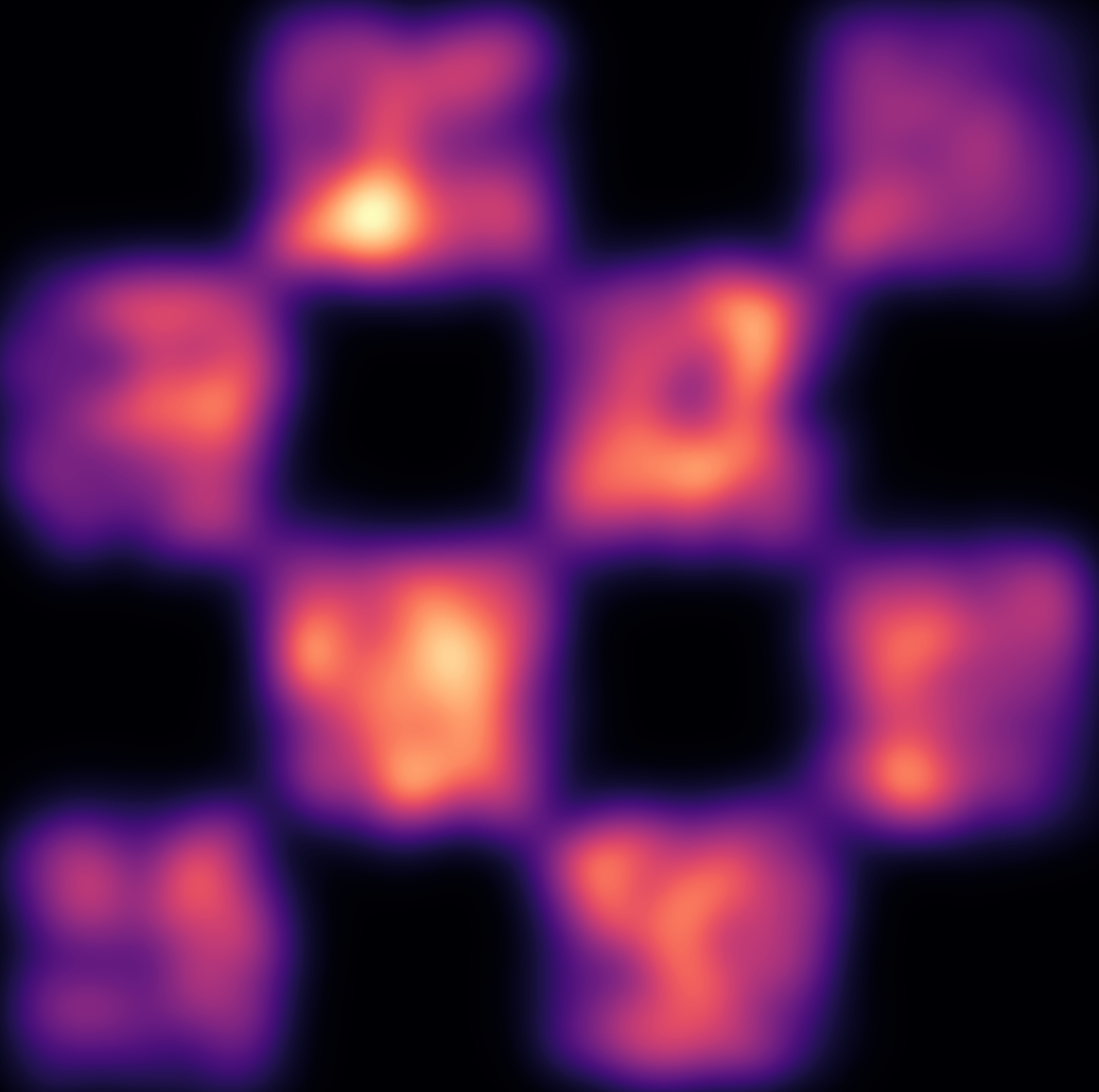}&
            \includegraphics[width=\linewidth]{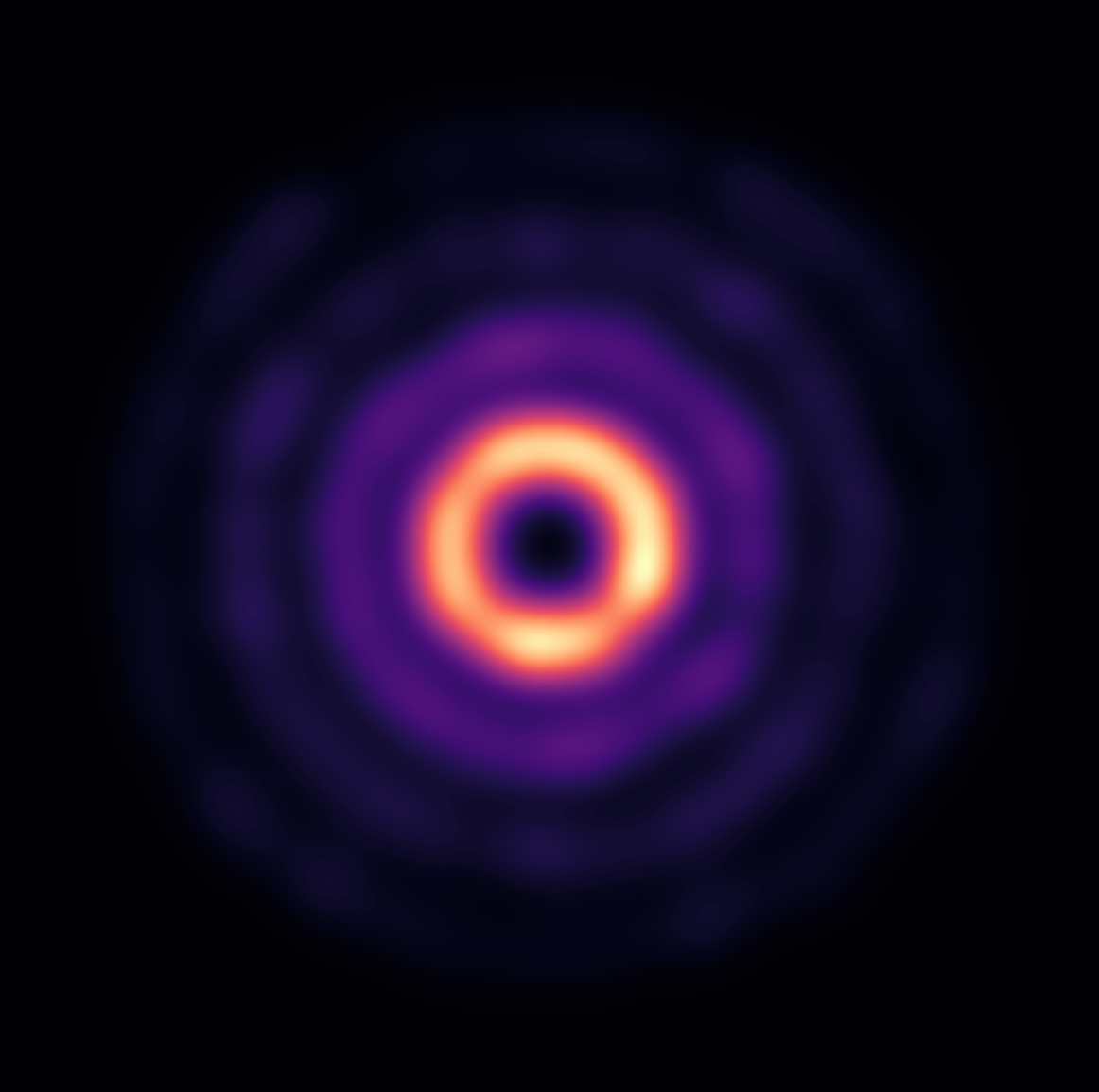}\\
            % \caption{Density estimation on 2D toy data. \textbf{Top to bottom: (1)} Training data samples, and learned densities of \textbf{(2)} GMM \textbf{(3)} Glow \textbf{(4)} Proposed kPF operator. Details of this experiment can be found in the supplement.}
            % \vspace{-4em}
            % \label{fig:density}
            % \vspace{-7em}
        \end{tabular}
         & 
         \setlength{\tabcolsep}{1pt}
         \begin{tabular}{m{1cm} m{1.2cm} m{1.2cm} m{1.2cm} m{1.2cm}}
            Data&
            \includegraphics[width=\linewidth]{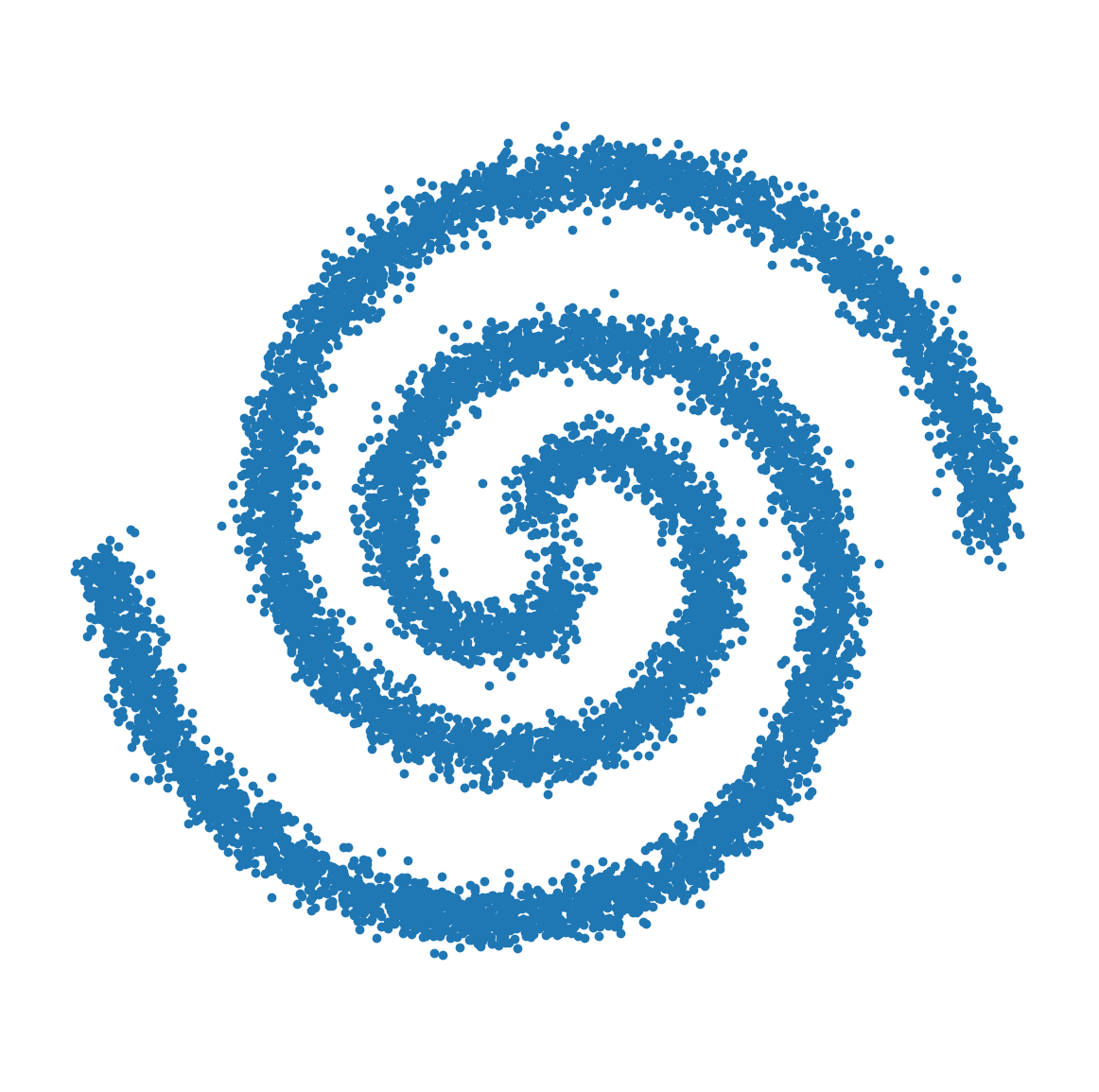}&
            \includegraphics[width=\linewidth]{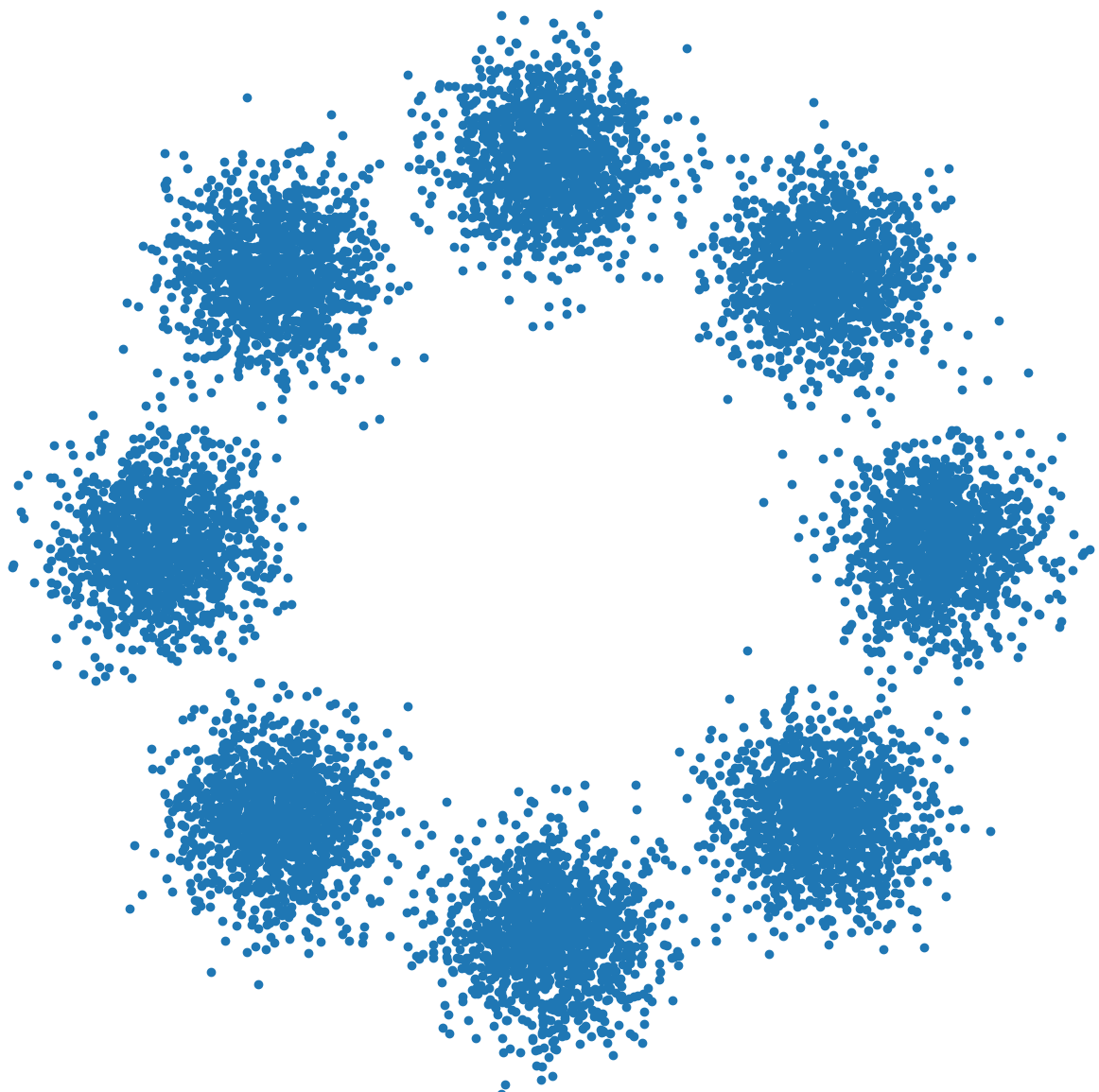}&
            \includegraphics[width=\linewidth]{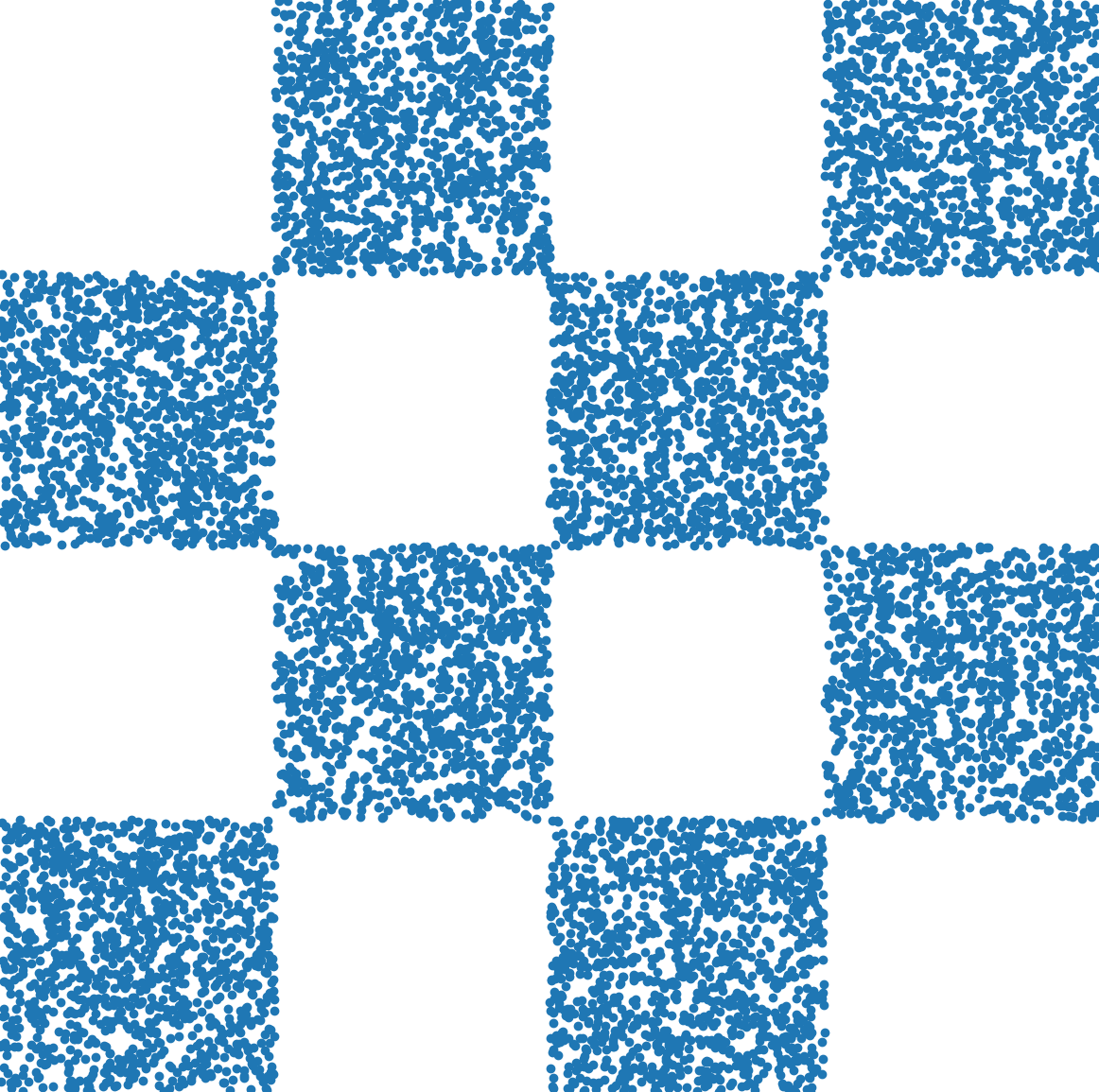}&
            \includegraphics[width=\linewidth]{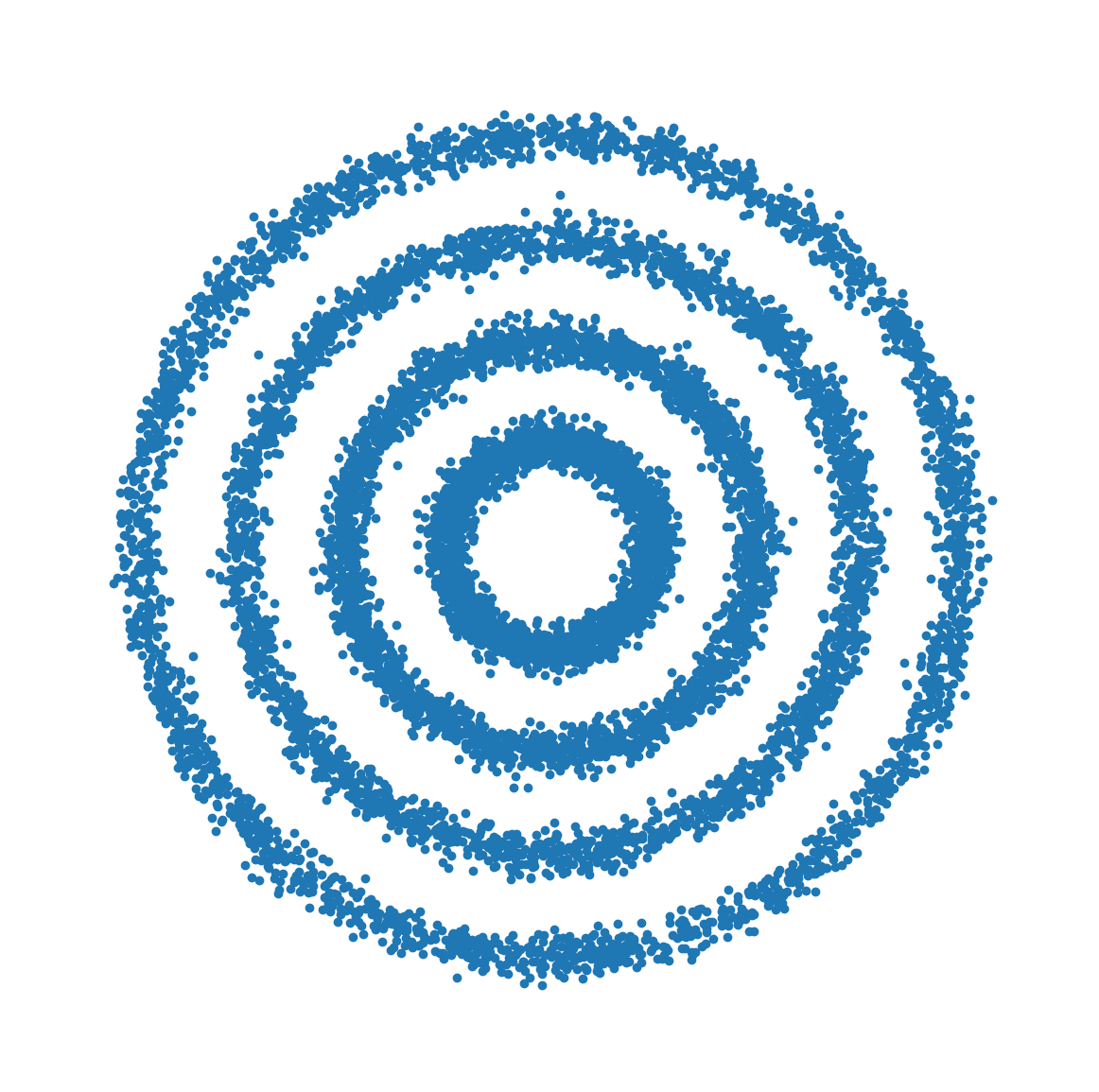}\\
            MDS &
            \includegraphics[width=\linewidth]{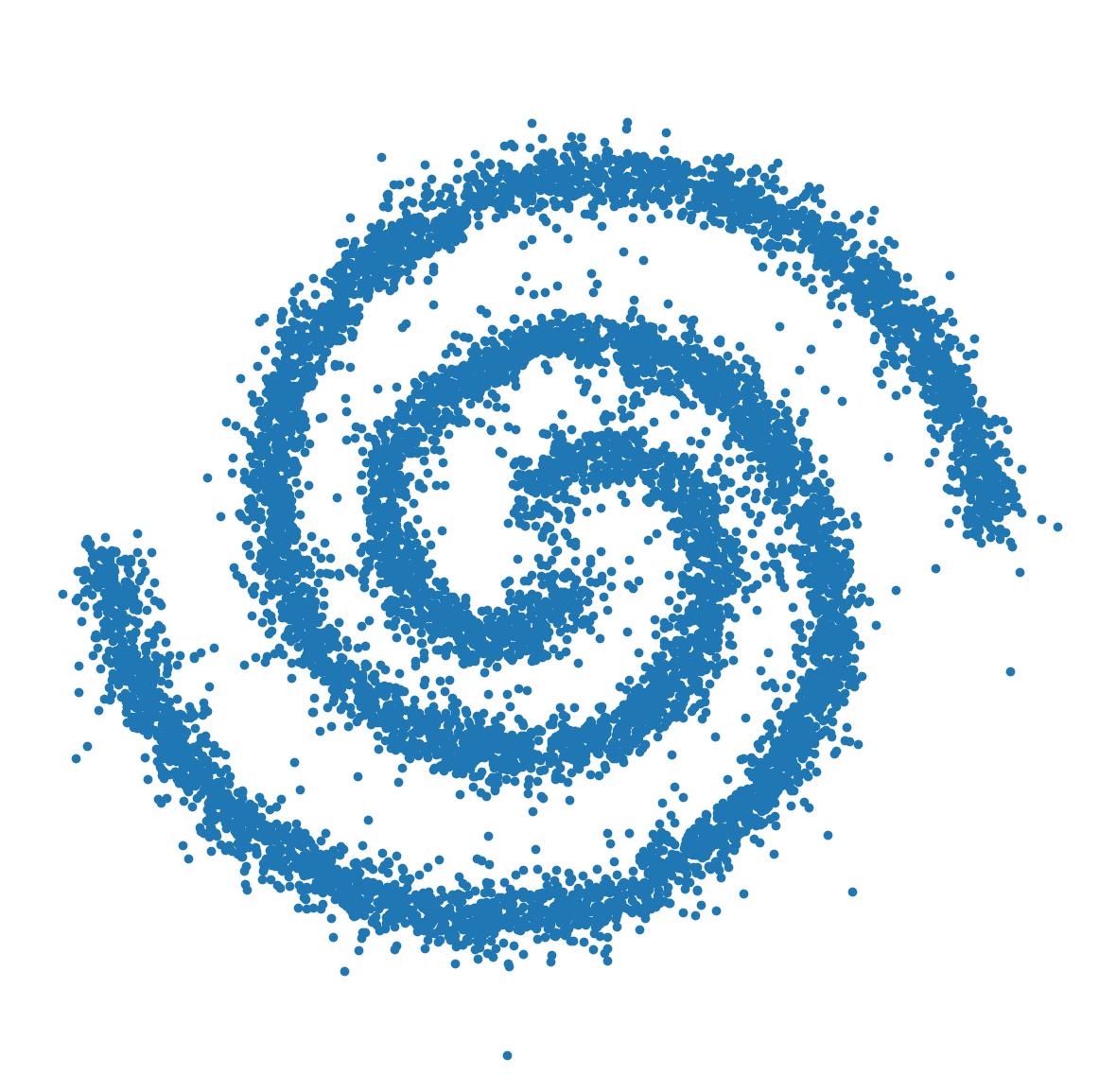}&
            \includegraphics[width=\linewidth]{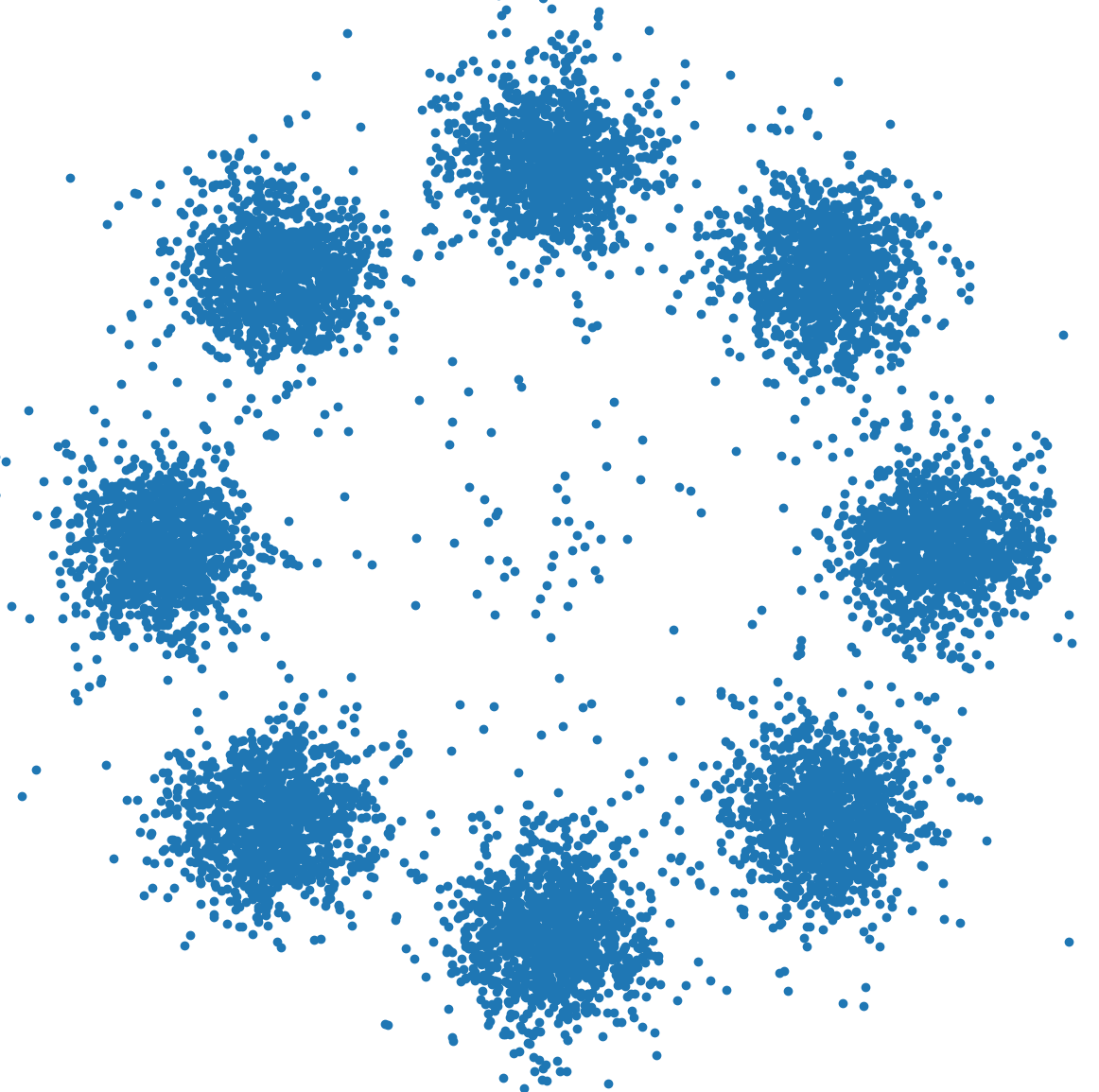}&
            \includegraphics[width=\linewidth]{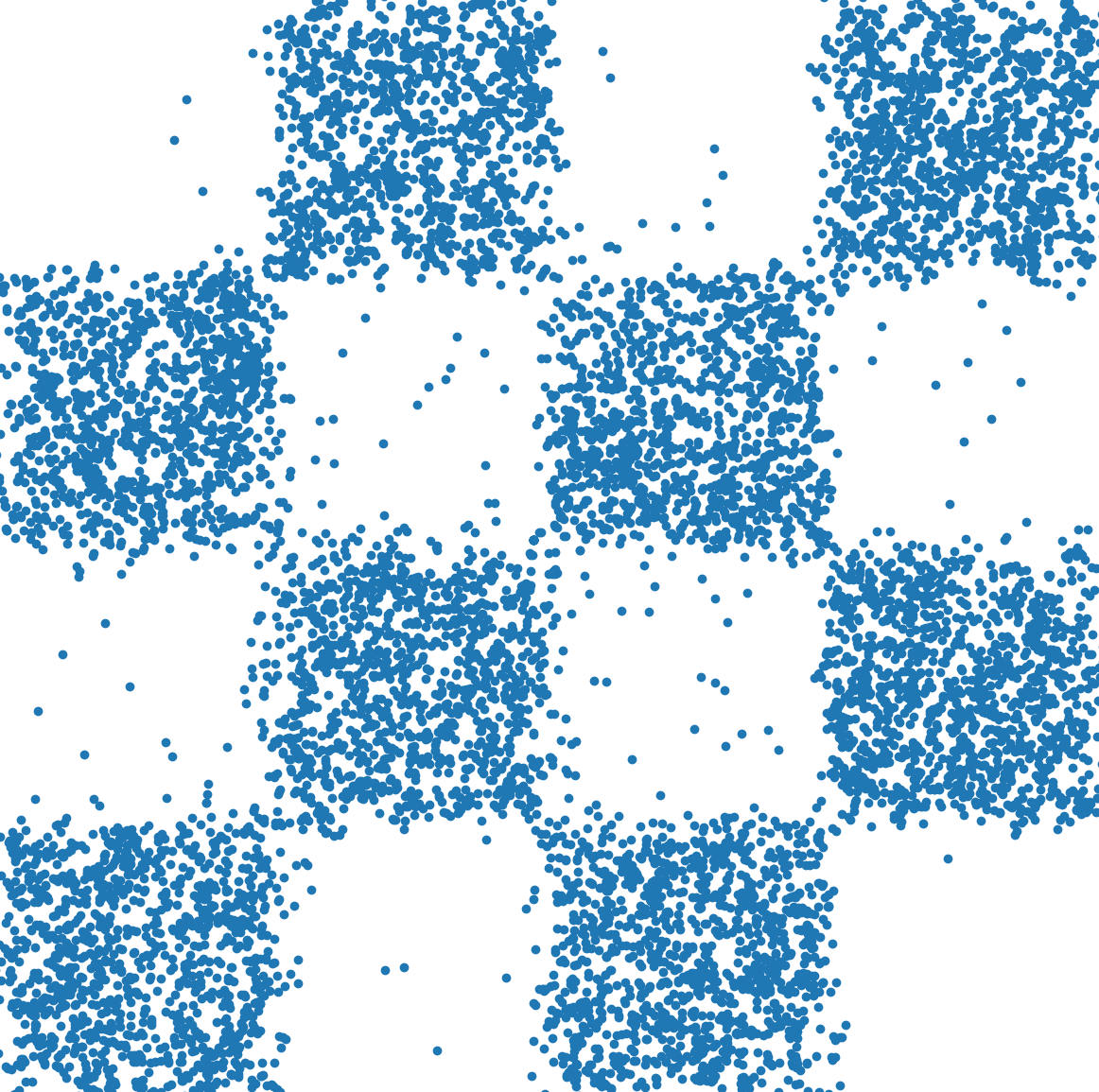}&
            \includegraphics[width=\linewidth]{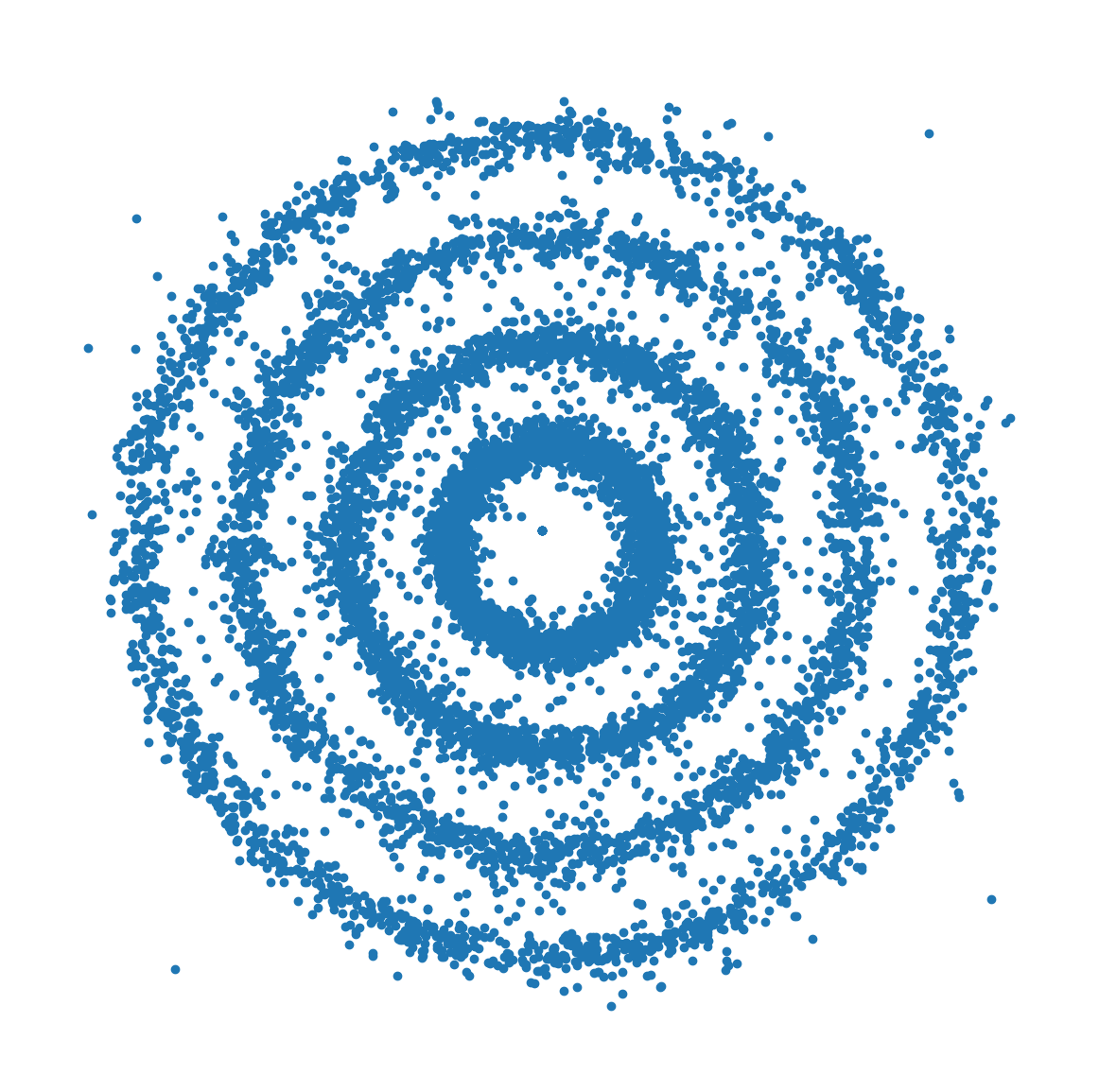}\\
            wFM &
            \includegraphics[width=\linewidth]{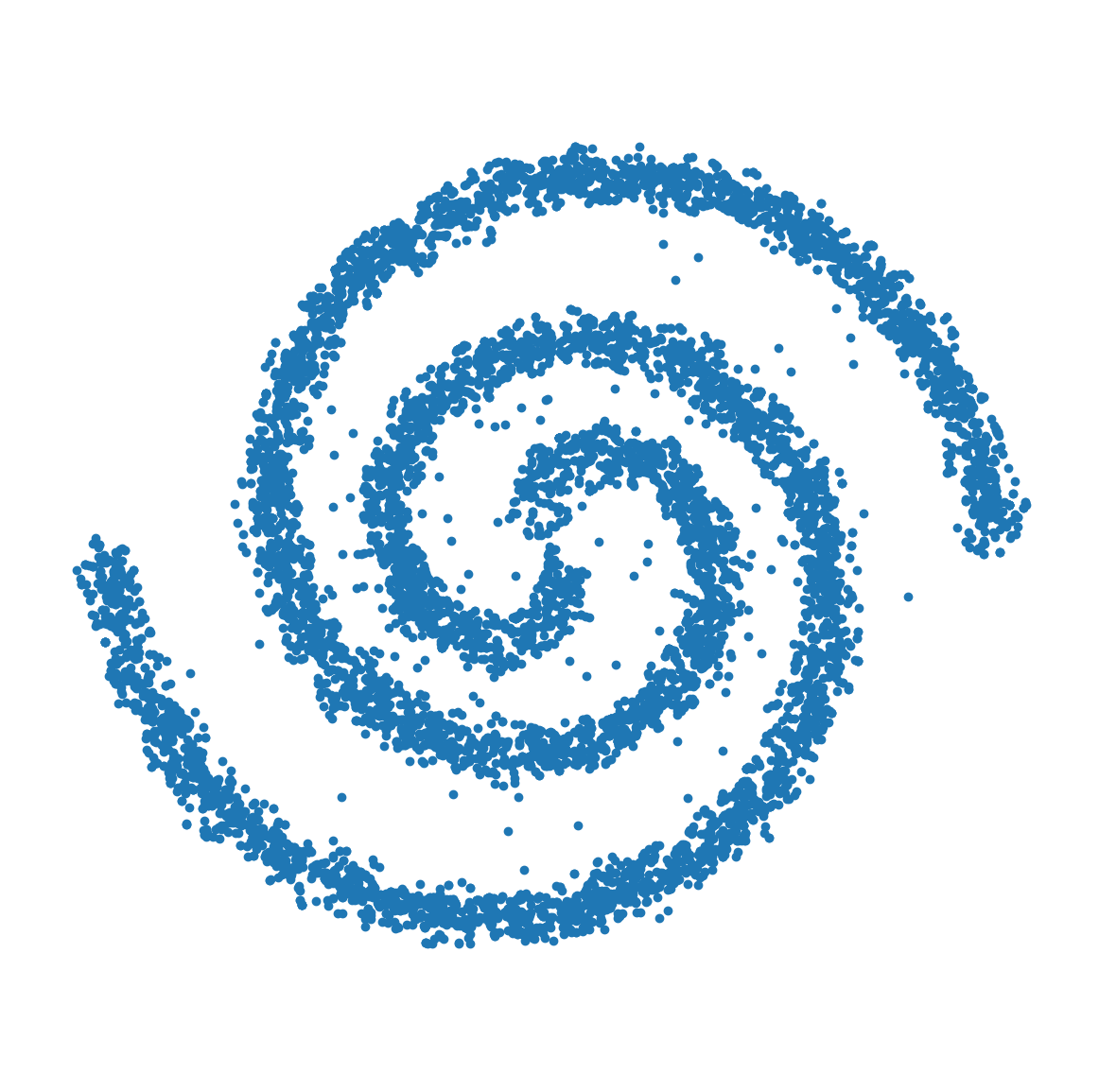}&
            \includegraphics[width=\linewidth]{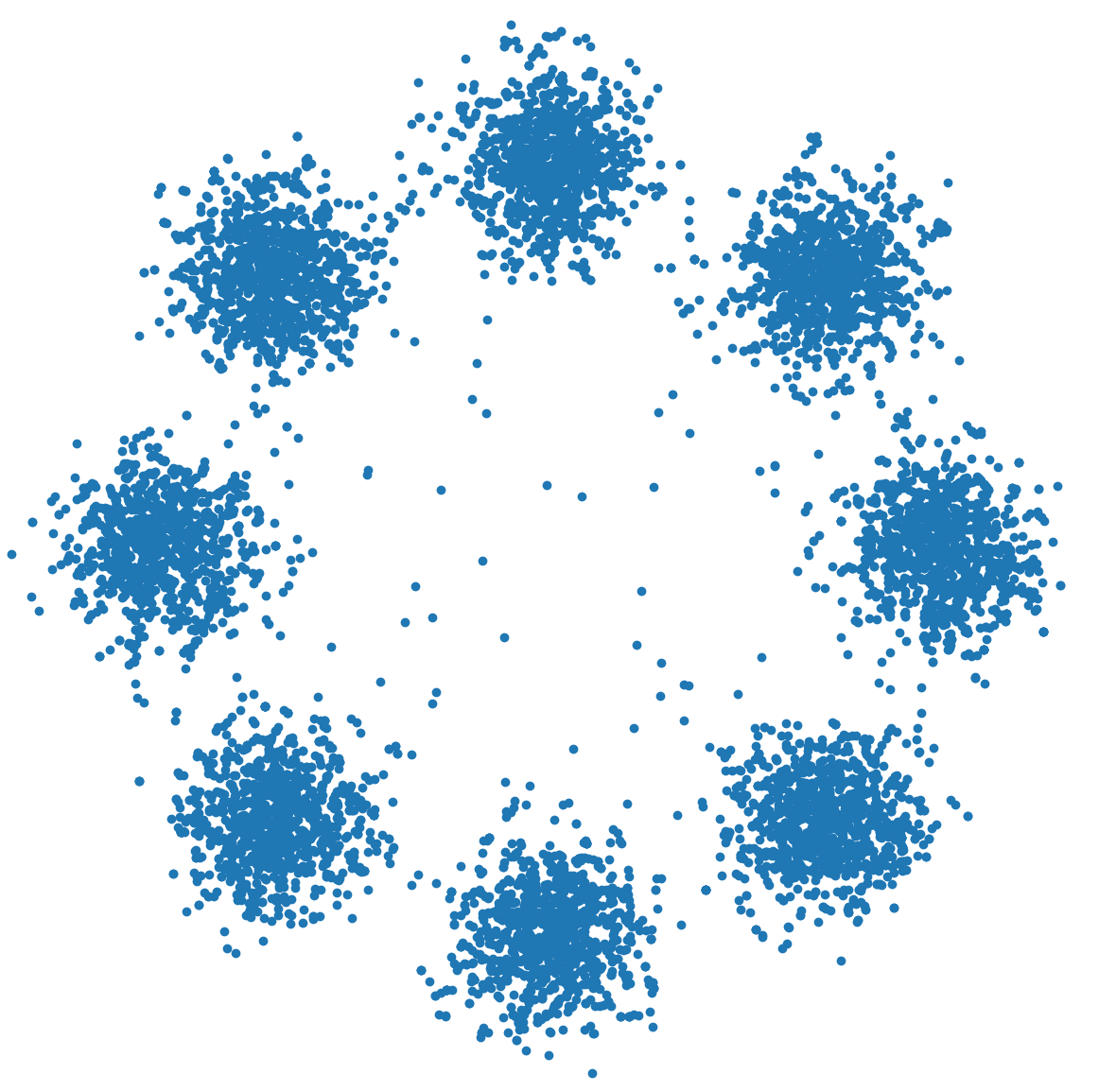}&
            \includegraphics[width=\linewidth]{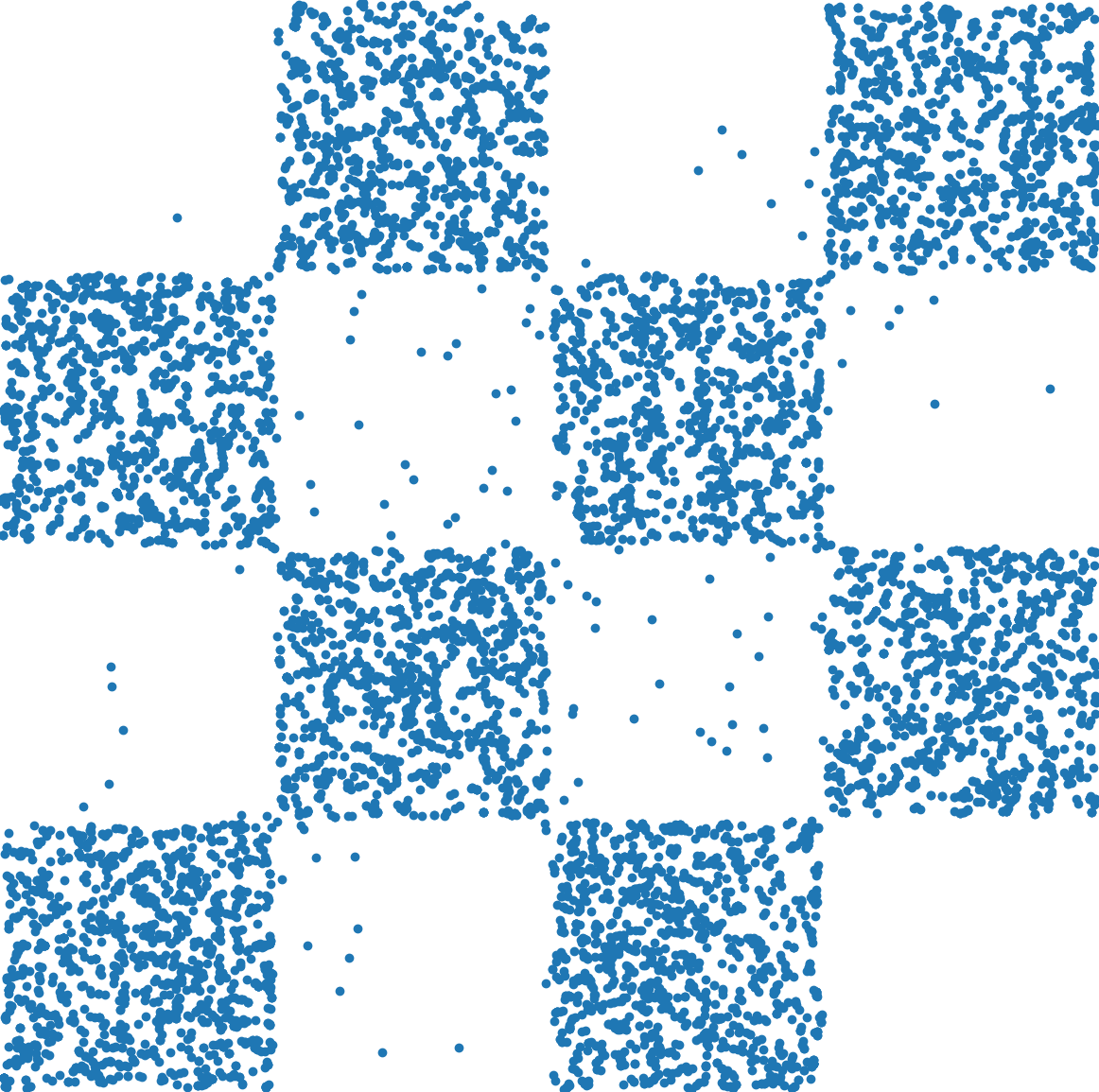}&
            \includegraphics[width=\linewidth]{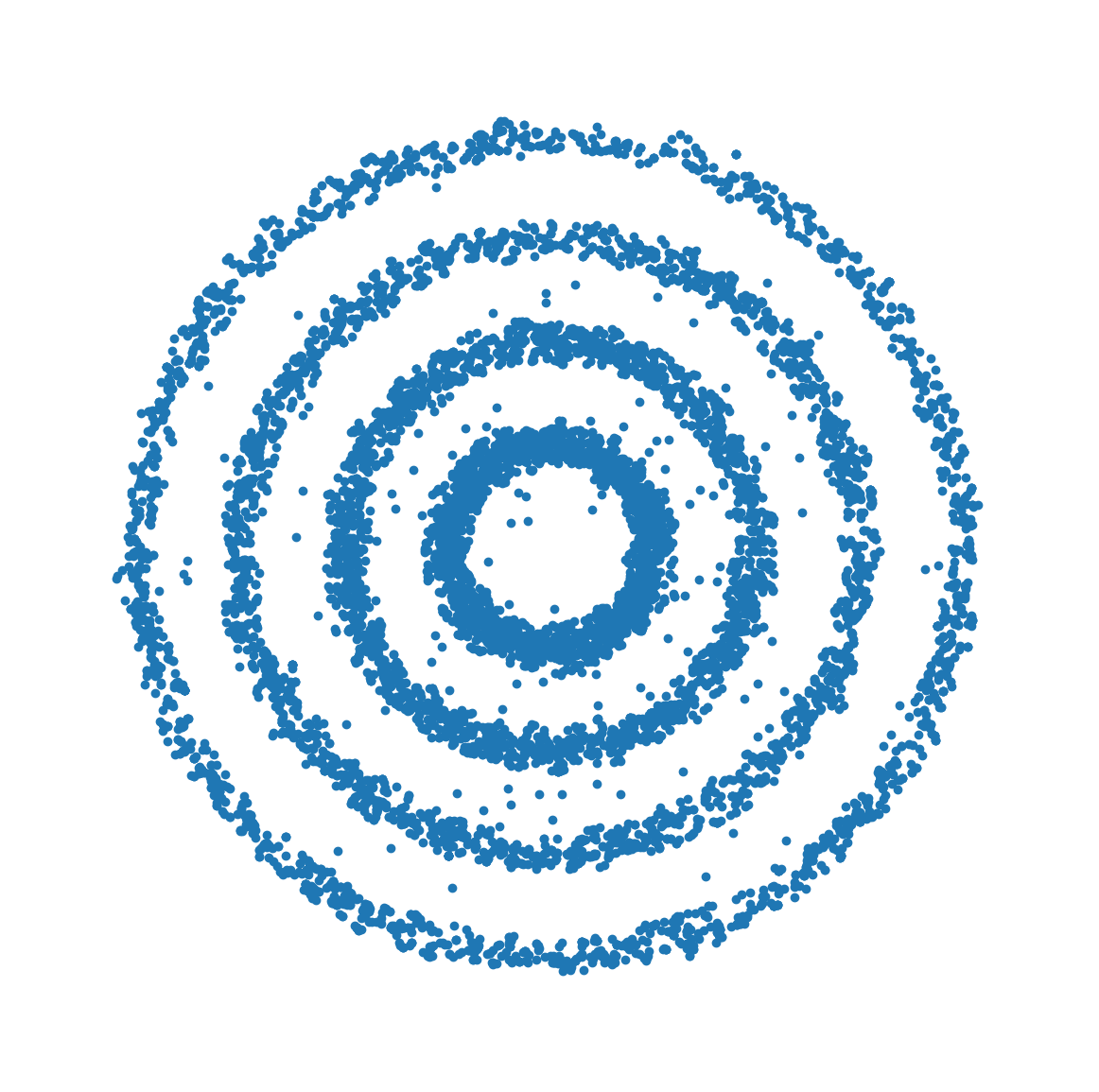}\\
         \end{tabular}
    \end{tabular}
    \vspace{-8pt}
    \caption{\textit{Left figure:} Density estimation on 2D toy data. \textbf{Top to bottom: (1)} Training data samples, and learned densities of \textbf{(2)} GMM \textbf{(3)} Glow \textbf{(4)} Proposed kPF operator. More details in appendix. \textit{Right figure:} Sample generation results. \textbf{Top:} Data samples \textbf{Middle:} MDS-based preimage samples \textbf{Bottom:} Weighted Fr\'{e}chet mean samples.}
    \label{fig:density}
\end{figure}

%% file: results.tex
\section{Experimental Results}
\begin{figure*}[b]
    \vspace{0.3em}
    \centering
    % \textcolor{cyan!50!white}{\fboxrule=1pt\fbox{\includegraphics[width=.225\textwidth, valign=t]{imgs/celeba_gmm_collage.png}}}
    % \textcolor{cyan!50!white}{\fboxrule=1pt\fbox{\includegraphics[width=.225\textwidth, valign=t]{imgs/celeba_glow_collage.png}}}
    % \textcolor{cyan!50!white}{\fboxrule=1pt\fbox{\includegraphics[width=.225\textwidth, valign=t]{imgs/celeba_2svae_collage.png}}}
    % \textcolor{cyan!50!white}{\fboxrule=1pt\fbox{\includegraphics[width=.225\textwidth, valign=t]{imgs/celeba_kpf_collage.png}}}
    \setlength{\tabcolsep}{1.5pt}
    \renewcommand{\arraystretch}{1.2}
    \begin{tabular}{cccc}
    \includegraphics[width=.24\textwidth, valign=t]{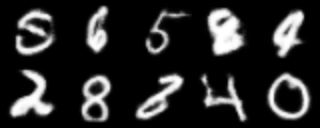} &
    \includegraphics[width=.24\textwidth, valign=t]{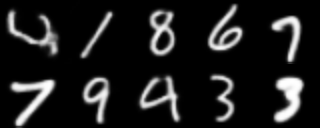} &
    \includegraphics[width=.24\textwidth, valign=t]{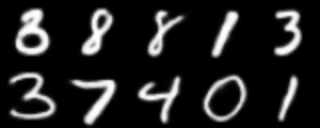} &
    \includegraphics[width=.24\textwidth, valign=t]{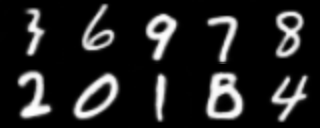}\\
    \includegraphics[width=.24\textwidth, valign=t]{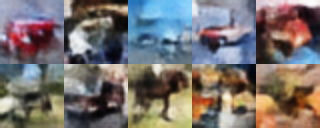} &
    \includegraphics[width=.24\textwidth, valign=t]{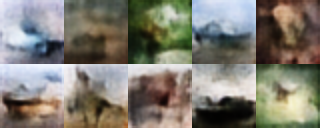} &
    \includegraphics[width=.24\textwidth, valign=t]{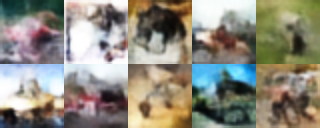} &
    \includegraphics[width=.24\textwidth, valign=t]{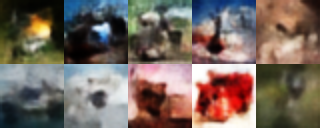}\\
    \includegraphics[width=.24\textwidth, valign=t]{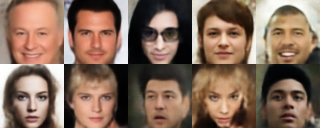} &
    \includegraphics[width=.24\textwidth, valign=t]{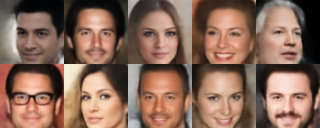} &
    \includegraphics[width=.24\textwidth, valign=t]{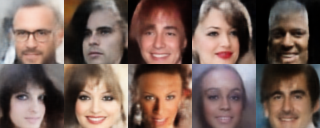} &
    \includegraphics[width=.24\textwidth, valign=t]{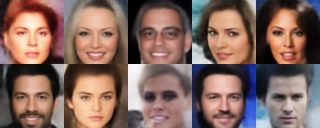}\\
     $\textrm{Two-stage VAE}$ & $\textrm{SRAE}_\textrm{Glow}$ & $\textrm{SRAE}_\textrm{GMM}$ & $\textrm{SRAE}_\textrm{NTK-kPF}$
    \end{tabular}
    \vspace{-8pt}
    \caption{Comparison of different sampling techniques using AE trained on CelebA 64x64. Left to right: samples of (1) Two-Stage VAE \citep{dai2018diagnosing} (2) $\textrm{SRAE}_{\rm Glow}$ \citep{Kingma2018Glow} (3) $\textrm{SRAE}_{\rm GMM}$ (4) $\textrm{SRAE}_\textrm{ NTK-kPF}$ using 10k latent points.}
    \label{fig:celeba}
    \vspace{-1em}
\end{figure*}

\textbf{Goals.} In our experiments, we seek to answer three questions: 
\begin{inparaenum}[\bfseries (a)] \label{first}
\item With sufficient data, can our method generate new data with comparable performance with other state-of-the-art generative models?
\item If only limited data samples were given, can our method still estimate the density with reasonable accuracy?
\item What are the runtime benefits, if any?
\end{inparaenum}

\textbf{Datasets/setup.} To answer the first question, we evaluate our method on standard vision datasets, including MNIST, CIFAR10, and CelebA, where the number of data samples is much larger than the latent dimension. We compare our results with other VAE variants (Two-stage VAE \citep{dai2018diagnosing}, WAE \citep{arjovsky2017wasserstein}, CV-VAE \citep{Ghosh2020From}) and flow-based generative models (Glow \citep{Kingma2018Glow}, CAGlow \citep{Liu2019CAflow})
The second question is due to 
the broad use of kernel methods in small sample size settings. For this more challenging case, we randomly choose 100 training samples ($<1$\% of the full dataset) from CelebA and evaluate the quality of generation compared to other density approximation schemes. We also use a dataset of T1 Magnetic Resonance (MR) images from the Alzheimer's Disease Neuromaging Initiative (ADNI) study. 

% The purpose of the first set of experiments is to show that the proposed method yields competitive measures on the task of image generation compared to other non-adversarial generative methods while enjoying the benefit of {\it one step} density estimation.
% Our second setting is motivated by \cite{Arora2020Harnessing}, where kernel methods consistently outperform neural networks in a small data setting. Our goal is to demonstrate that unlike other density approximation methods based on deep neural architectures, our proposed method produces reasonable sample generation both quantitative and qualitatively when only few data were given.

% (on MNIST using kPF with random projection with RBF kernel, we can achieve $19.1$ FID score)

\textbf{Distribution transfer with many data samples.}
We evaluate the quality by calculating the Fr\'{e}chet Inception Distance (FID) \citep{Martin2017ttur} with 10K generated images from each model. Here, we use a pretained regularized autoencoder \citep{Ghosh2020From} with a latent space restricted to the hypersphere (denoted by SRAE) to obtain \textit{smooth} latent representations. We compare our kPF to competitive end-to-end deep generative baselines (i.e. flow and VAE variants) as well as other density estimation models over the same SRAE latent space. For the latent space models, we experimented with Glow \citep{Kingma2018Glow}, VAE, Gaussian mixture model (GMM), and two proposed kPF operators with Gaussian kernel (RBF-kPF) and NTK (NTK-kPF) as the input kernel. The use of NTK is motivated by promising results at the interface of kernel methods and neural networks \citep{jacot2018NTK, Arora2020Harnessing}. Implementation details are included in the appendix.

\begin{figure*}[t]
    \centering
    \includegraphics[scale=0.24]{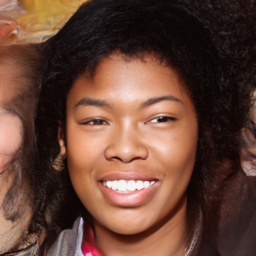}
    \includegraphics[scale=0.24]{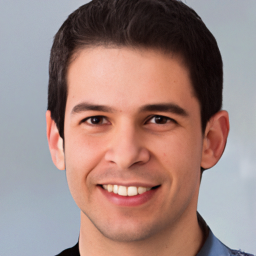}
    \includegraphics[scale=0.24]{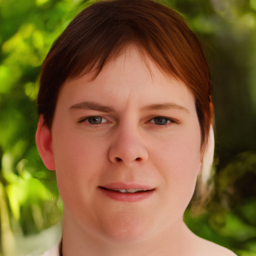}
    \includegraphics[scale=0.24]{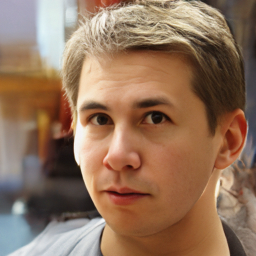}
    \includegraphics[scale=0.24]{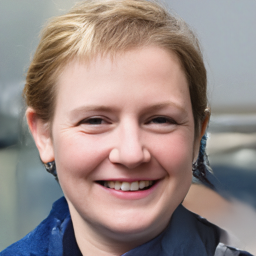}
    \includegraphics[scale=0.24]{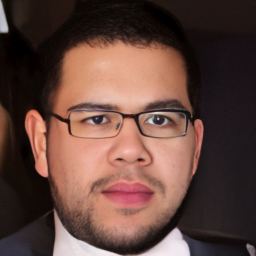}
    \vspace{-8pt}
    \caption{Representative samples from learned kPF on pre-trained NVAE latent space}
    \vspace{-1em}
    \label{fig:nvae}
\end{figure*}

Comparative results are shown in Table \ref{tab:fid_table}. We see that for images with structured feature spaces, e.g., MNIST and CelebA, our method matches other non-adversarial generative models, which provides evidence in support of the premise that the forward operator 
can be simplified. 
\begin{figure}[b]
    \centering
    \vspace{-1em}
    \begin{minipage}{0.55\textwidth}
               {
       \footnotesize
       \centering
            \begin{tabular}{cccc} 
            \specialrule{1pt}{1pt}{0pt} \rowcolor{azure!20}
            &  MNIST  &  CIFAR  &  CelebA \\
            \specialrule{1pt}{0pt}{1pt}
                $\textrm{Glow}^\ddag$    & 25.8         & -          & 103.7\\
                $\textrm{CAGlow}^\ddag$    & 26.3         & -          & 104.9\\
                Vanilla VAE         & 36.5 & 111.0      & 52.1\\
                $\textrm{CV-VAE}^\dagger$ & 33.8          & 94.8       & 48.9\\
                $\textrm{WAE}^\dagger$    & 20.4          & 117.4      & 53.7\\
                Two-stage VAE       & 16.5          & 110.3      & 44.7\\
            \hline 
            % $\textrm{SAE}_{\textit{rand}}$ & 55.6 & 187.8 & 86.5\\
                $\textrm{SRAE}_{\textrm{Glow}}$ & \textbf{15.5} &  85.9 & \textbf{35.0} \\
                $\textrm{SRAE}_{\textrm{VAE}} $ & 17.2 & 198.0 & 48.9\\
                $\textrm{SRAE}_{\textrm{GMM}} $ & 16.7 & 79.2 & 42.0\\
                $\textrm{SRAE}_{\textrm{RBF-kPF}} (\textit{ours})$ &  19.7 & 77.9 & 41.9\\
                $\textrm{SRAE}_{\textrm{NTK-kPF}} (\textit{ours})$ & 19.5 & \textbf{77.5} & 41.0 \\
            % $\textrm{SRAE}_{\textit{NTK-kPF Nystr\"{o}m}} (\textit{ours})$ &  &  & \\
                %$\textrm{SAE}_{\textit{NNGP-kPF 10k}} (\textit{ours})$ & $\textbf{19.1}$ & - & - \\
            \bottomrule
            \end{tabular}
        }
        
       \captionof{table}{ Comparative FID values. SRAE indicates an autoencoder with hyperspherical latent space and spectral regularization following \cite{Ghosh2020From}. Results reported from $\ddag$: \cite{Liu2019CAflow}. $\dagger$: \cite{Ghosh2020From}.}
        \label{tab:fid_table}
    \end{minipage}
    \hfill
    \begin{minipage}{0.42\textwidth}
        \includegraphics[width=\linewidth, trim=10 0 10 10, clip]{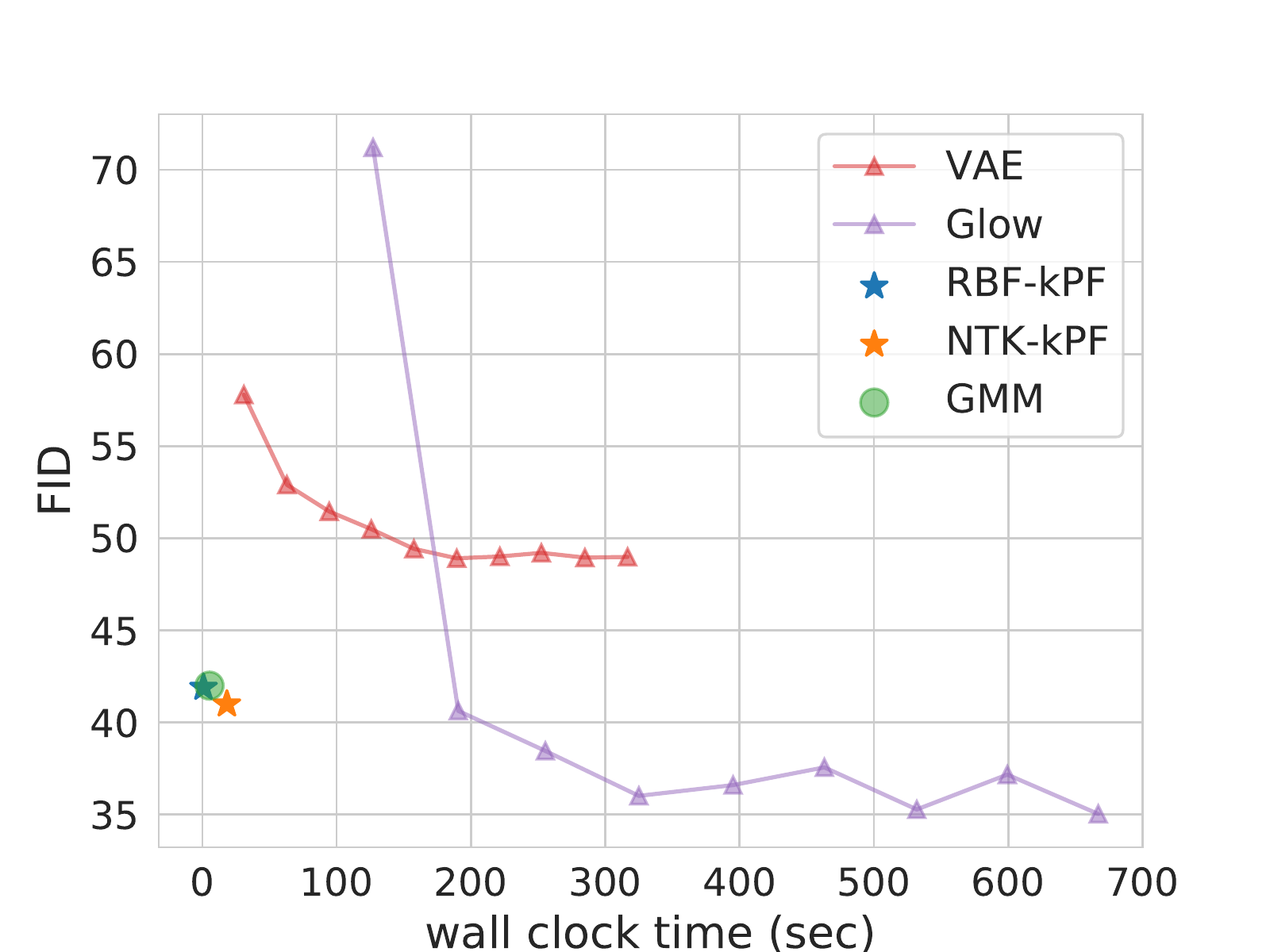}
          \vspace{-14pt}
        \captionof{figure}{FID  \textit{versus} training time for latent space models. All models are learned on the latent representations encoded by the same pre-trained SRAE.} %Our $10k$ and $\textrm{nystr\"om}$ model shows $>40$x and $\sim4$x speedup, respectively, compared to other models trained on latent space. }
        \label{fig:runtime}
    \end{minipage}
\end{figure}
%
% \begin{figure*}[th]
%     \centering
%     \vspace{1em}
%     \setlength{\fboxsep}{0pt}
%     \setlength{\fboxrule}{1.5pt}
%     \textcolor{red}{\fbox{\includegraphics[width=0.045\textwidth, height=0.22\textwidth]{imgs/mnist_nearest_generation.png}}}
%     \includegraphics[width=.20\textwidth, height=0.22\textwidth]{imgs/mnist_nearest_top5.png}
%     \hspace{0.05in}
%     \textcolor{red}{\fbox{\includegraphics[width=0.045\textwidth, height=0.22\textwidth]{imgs/cifar_nearest_generation.png}}}
%     \includegraphics[width=.20\textwidth, height=0.22\textwidth]{imgs/cifar_nearest_top5.png}
%     \hspace{0.05in}
%     \textcolor{red}{\fbox{\includegraphics[width=0.045\textwidth, height=0.22\textwidth]{imgs/celeba_nearest_generation.png}}}
%     \includegraphics[width=.20\textwidth, height=0.22\textwidth]{imgs/celeba_nearest_top5.png}
%     \vspace*{-0.5em}
%     \caption{Generations (in red box) and training samples corresponding to the top-5 latent representations used in geodesic interpolation. It can be observed that the samples with top kernel values indeed share high visual similarity.}
%     \label{fig: generation_topk}
% \end{figure*}
% \vspace{1em}
%
Further, we present qualitative results on all datasets (in Fig. \ref{fig:celeba}), where we compare our kPF operator based model with other density estimation techniques on the latent space. Observe that our model
generates comparable visual results as $\textrm{SRAE}_{\rm Glow}$.

% \begin{wrapfigure}{r}{0.45\textwidth}
%     \centering
%     \vspace{-1em}
%       \includegraphics[width=\linewidth, trim=10 10 10 10, clip]{imgs/training_time_comparison.png}
%       \vspace{-14pt}
%     \caption{ Comparison of additional training time for density estimation on latent space. All models were fitted on the latent representations of CelebA images encoded by an SRAE.} %Our $10k$ and $\textrm{nystr\"om}$ model shows $>40$x and $\sim4$x speedup, respectively, compared to other models trained on latent space. }
%     \label{fig:runtime}
%     \vspace{-2em}
% \end{wrapfigure}

Since kPF learns the distribution on a pre-trained AE latent space for image generation, using a more powerful AE can 
offer improvements in generation quality. In Fig. \ref{fig:nvae}, we present representative images by learning 
our kPF on NVAE \cite{vahdat2020NVAE} latent space, pre-trained on the FFHQ dataset. NVAE builds a hierarchical prior and achieves state-of-the-art generation quality among VAEs. We see that kPF can indeed generate high-quality and diverse samples with the help of NVAE encoder/decoder. In fact, any AE/VAE may be substituted in, assuming that the latent space is \textit{smooth}.\\*[3pt]
\noindent {\bf Summary:} When a sufficient number of samples are available, our algorithm performs as well as the alternatives, which is attractive given the efficient training.  In Fig. \ref{fig:runtime}, we present comparative result of FIDs with respect to the training time. Since kPF can be computed in closed-form, it achieves significant training efficiency gain compared to other deep generative methods while delivering competitive generative quality.

\textbf{Distribution transfer with limited data samples.} Next, we present our evaluations when only a limited number of samples are available. Here, each of the density estimators was trained on latent representations of the same set of 100 randomly sampled CelebA images, and 10K images were generated to evaluate FID (see Table \ref{tab:fid_table_limited}).
Our method outperforms Glow and VAE, while offering competitive performance with GMM. Surprisingly, GMM remains a strong baseline for both tasks, which agrees with results in \citet{Ghosh2020From}. However, note that GMM is restricted by its parametric form and is less flexible than our method (as shown in Fig \ref{fig:density}).

\setlength{\intextsep}{10pt}%
\setlength{\columnsep}{7pt}
\begin{wraptable}{r}{0.51\textwidth}
   \centering
%   \vspace{-1em}
   {\small
    \begin{tabular}{ccccc} 
      \specialrule{1pt}{1pt}{0pt} \rowcolor{azure!20}
        VAE &  Glow  &  GMM  &  RBF-kPF & NTK-kPF \\
     \specialrule{2pt}{0pt}{1pt}
        59.3 & 77.0 & 39.6 & 40.6 & 40.9\\
        \bottomrule
        \end{tabular}
    }
    \vspace{-1em}
   \caption{FID values for few samples setting density approximation on CelebA. }
   \vspace{-1em}
    \label{tab:fid_table_limited}
\end{wraptable}

Learning from few samples is  common in biomedical  applications where acquisition is costly. Motivated 
by interest in making synthetic but 
statistic preserving data (rather than the 
real patient records) 
publicly available to researchers 
(see NIH N3C Data Overview), 
we present results on generating high-resolution
$(160\times 196\times 160)$ brain images:  $183$ samples from group AD
(diagnosed as Alzheimer's disease) and $291$ samples from group CN (control normals).
For $n = 474 \ll d = 5017600$, using our kernel operator, we can generate high-quality samples 
that are in-distribution. We present comparative results with VAEs.
The generated samples in Fig. \ref{fig: mr_gen} clearly show that our method generates sharper images.
%which can indicate the generated distribution is more aligned to the data distribution.
%
To check if the results are also scientifically meaningful, 
we test consistency between group difference testing (i.e., cases versus controls differential tests on each voxel) on the real images (groups were AD and CN) and   
the same test was performed on the generated samples (AD and CN groups), 
%indeed represent the original data distribution
using a FWER corrected two-sample $t$-test \cite{ashburner2000voxel}. 
%between samples generated using images from either group CN or group AD.
The results (see Fig \ref{fig: mr_gen}) show that while there is a deterioration in regions
identified to be affected by disease (different across groups), 
many statistically-significant regions from  tests on the real images are preserved in  voxel-wise tests
on the generated images.\\*[3pt]
\noindent {\bf Summary:} We achieve improvements in the small sample size setting compared to other generative methods. This is useful in many data-poor settings. For larger datasets, our method still compares competitively with alternatives, but with a smaller resource footprint.

\begin{figure}[b!]
    \centering
    \begin{minipage}{.63\linewidth}
      \centering
      \includegraphics[width=0.98\textwidth]{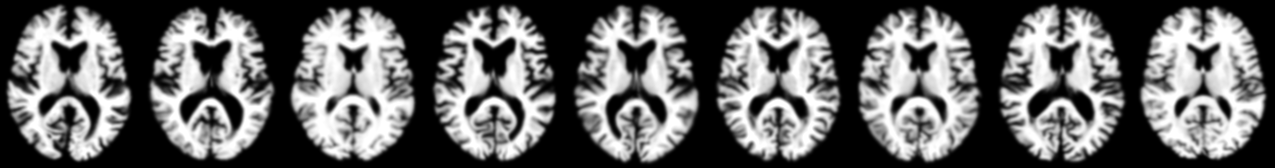}
      \includegraphics[width=0.98\textwidth]{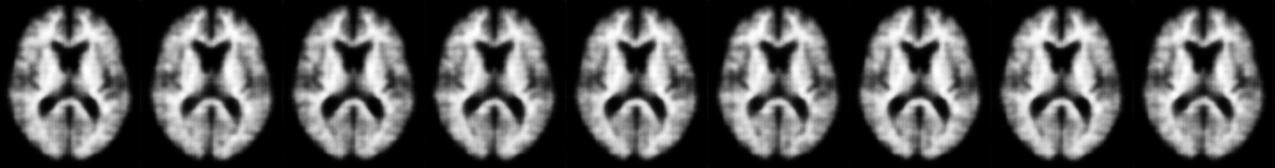}
      \includegraphics[width=0.98\textwidth]{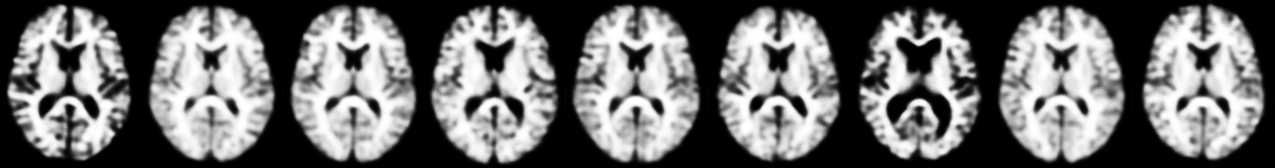}
    \end{minipage}%
    \begin{minipage}{0.25\linewidth}
        \centering
        \includegraphics[width=0.95\textwidth]{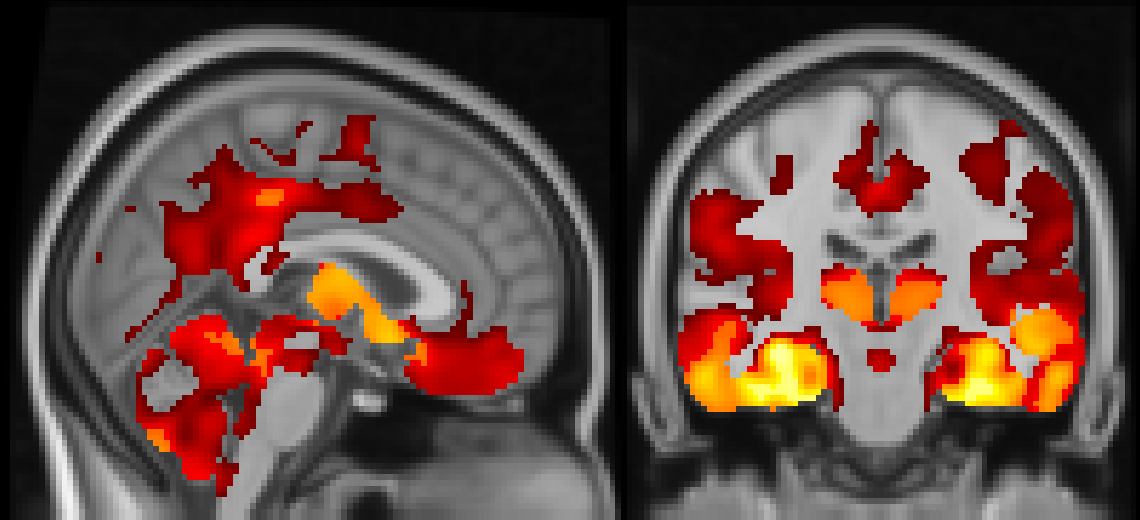}
        \includegraphics[width=0.95\textwidth]{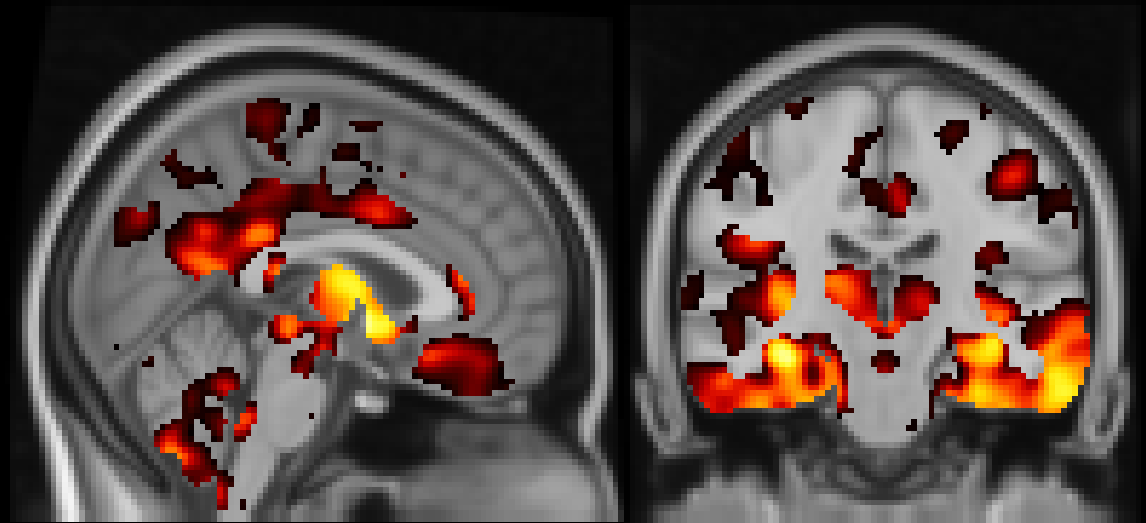}
        \includegraphics[width=\textwidth]{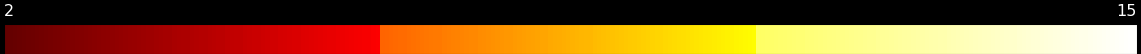}
    \end{minipage}
    \vspace{-8pt}
    \caption{\footnotesize {\bf Left.} {\it Top: data}, generated samples of {\it Middle: VAE , Bottom: kPF with SRAE}.
    {\bf Right.} Statistically significant regions {\it Top: voxel-wise tests on real data, Bottom: voxel-wise tests on generated samples} are shown in negative log $p$-value thresholded at $\alpha=0.01$.}
    \label{fig: mr_gen}
    \vspace{-0.8em}
\end{figure}

%% file: related_works.tex
\section{Limitations}
\label{sec:limitations}
%When a dataset allows
%a low-dimensional structured latent %space representation,
%a kernel Perron-Frobenius operator %is effective for simplifying
%forward operators in deep %generative models. 
%However, some direction we did not explore in our work is how to
Our proposed simplifications can be variously
useful, but deriving the density of the posterior given a mean embedding or providing an exact preimage for the generated sample in RKHS is unresolved at this time. 
While density estimation from kPF has been partially addressed in \citet{schuster2020kernelconditional}, finding the pre-image is often ill-posed. The weighted Fr\'{e}chet mean preimage only provides an approximate solution and evaluated empirically,and due to the interpolation-based sampling strategy, samples cannot be obtained beyond the convex hull of training examples. Making $Z$ and $X$ independent RVs also limits its use for certain downstream task such as representation learning or semantic clustering. Finally, like other kernel-based methods, the sup-quadratic memory/compute cost can be a bottleneck on large datasets and kernel approximation (e.g. \citep{rahimi2008random}) may have to be applied; we discuss this in appendix \ref{sec:nystrom}.

% \paragraph{Negative Societal Impact.} While the paper is focused on efficiency, we acknowledge that improvements in deep generative models can be used nefariously, including misinformation propagation. 

\section{Conclusions}
We show that using recent developments in regularized autoencoders, a linear kernel transfer operator can potentially be an efficient substitute for the forward operator in some generative models, if 
some compromise in capabilities/performance is acceptable. Our proposal, despite its simplicity, shows comparable empirical results to other generative models, while offering efficiency benefits. Results on brain images also show promise for applications to high-resolution 3D imaging data generation, which is being pursued actively in the community.

%% file: appendix.tex
\clearpage
\appendix
% \section{Appendix}

\section{Choice of Kernel is relevant yet flexible}\label{sec:choice_of_kernel}

In some cases, one would focus on identifying (finite) eigenfuntions and modes of the underlying operator \citep{williams2015extDMD,brunton2016sindy}. But rather than finding certain modes that best characterize the dynamics, we care most about minimizing the error of the transferred density and therefore whether the span of functions is rich/expressive enough. In particular, condition (iii) in Proposition \ref{prop:klus} requires an input RKHS spanned by sufficiently rich bases \citep{fukumizu2013kernel}. For this reason, the choice of kernel here 
cannot be completely ignored since it determines the family of functions contained in the induced RKHS.

To explore the appropriate kernel setup for our application, we empirically evaluate the effect of using several different kernels via a simple experiment on MNIST. We first train an autoencoder to embed MNIST digits on to a hypersphere $S^2$, then generate samples from kPF by the procedure described by Alg. \ref{alg:gen_algo} using the respective kernel function as the input kernel $k$. Subplot (b) and (c) in Fig. \ref{fig:mnist_gen} show the generated samples using Radial Basis Function (RBF) kernel and arc-cosine kernel, respectively. Observe that the choice of kernel has an influence on the sample population, and a kernel function with superior empirical behavior is desirable. 

Motivated by this observation, we evaluated the Neural Tangent Kernel (NTK) \citep{jacot2018NTK}, a well-studied neural kernel in recent works. We use it for a 
few reasons, 
\begin{inparaenum}[\bfseries (a)] \item NTK, in theory, corresponds to a trained infinitely-wide neural network, which spans a rich set of functions that satisfies the assumption. 
  %Therefore, NTK is spanned by a rich set of nonlinear functions
\item For well-conditioned inputs (i.e., no duplicates) on hypersphere,
  the positive-definiteness of NTK is proved in \citep{jacot2018NTK}. Therefore, invertibility of the Gram matrix $K_{\var{z}\var{z}} = \Phi^T\Phi$ is \textit{almost guaranteed} if the prior distribution $p_{\var{z}}$ is restricted on a hypersphere
\item
  NTK can be non-asymptotically approximated \citep{arora2019onexact}.
\item Unlike other parametric kernels such as RBF kernels, NTK is less sensitive to hyperparameters, as long as the number of units used is large enough \citep{arora2019onexact}.
\end{inparaenum} Subplot (d) of Fig. \ref{fig:mnist_gen} shows that kPF learned with NTK as input kernel is able to generate samples that are more consistent with the data distribution. However, we should note that NTK is merely a convenient choice of kernel that requires less tuning, and is not otherwise central to our work. In fact, as shown in our experiment in Tab. \ref{tab:fid_table}, a well-tuned RBF kernel may also achieve a similar performance. Indeed, in practice, any suitable choice of kernel may be conveniently adopted into the proposed framework without major modifications.

\begin{figure*}[h]
%\vspace*{-2em}
    \centering
    \textcolor{red!50!white}{\fboxrule=1pt\fbox{
    \includegraphics[width=0.22\linewidth]{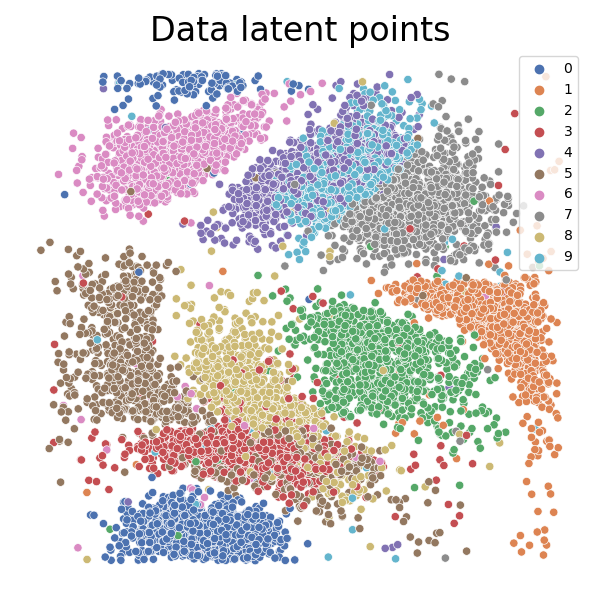}}}
    \textcolor{blue!50!white}{\fboxrule=1pt\fbox{
    \includegraphics[width=0.22\linewidth]{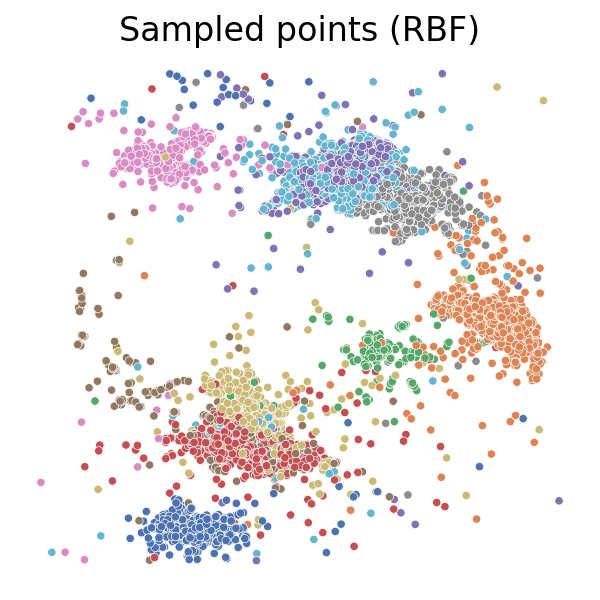}
    \includegraphics[width=0.22\linewidth]{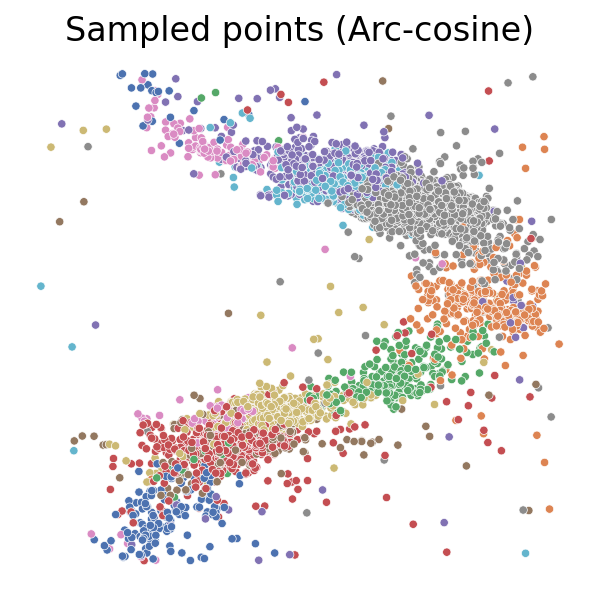}
    \includegraphics[width=0.22\linewidth]{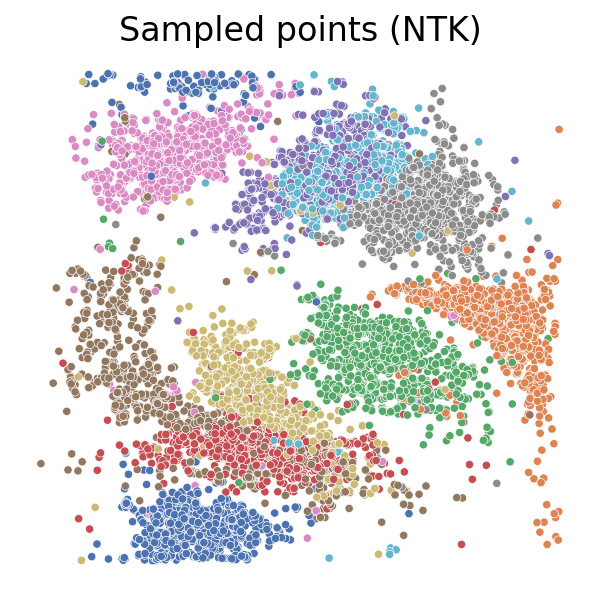}}}
    % \vspace*{-1em}
    \caption{\footnotesize 10k samples from MNIST dataset ({\it left to right}) (a) projected on $\mathbf{S}^2$ shown in $(\theta, \phi)$ using auto-encoder, and 10K generated samples from kPF with input kernel of type (b) RBF (c) arccos (d) NTK. Color of sampled points represents the class of their nearest training set neighbor in the output RKHS.}  \label{fig:mnist_gen}
\end{figure*}

\clearpage
\section{Density Estimation with kPF Operator}

The displayed transferred density with the kPF operator on toy data in Fig. 1 is \textit{approximated} using the empirical kernel conditional density operator (CDO) \citep{schuster2020kcdo}, since there is currently no known methods that can exactly reconstruct density from the transferred mean embeddings. The marginalized transferred density $p_{\mathcal{P}_\mathcal{E}}$ has the following form
\begin{equation}
p_{\mathcal{P}_\mathcal{E}} = C_\rho^{-1}\mathcal{P}_\mathcal{E}\mu_{\var{z}} = C_\rho^{-1}C_{\var{x}\var{z}}C_{\var{z}\var{z}}^{-1}\mu_{\var{z}},  
\end{equation}

where $C_\rho = E_{y \sim \rho}[l(y, \cdot) \otimes l(y, \cdot)]$ is the covariance operator of a independent reference density $\rho$ in $\spc{G}$. The above density function is also an element of RKHS $\mathcal{G}$, and therefore we can evaluate the density at $x^*$ by using the reproducing property $p_{\mathcal{P}_\mathcal{E}}(x^*) = \langle p_{\mathcal{P}_\mathcal{E}}, l(x^*, \cdot) \rangle$. The results in \citet{schuster2020kcdo} show that the empirical estimate of $p_{\mathcal{P}_\mathcal{E}}$ may be constructed from $m$ samples of the reference density $\{y_i\}_{i \in [m]}$ and $n$ training samples $\{x_i\}_{i \in [n]}$ and $\{z_i\}_{i \in [n]}$ as 
\begin{equation}
\label{eq:empirical_kcdo}
    \hat{p}_{\mathcal{P}_\mathcal{E}} = (\hat{C}_\rho + \alpha' I)^{-1}\hat{C}_{\var{x}\var{z}}(\hat{C}_{\var{z}\var{z}} + \alpha I)^{-1}\hat{\mu}_{\var{z}} = \sum_{i = 1}^{m} \beta_i l(y_i, \cdot)
\end{equation}

where 
\begin{equation}\beta = m^{-2}(L_{\var{y}} + \alpha'I)^{-2}L_{\var{y}\var{x}}(K_{\var{z}} + \alpha I)^{-1}\Phi^\top \hat{\mu}_{\var{z}}
\end{equation} and 
\begin{equation}
L_{\var{y}}= [l(y_i, y_j)]_{ij} \in \mathbb{R}^{m \times m} , L_{\var{y}\var{x}}= [l(y_i, x_j)]_{ij} \in \mathbb{R}^{m \times n}, K_{\var{z}}= [k(z_i, z_j)]_{ij} \in \mathbb{R}^{n \times n}
\end{equation}

In Fig. \ref{fig:inaccurate-density}, we use a uniform density in the square $(\pm 4, \pm 4)$ as the reference density $\rho$ and constructed $\hat{p}_{\mathcal{P}_\mathcal{E}}$ using $m = 10000$ samples $\{y_i\}_{i \in [m]}$ from $\rho$. Due to the form of the empirical kernel CDO, where the estimated density function $\hat{p}_{\mathcal{P}_\mathcal{E}}$ is a linear combination of $\{l(y_i, \cdot)\}_{i \in [m]}$ (as in Eq. \ref{eq:empirical_kcdo}), the approximation can be inaccurate if reference samples are relatively sparse and the densities are 
`sharp'. In those cases, to obtain a better density estimate, we may either increase the number of reference samples used to construct the empirical CDO (which can be computationally difficult due to the need to compute $(L_{\var{y}} + \alpha'I)^{-2}$), or, with some prior knowledge to the true distribution, choose a reference density which is localized around the ground truth density.

\begin{figure}[h]
    \centering
    \includegraphics[scale=0.1]{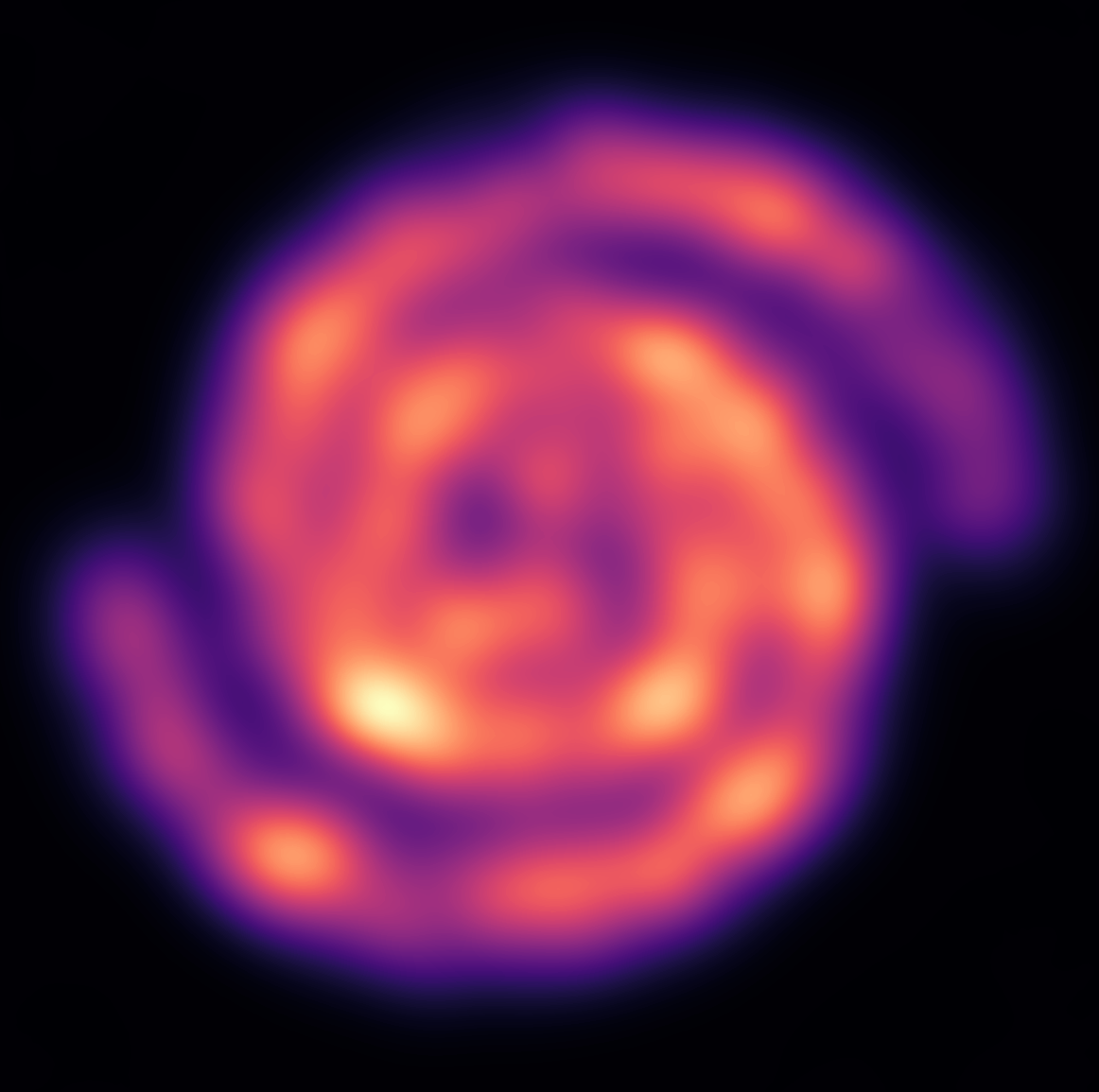}
    \includegraphics[scale=0.1]{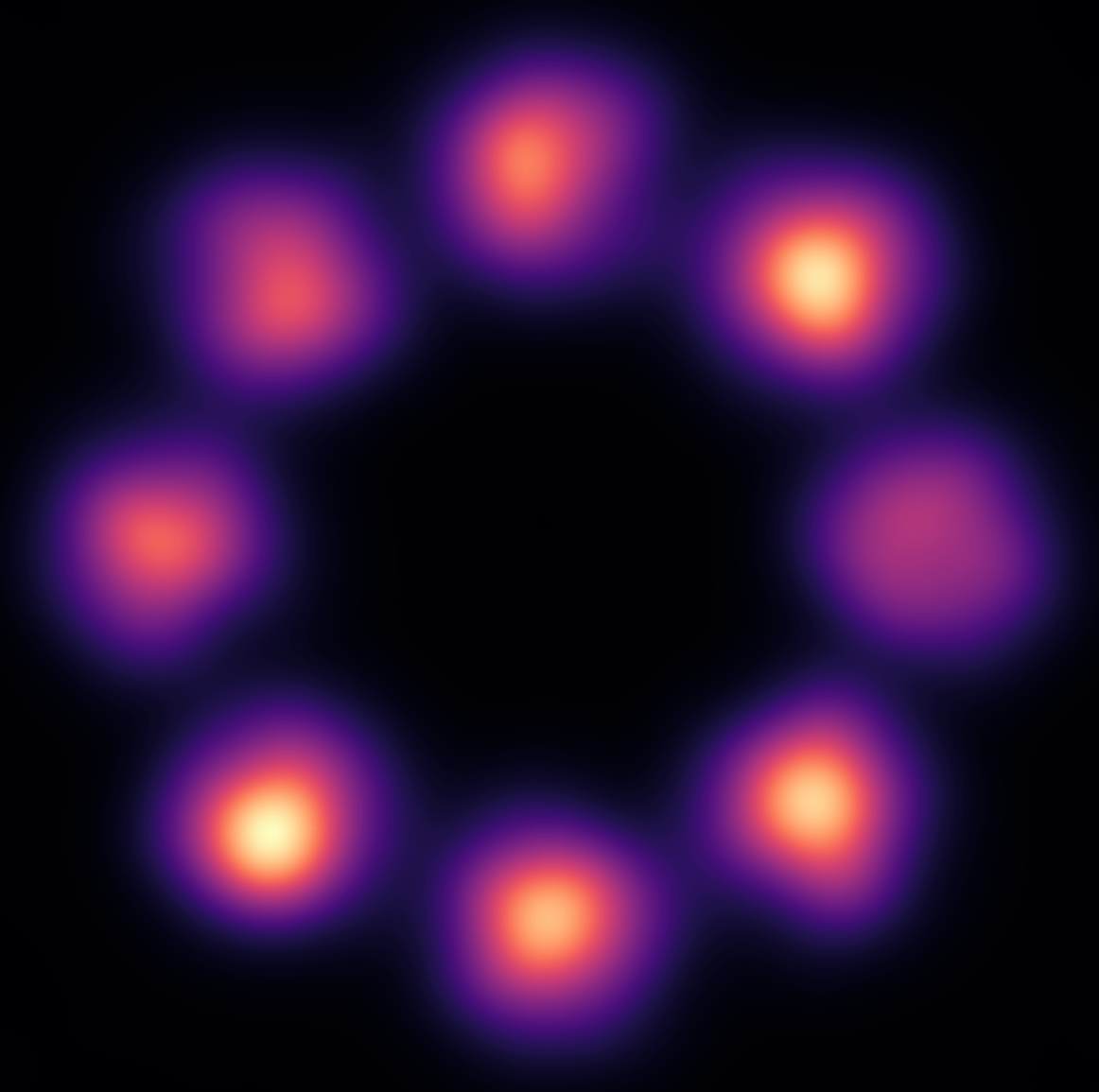}
    \includegraphics[scale=0.1]{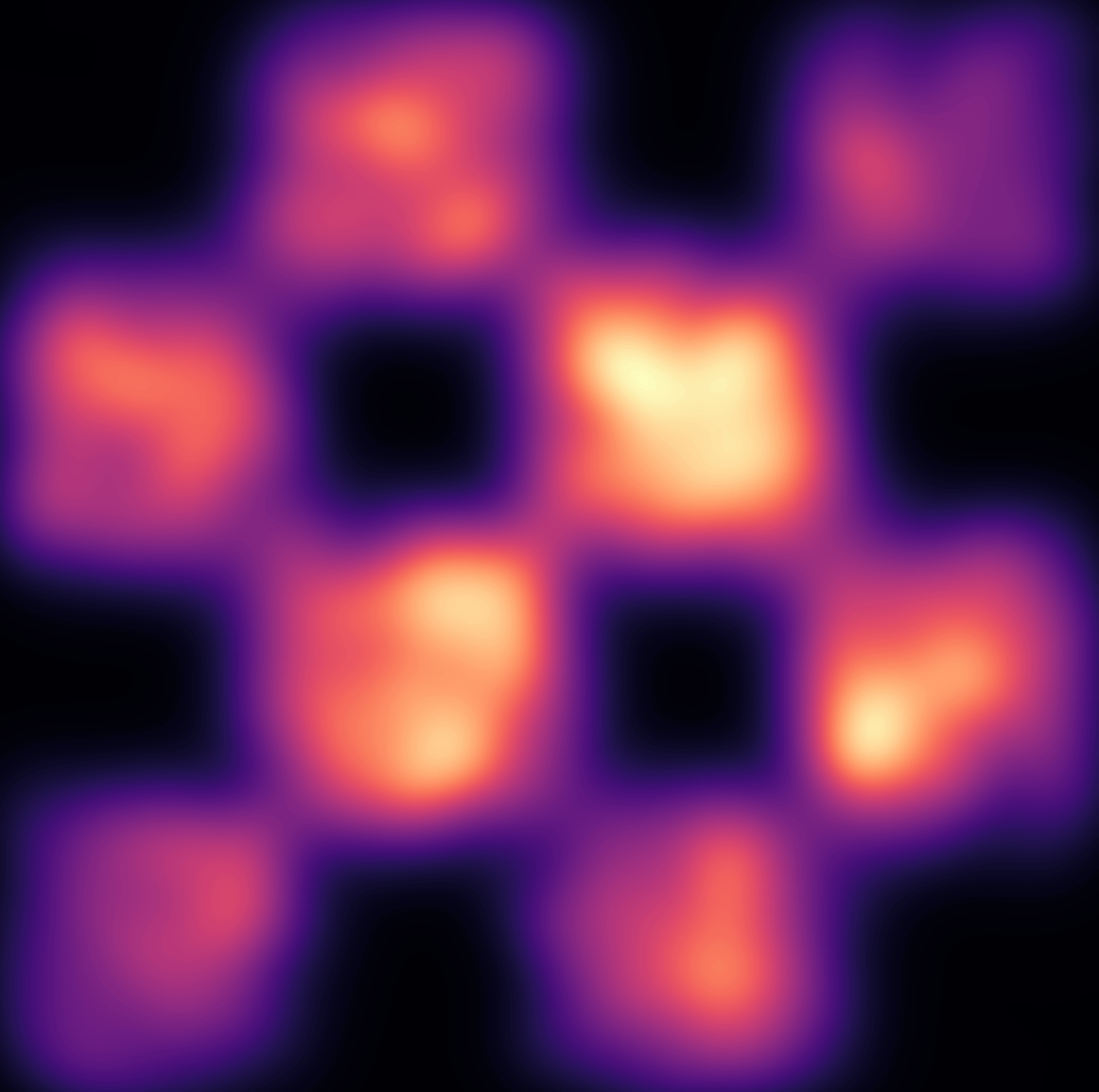}
    \includegraphics[scale=0.1]{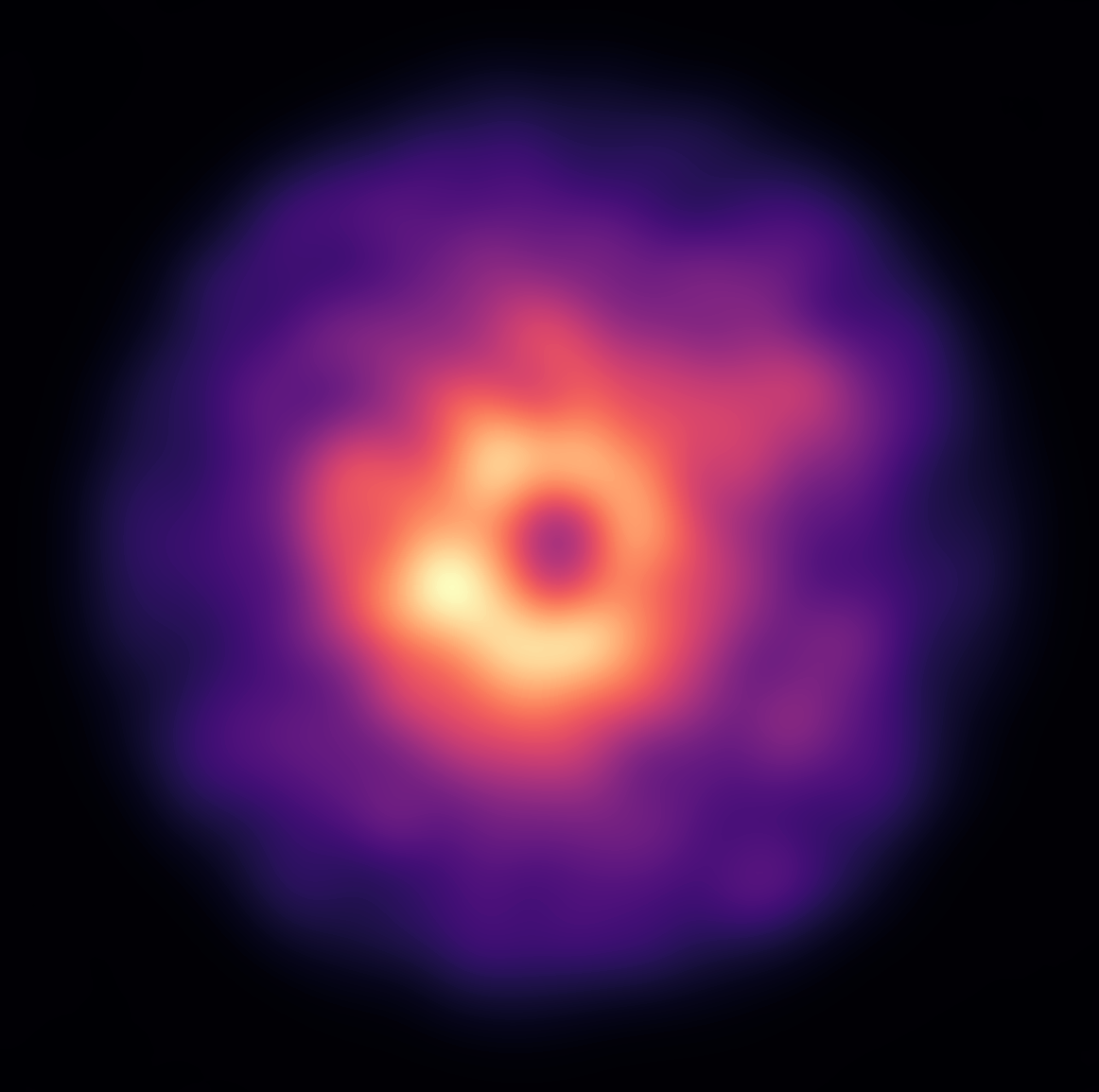}
    \caption{Inaccurate density estimation result from kernel CDO using 10k samples from uniform reference density $\rho$}
    \label{fig:inaccurate-density}
\end{figure}

Therefore, to show a more faithful density estimate of the transferred distribution for visualization purpose, we use a composite of the uniform density and the true density with weight $4:1$ as the reference density $\rho$ in Fig. \ref{fig:density}. In this case, approximately 20\% of the reference samples concentrates around the high-density areas of the true density, which helps to form a better basis for $\hat{p}_{\mathcal{P}_\mathcal{E}}$. Note that the choice of $\rho$ does not affect the transferred density embedding $\hat{\mathcal{P}}_\mathcal{E}\hat{\mu}_{\var{z}}$ since it is independent of $p_{\var{x}}$ and $p_{\var{z}}$. After this modification, the reconstructed density more accurately reflects the true density compared to GMM and GLOW, indicating the transferred distribution by kPF in RKHS matches better to the ground truth distribution. This is also reflected in the generated samples in Fig. \ref{fig:samples}, where samples generated by the kPF operator are clearly more aligned with the ground truth distribution.

\begin{figure}[h]
    %\begin{minipage}{0.55\textwidth}
        \centering
        \begin{tabular}{ m{1cm} m{2cm} m{2cm} m{2cm}  m{2cm}  }
            \textrm{GT} &
            \includegraphics[width=\linewidth]{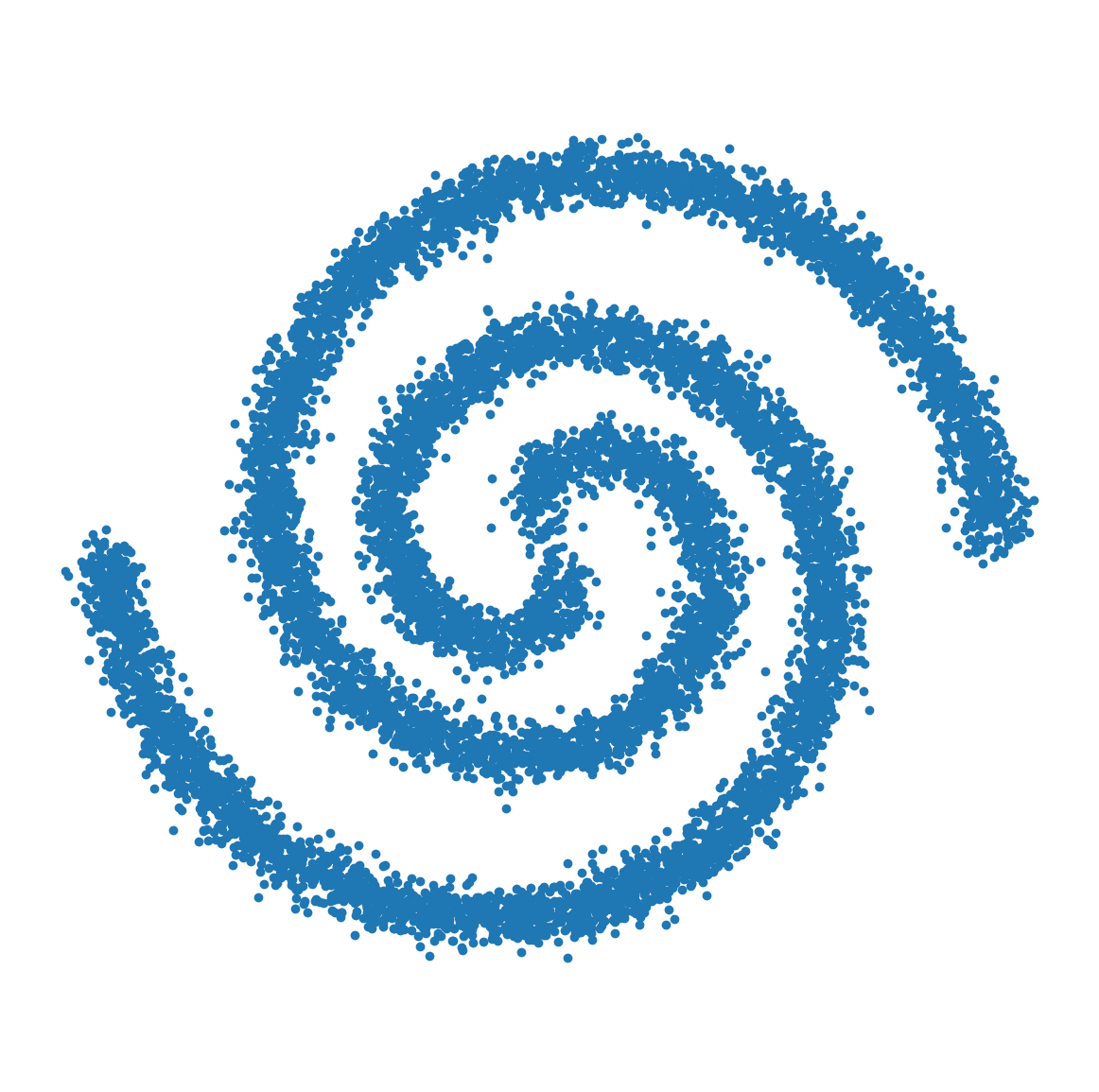} &
            \includegraphics[width=\linewidth]{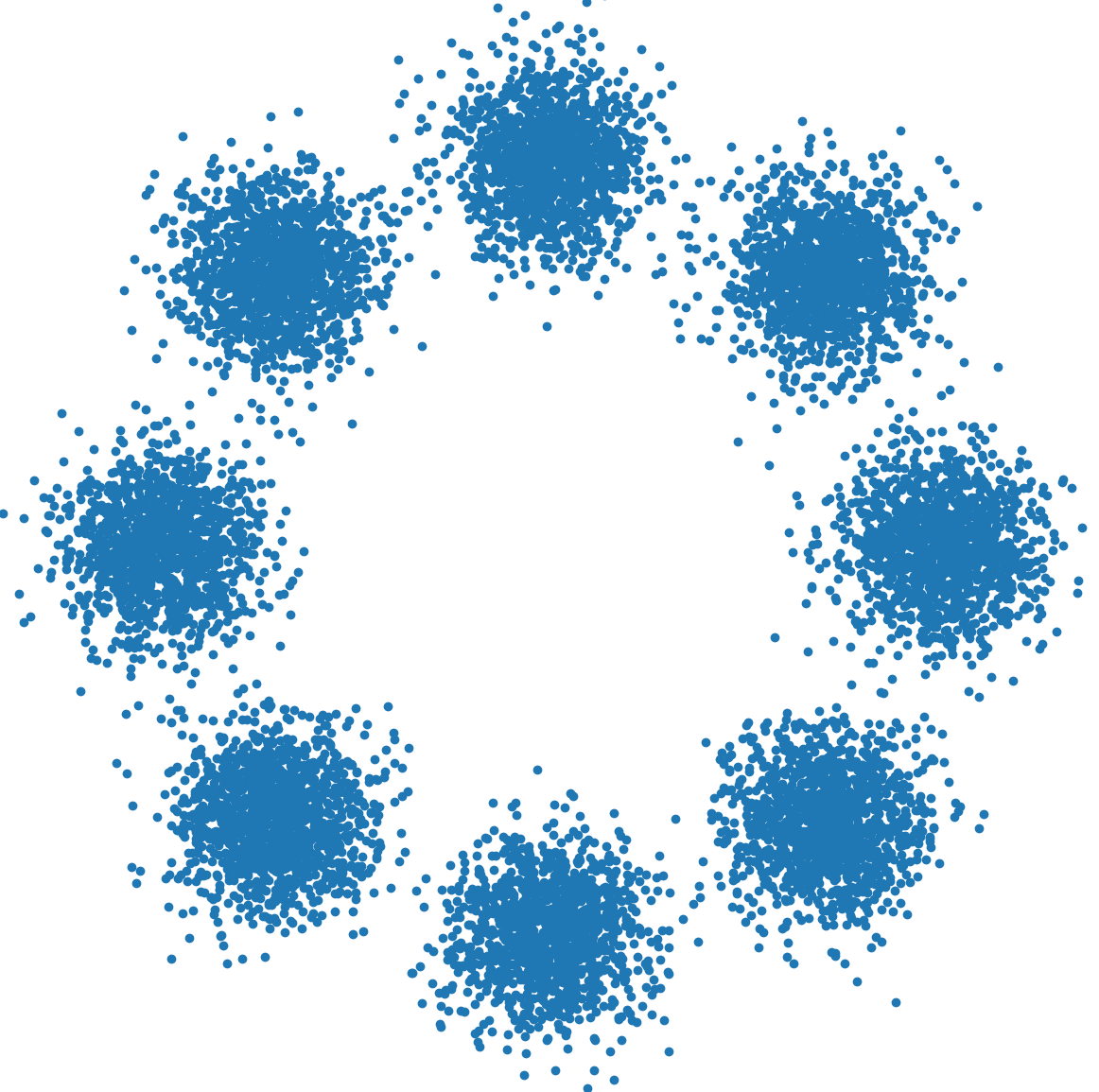} &
            \includegraphics[width=\linewidth]{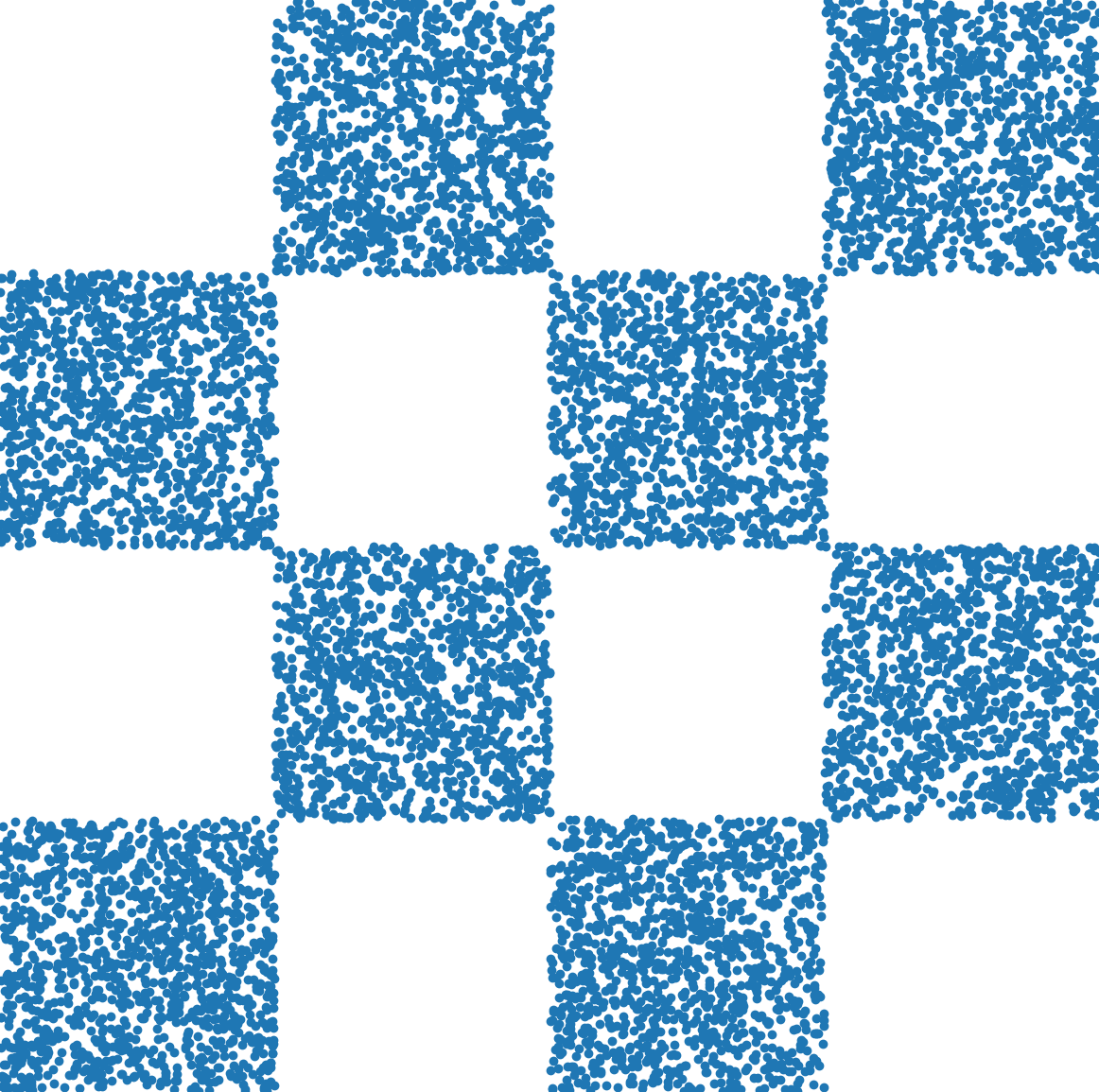} &
            \includegraphics[width=\linewidth]{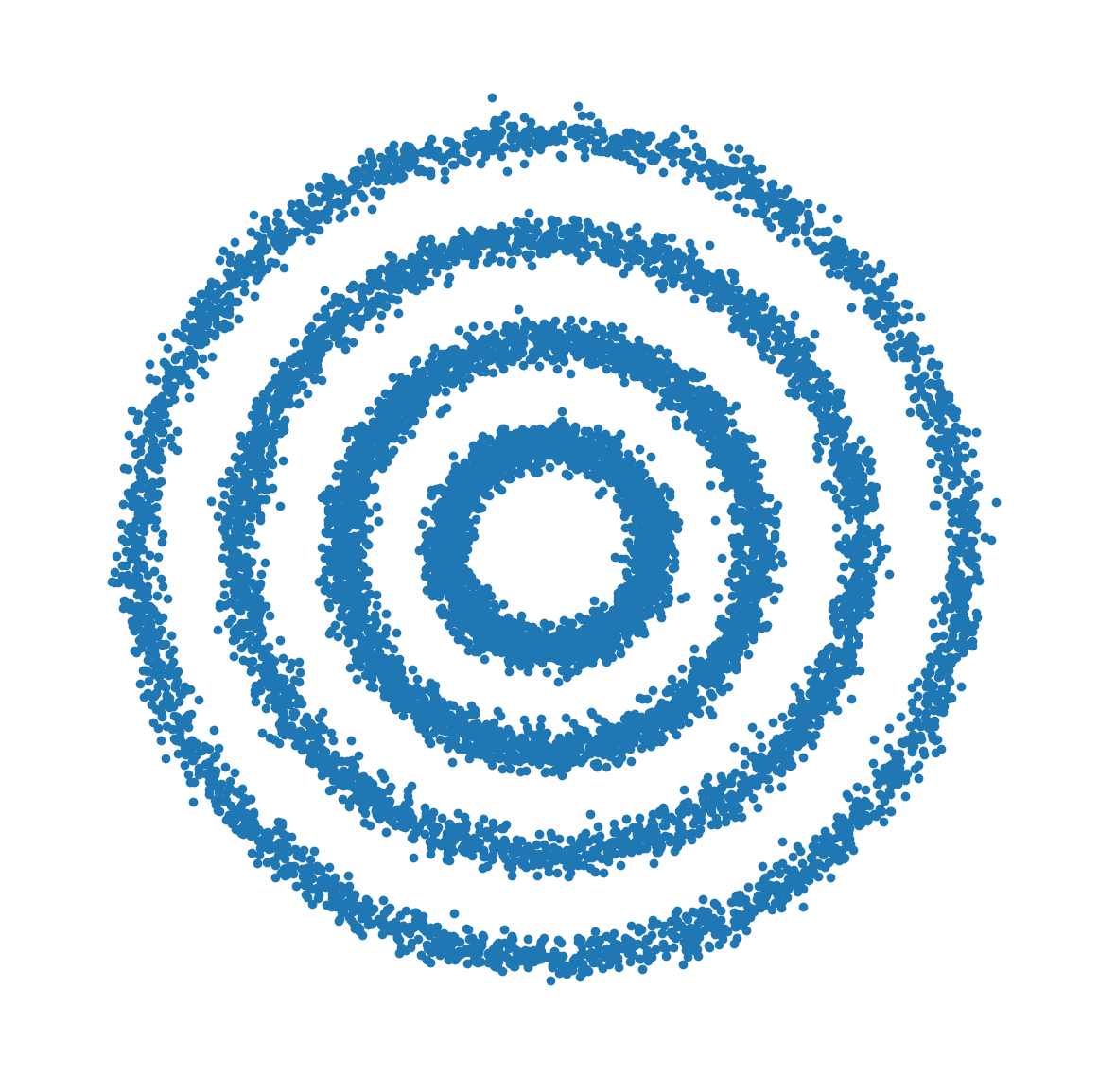}\\ 
            \textrm{GMM} &
            \includegraphics[width=\linewidth]{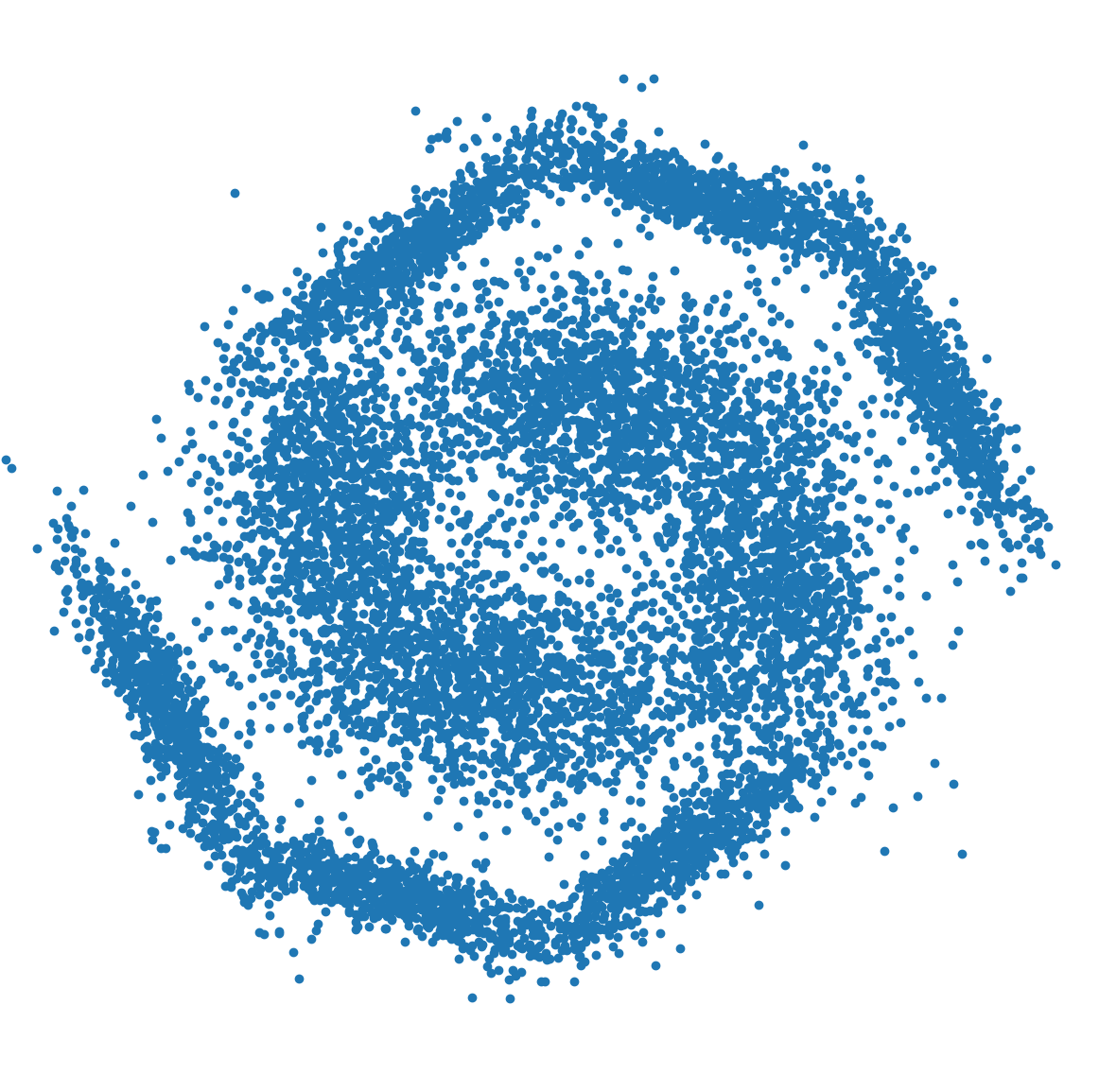} &
            \includegraphics[width=\linewidth]{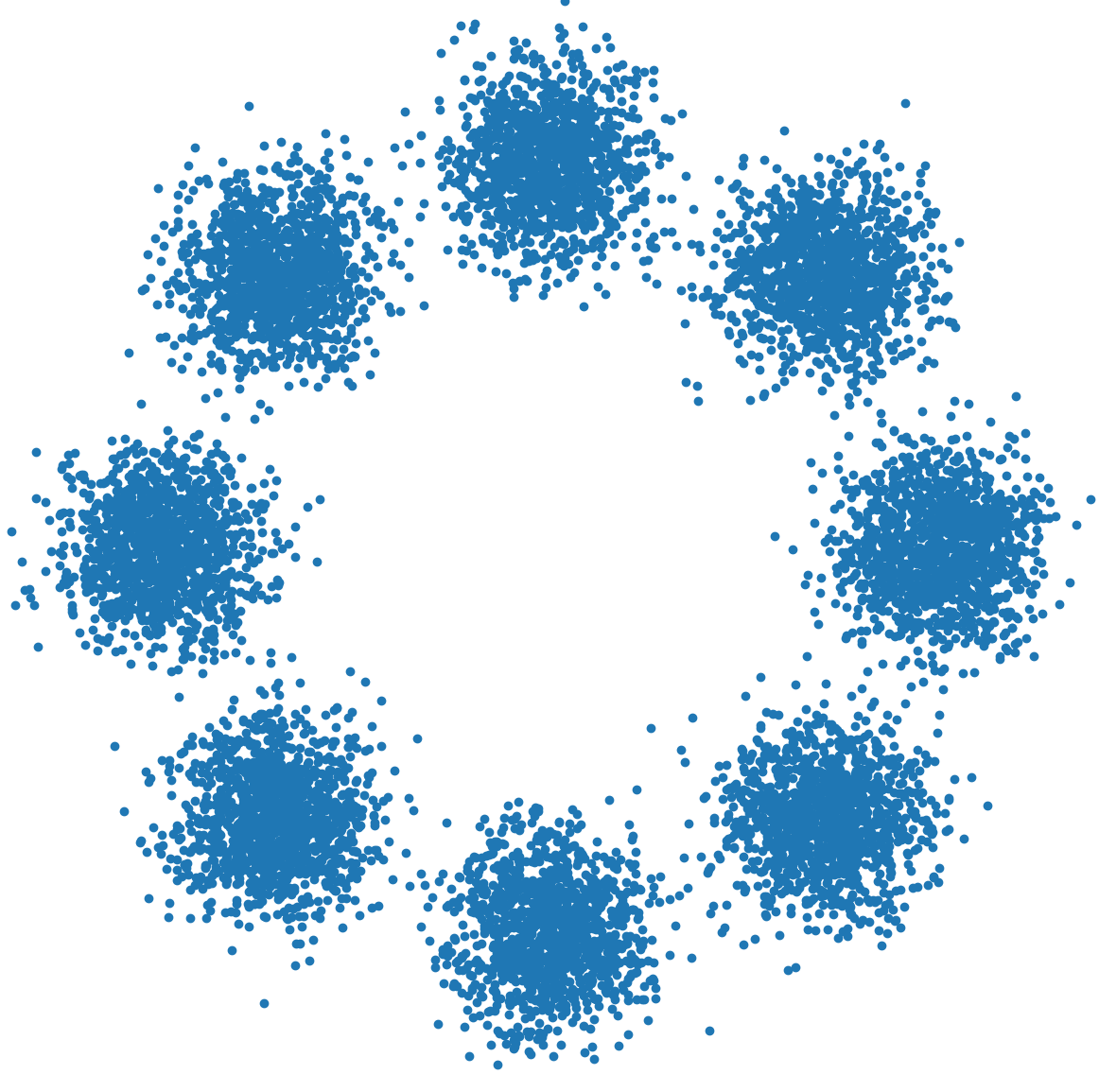} &
            \includegraphics[width=\linewidth]{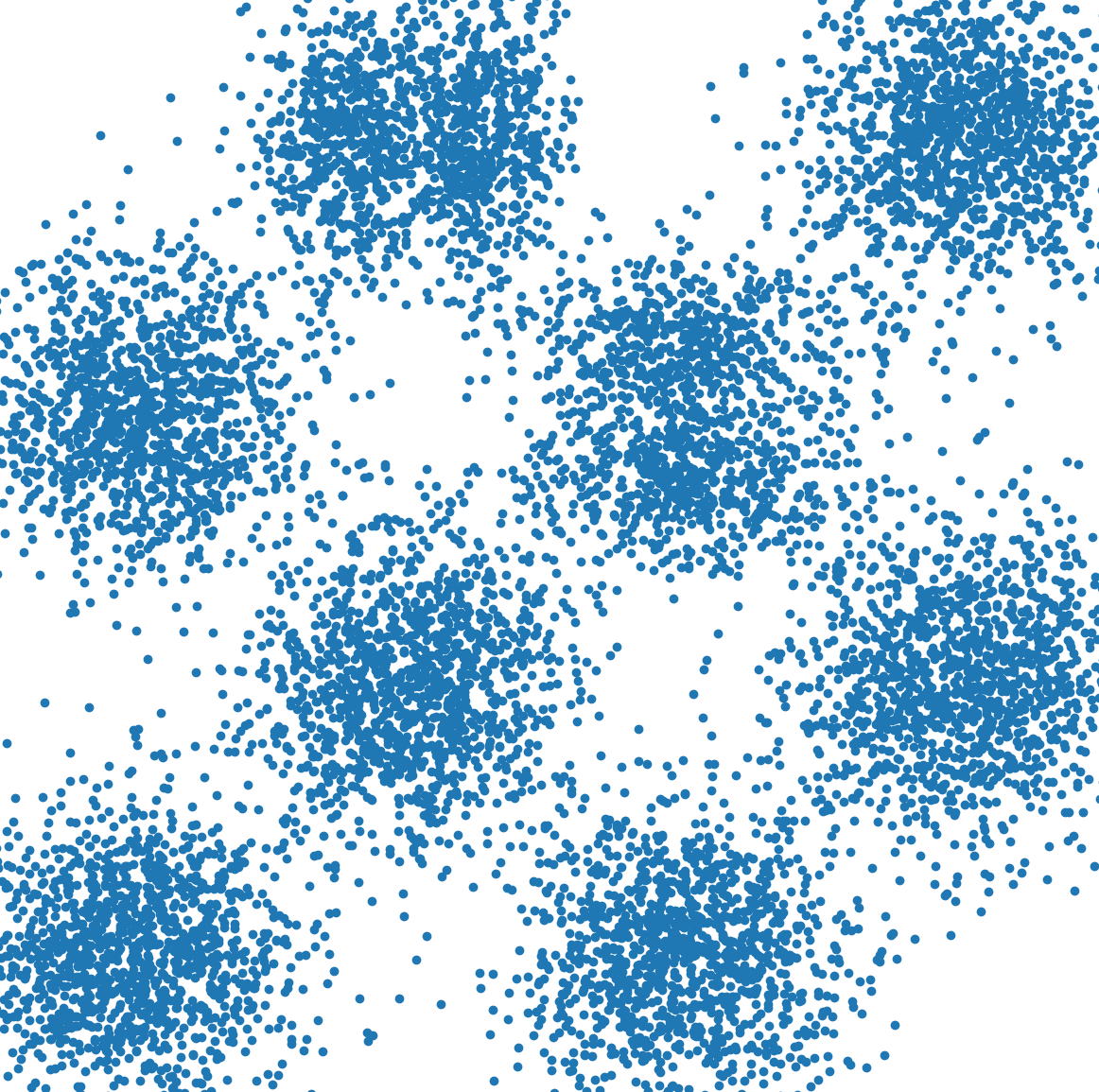} &
            \includegraphics[width=\linewidth]{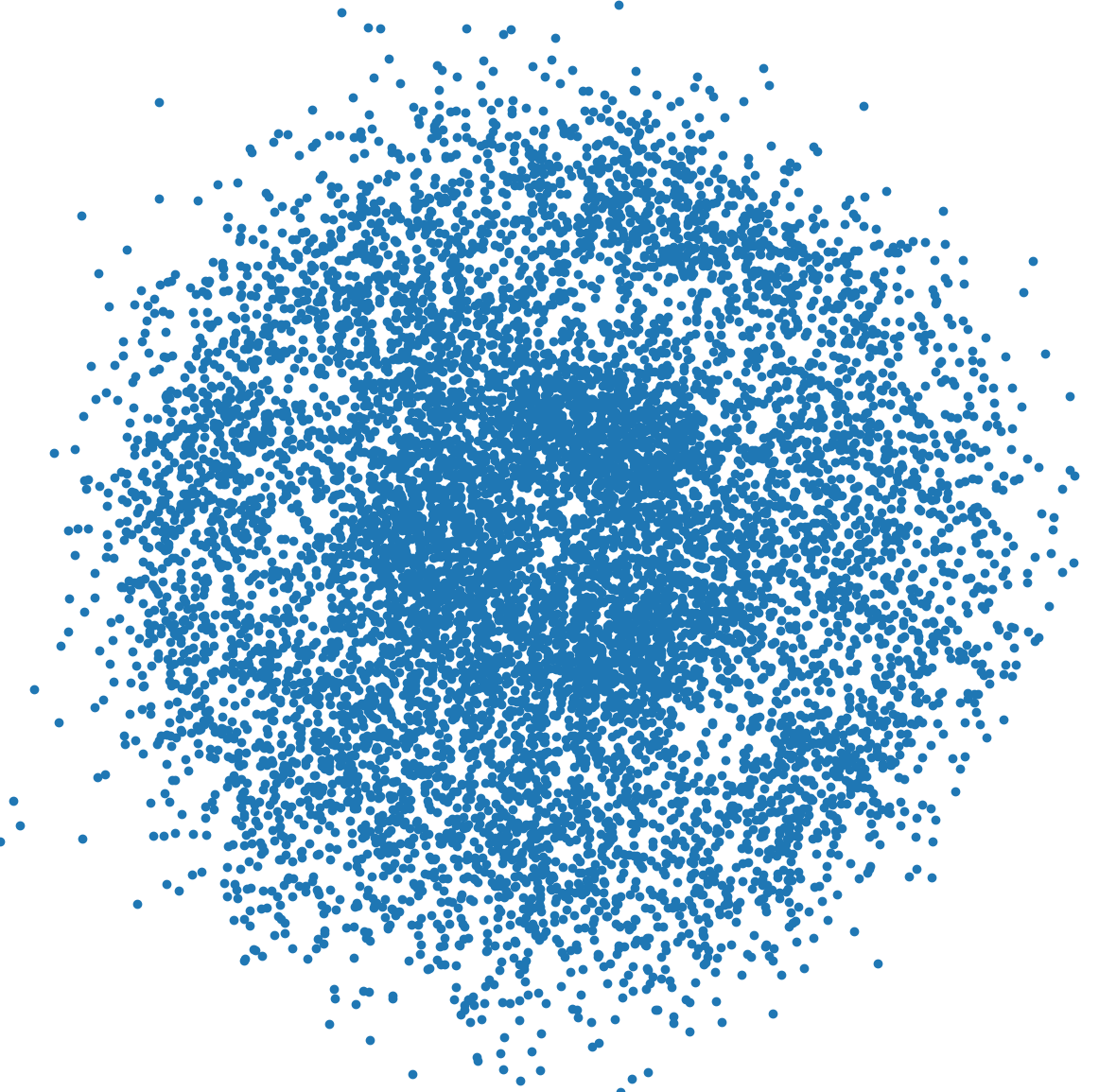}\\
            \textrm{Glow} &
            \includegraphics[width=\linewidth]{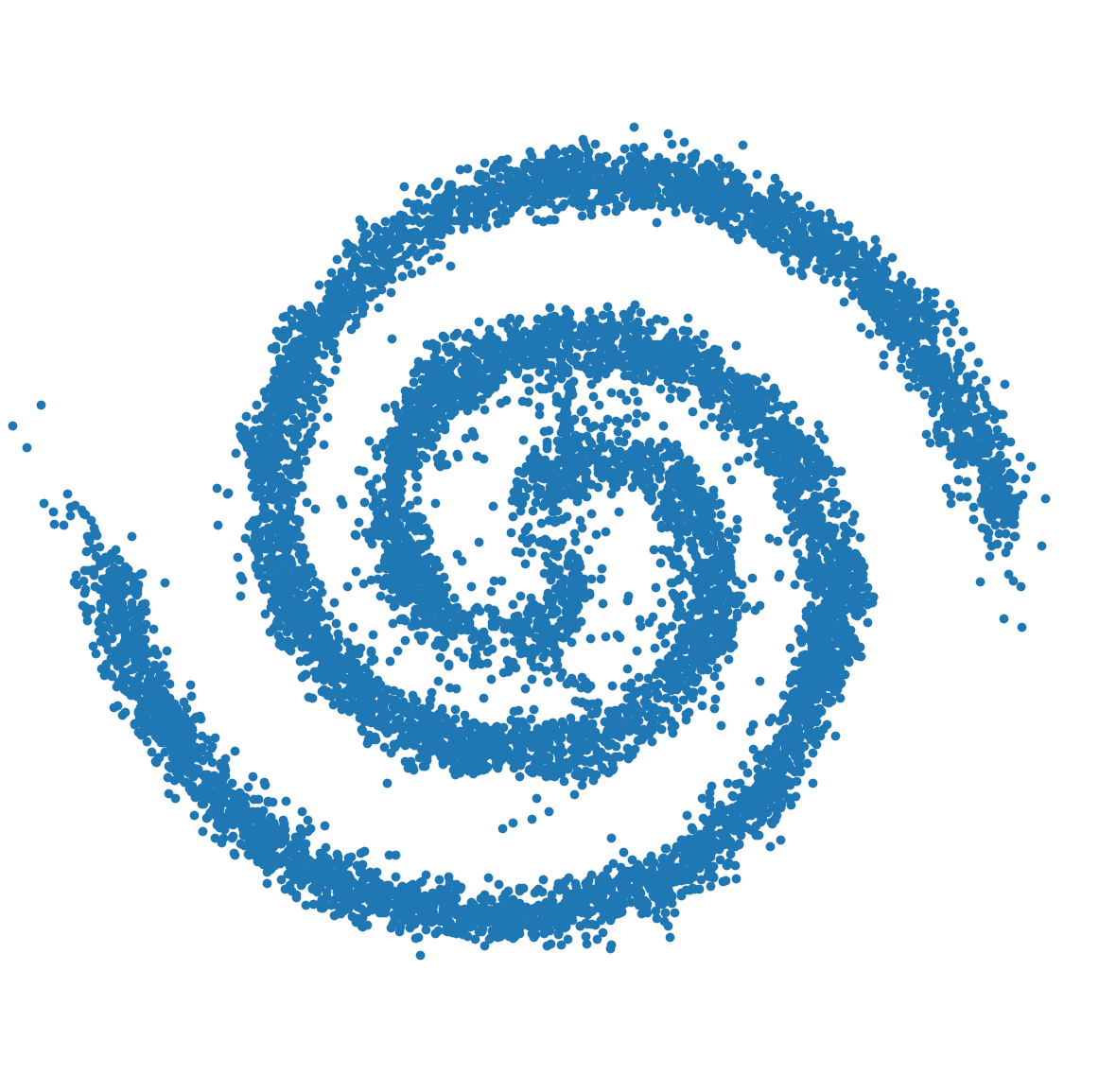} &
            \includegraphics[width=\linewidth]{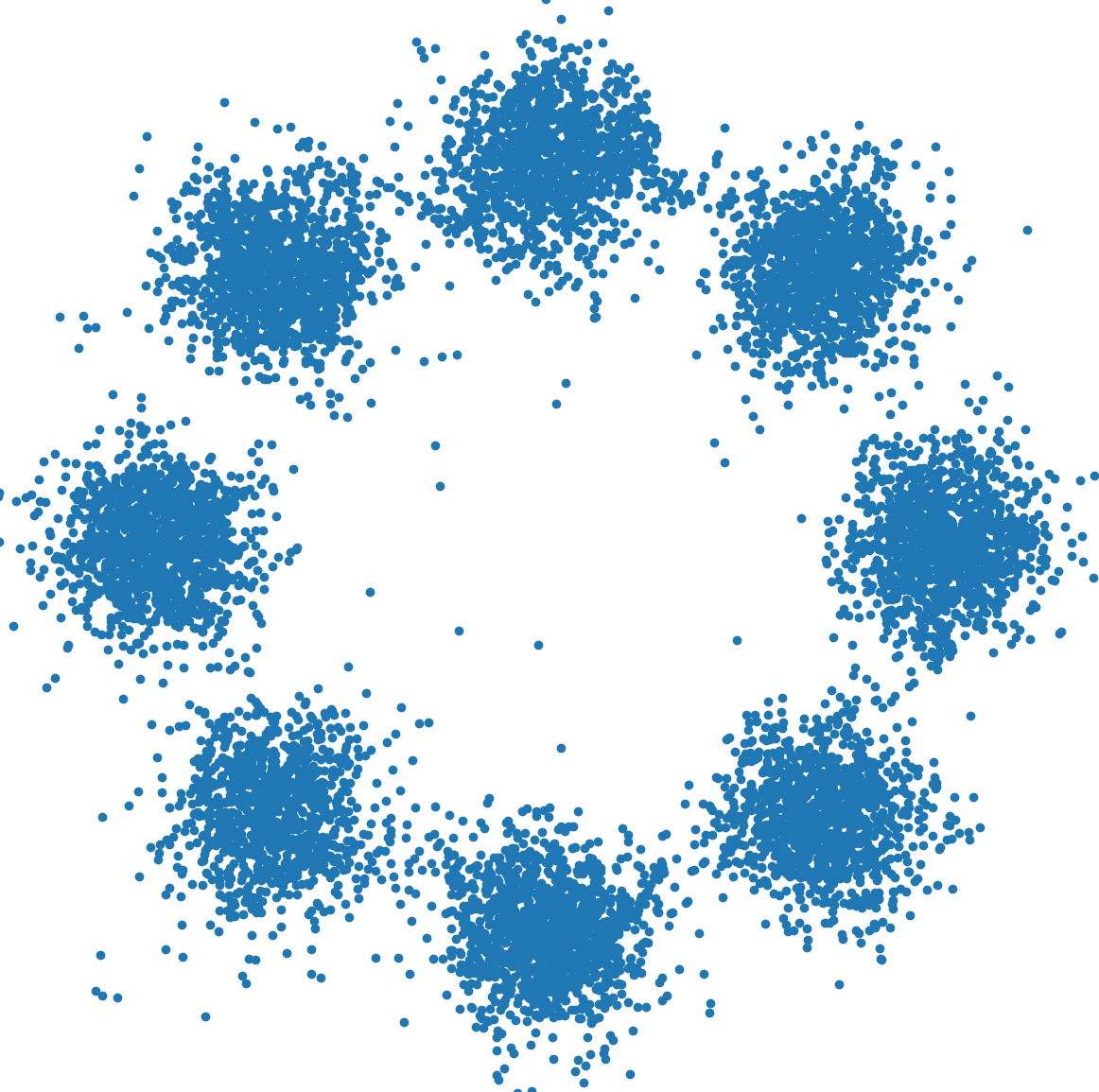} &
            \includegraphics[width=\linewidth]{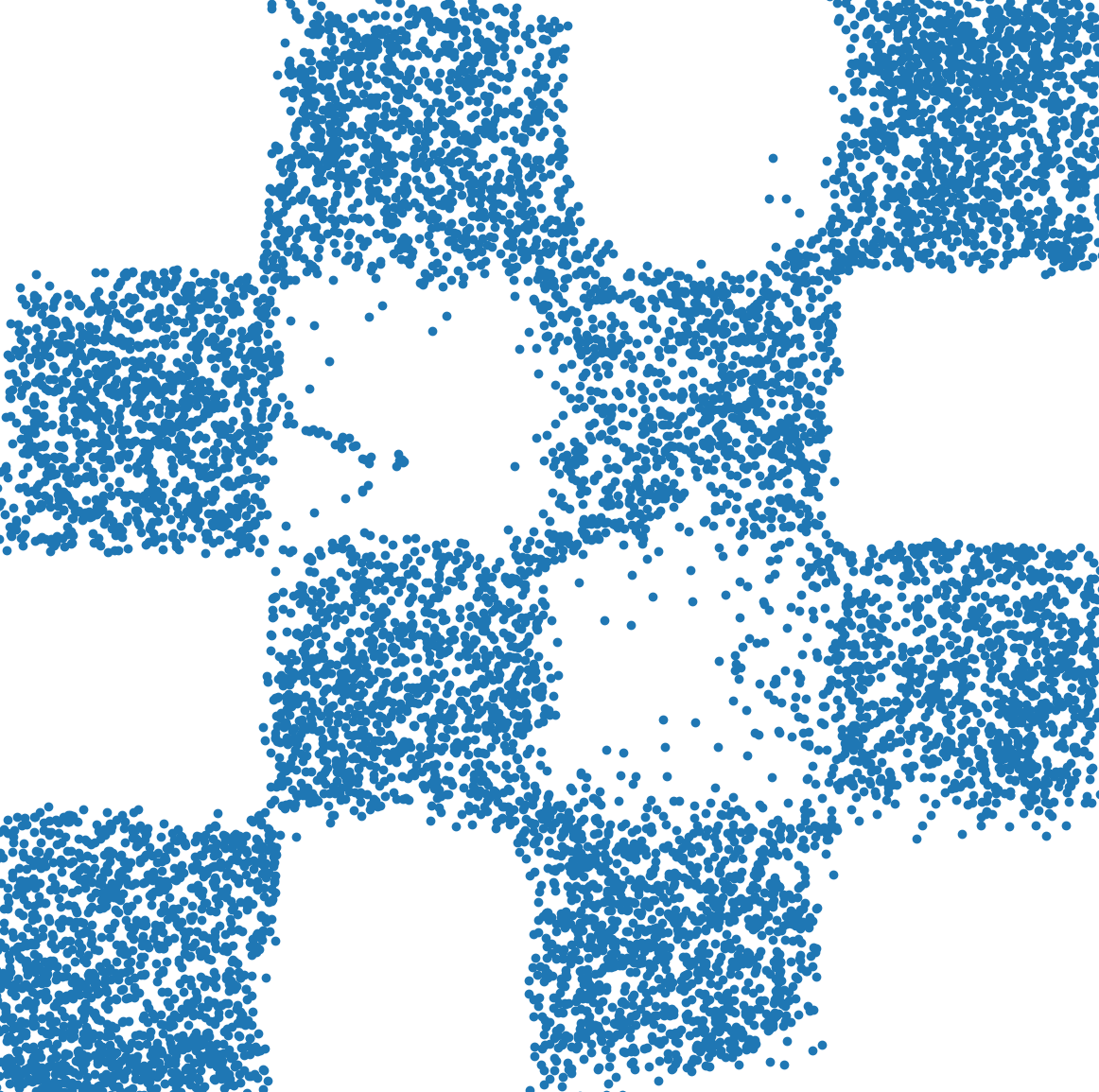} &
            \includegraphics[width=\linewidth]{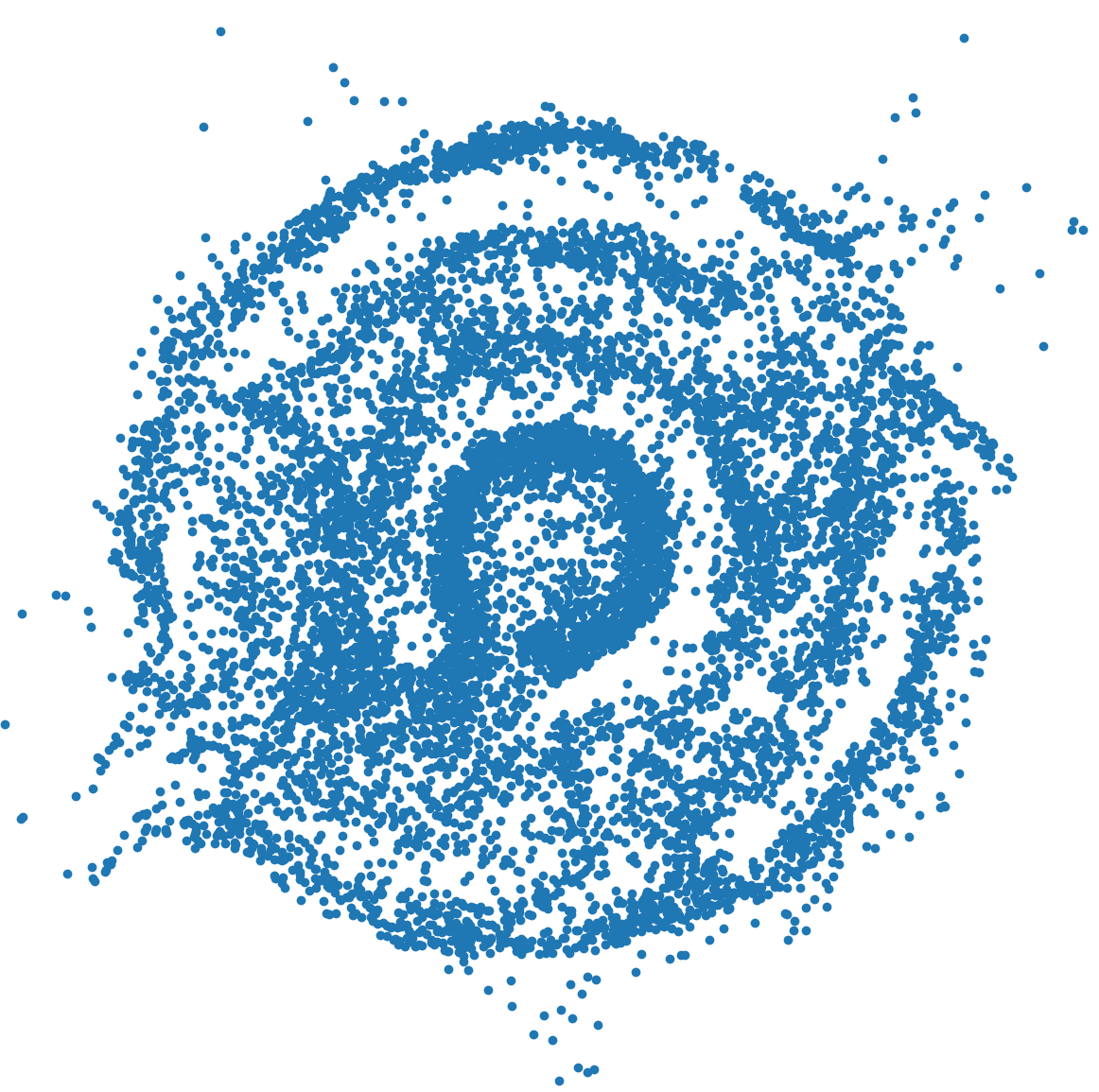}\\
            \textrm{NTK-kPF} &
            \includegraphics[width=\linewidth]{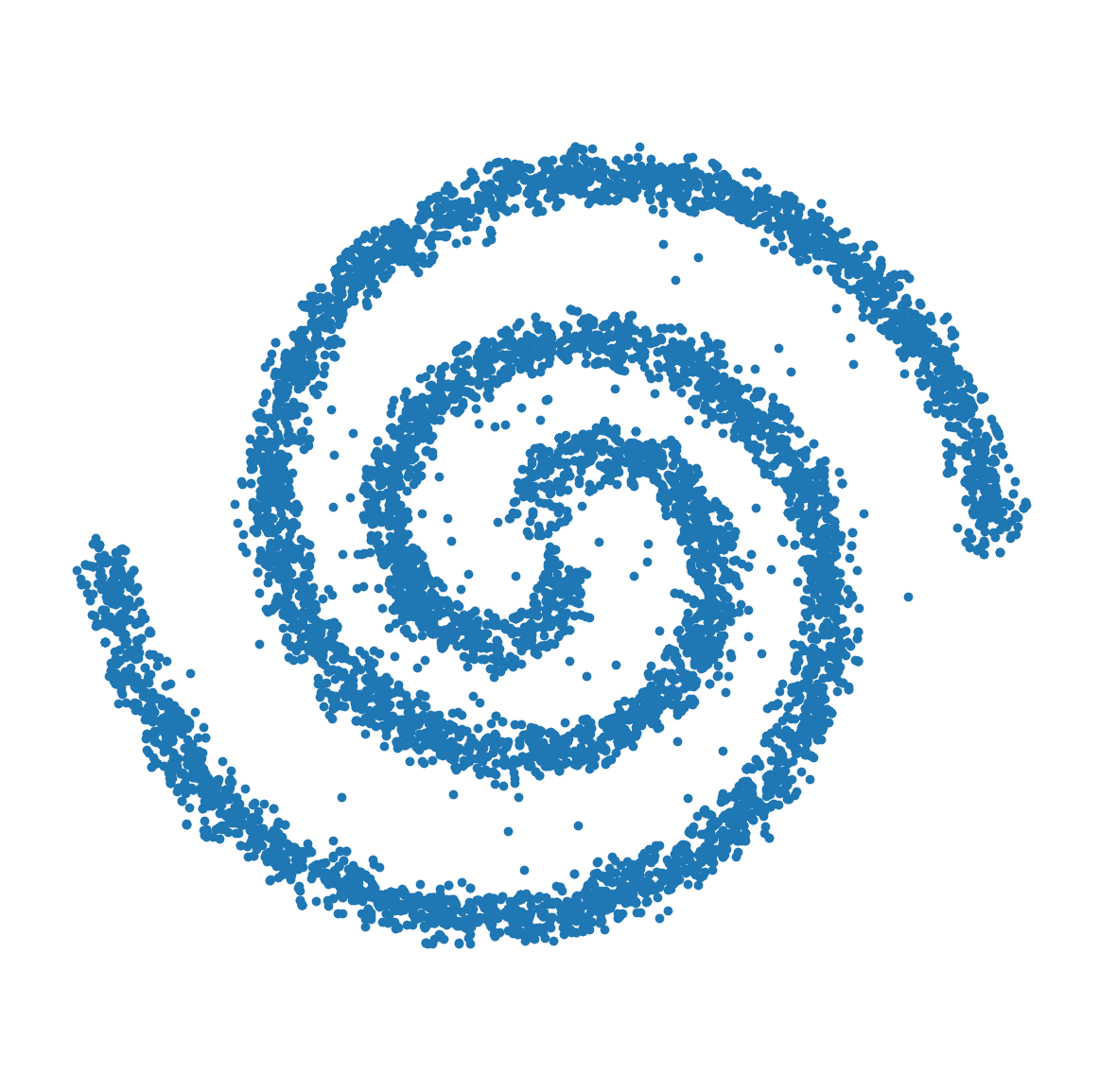} &
            \includegraphics[width=\linewidth]{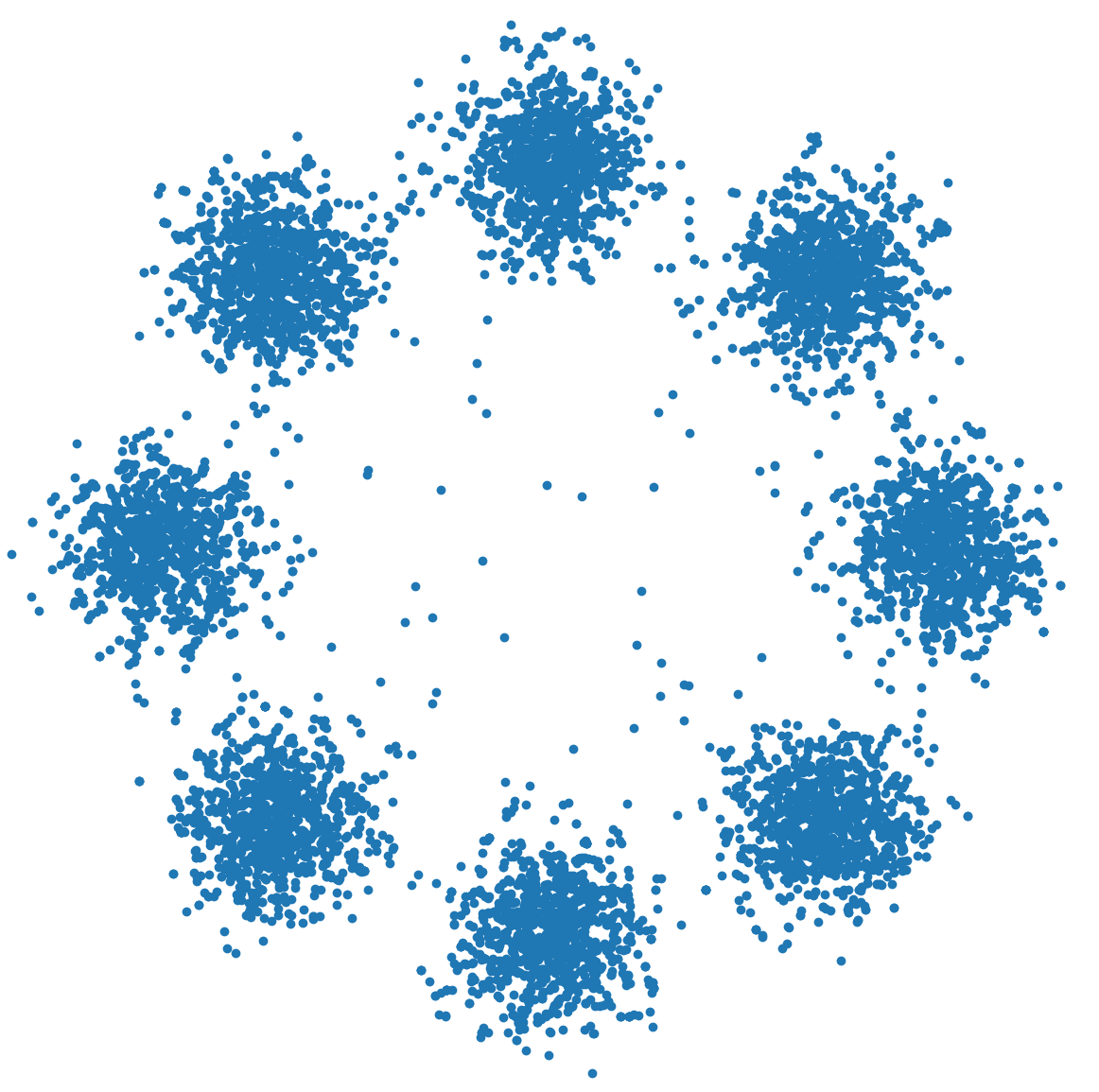} &
            \includegraphics[width=\linewidth]{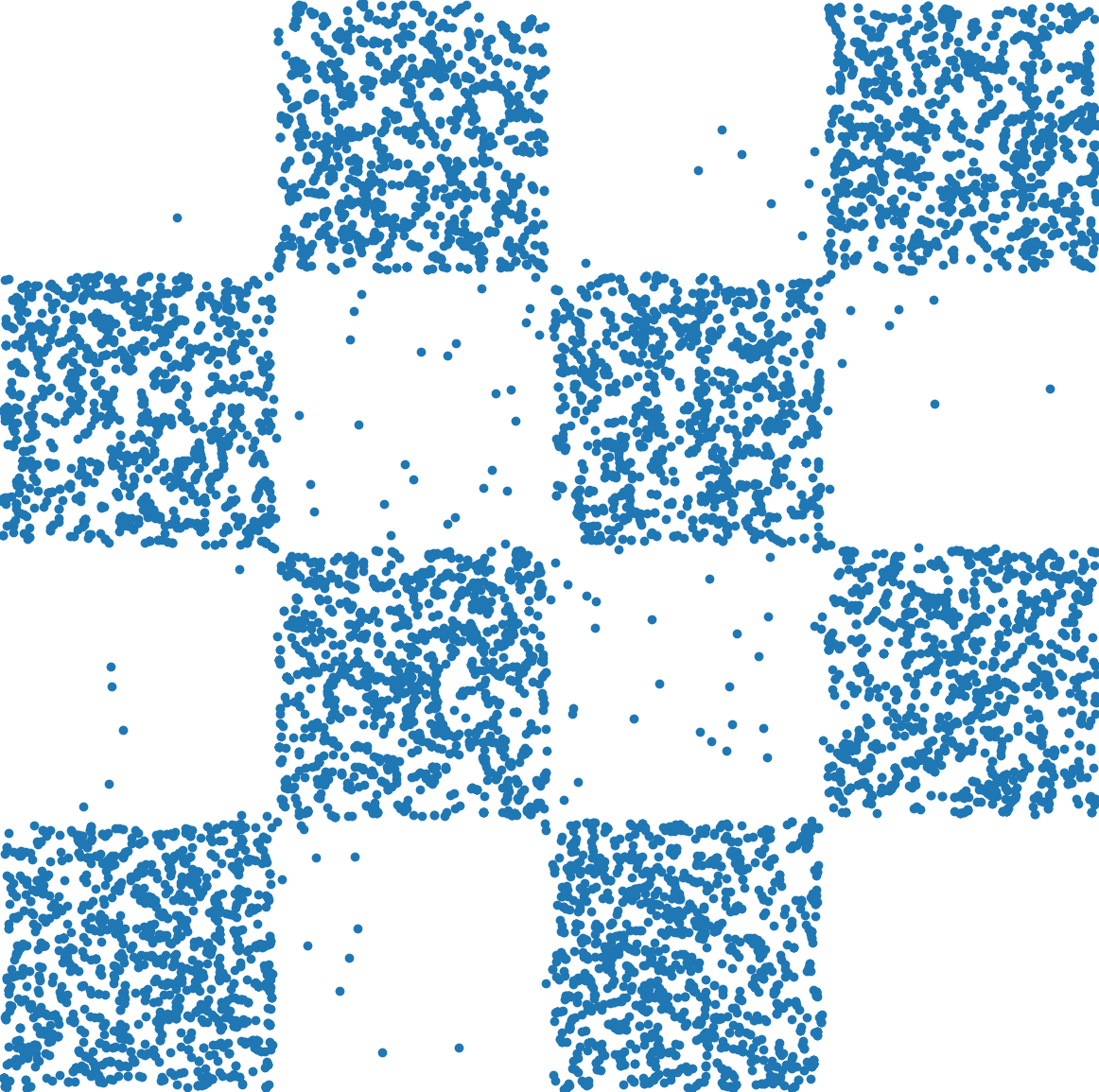} &
            \includegraphics[width=\linewidth]{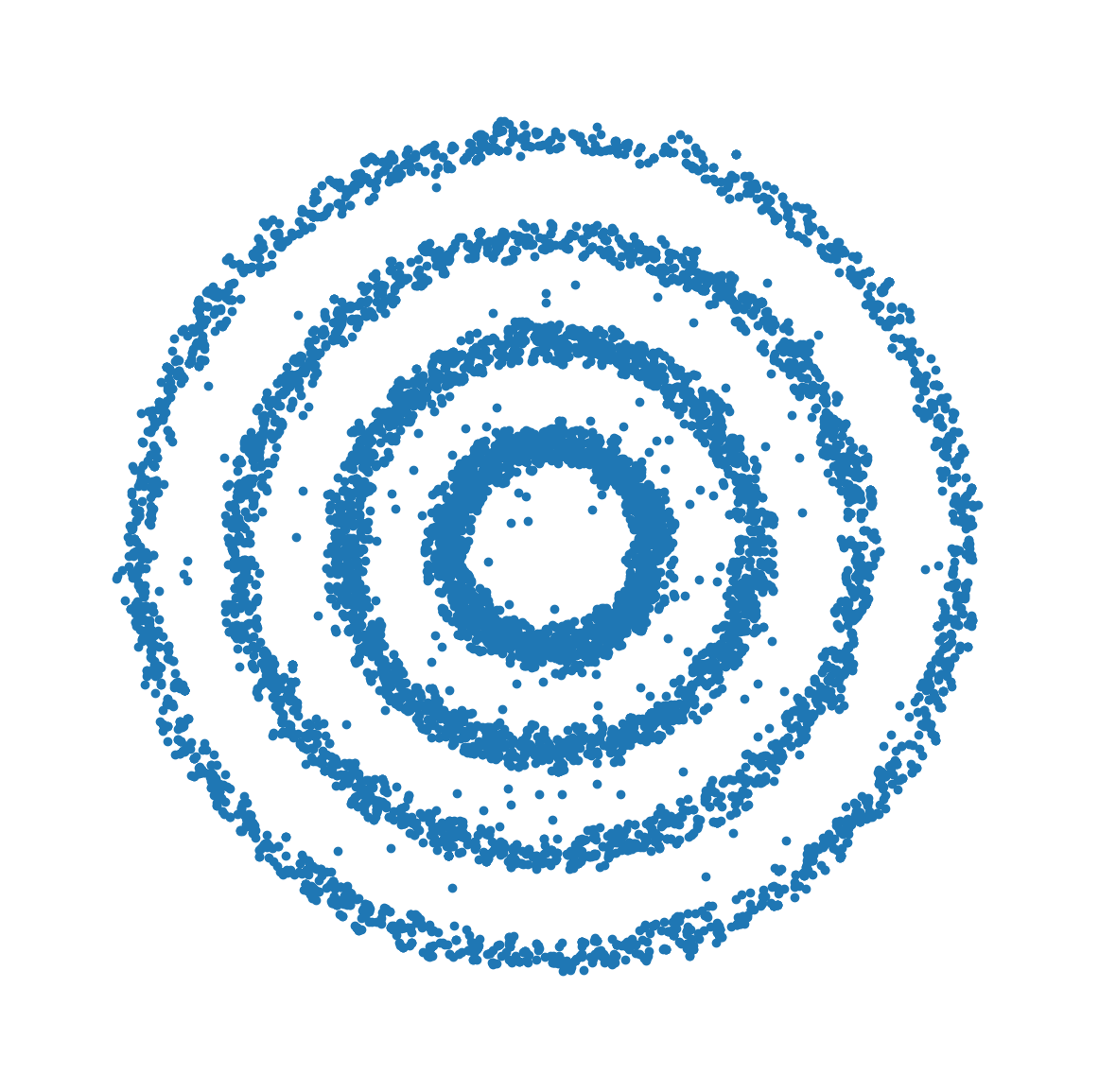}\\
        \end{tabular}
        \caption{\label{fig:samples}Sample comparisons of the distribution matching methods}
\end{figure}

\newcommand{\centered}[1]{\begin{tabular}{l} #1 \end{tabular}}
    %\end{minipage}
    %\hspace{0.02\textwidth}

\clearpage
\section{Effect of $\gamma$ on Sample Quality}

In the sampling stage, our proposed method finds the approximate preimage of the transferred kernel embeddings by taking the weighted Fr\'{e}chet mean of the top $\gamma$ neighbors among the training samples.
The choice of $\gamma$ therefore influences the quality of generation. 

From Figure \ref{fig:fid_v_gamma}, we can observe that,
in general, FID worsens as $\gamma$ increases. This observation aligns with our intuition of preserving only the local similarities represented by the kernel, and similar ideas have 
been previously used in the literature \citep{hastie2001statisticallearning,kwok2004pre}. However, significantly decreasing $\gamma$ leads 
to the undesirable result where the 
generator merely generates the training samples (in the extreme case where $\gamma = 1$, generated samples will just be reconstructions of training samples). Therefore, in our experiments, we choose $\gamma = 5$ to achieve a balance between generation quality and the distance to training samples.

\begin{figure}[h]
    \centering
    \includegraphics[width=0.6\textwidth]{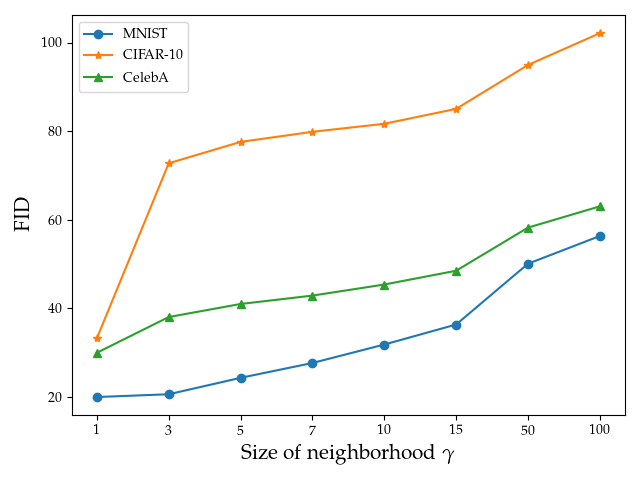}
    \caption{FID \textit{versus} $\gamma$ on a few computer vision datasets.}
    \label{fig:fid_v_gamma}
\end{figure}

\clearpage
\section{Weighted Fr\'{e}chet Mean on the Hypersphere}
\label{appdx:wfm_sphere}
While the weighted Fr\'{e}chet Mean in Euclidean space can be computed in closed-form as a weighted arithmetic mean (as in Eq. \ref{eq:wfm_euc}), on the hypersphere there is no known closed-form solution. Thus, we adopt the iterative algorithm in \citep{chakraborty2015recursive} for an approximate solution given data points $\mathbf{X} = \{\mathbf{x}_1 \dots \mathbf{x}_\gamma\}$ and weight vector $\mathbf{s}$:
\begin{align*}
    M_1 &= \mathbf{x}_1\\
    M_{i+1} &:= \cos(\|\mathbf{s}_{i+1}\mathbf{v}\|) M_i + \sin(\|\mathbf{s}_{i+1}\mathbf{v}\|)\frac{\mathbf{v}}{\|\mathbf{v}\|}
\end{align*}
where, $\mathbf{v} = \frac{\theta}{\sin(\theta)}\left(X_{i+1} - M_i \cos(\theta)\right), \theta = \arccos(M_i^tX_{i+1})$. This algorithm iterates through the data points once, yielding a complexity of only $O(\gamma d)$, where $d$ is the dimension of $\mathcal{X}$. Under the prescribed iteration, $M_{n}$ converges asymptotically to the true Fr\'{e}chet mean for finite data points. We refer the readers to \citep{chakraborty2015recursive} for further details.

\clearpage
\section{Fast approximation of Moore-Penrose inverse}

When computing the inverted kernel matrix $K_\textrm{inv}$ in Algo. \ref{alg:gen_algo}, conventional approaches typically performs SVD or Cholesky decomposition. Both procedures are hard to parallelize, and therefore, can be slow when $K$ is large. Alternatively, we can utilize an iterative procedure proposed in \citet{Razavi2014} to approximate the Moore-Penrose inverse.

\begin{align}
    Z_{1} &= K / (\Vert K \Vert_1 \Vert K \Vert_\infty)\\
    Z_{i+1} &:= Z_i (13I - KZ_i (15I - KZ_i (7I - KZ_i ))) 
\end{align}
where 
\begin{equation}
    \Vert K \Vert_1 = \max_j \sum_{i = 0}^{n} K_{ij}, \Vert K \Vert_\infty = \max_i \sum_{j = 0}^{n} K_{ij} 
\end{equation}

Since this iterative procedure mostly involves matrix multiplications, it can be efficiently parallelized and implemented on GPU. The same procedure has also seen success in approximating large self-attention matrices in language modeling \citep{xiong2021nystromformer}. For the NVAE experiment, we run this iteration for $10$ steps and use $K_\textrm{inv} = Z_{10}$.

\clearpage
\section{Nystrom Approximation of kPF}
\label{sec:nystrom}

Due to the need to store and compute a kernel matrix inverse $(K + \lambda nI)^{-1}$ or $K^\dagger$, the memory and computational complexity of kPF is at least $O(n^2)$ and $O(n^3)$, respectively. The sup-quadratic complexity hinders the use of kPF on extremely large datasets. In our experiments, we already adopted a simple subsampling strategy which randomly select 10k training samples from each dataset ($\leq$ 50k samples) to fit our hardware configuration which works 
well. But for larger datasets with potentially more modes, a larger set of subsamples must be considered, and in those cases kPF may not be 
suitable for commodity/affordable hardware. 
In order to overcome this problem, we can combine kPF with conventional kernel approximation methods such as the Nystr\"om method \citep{williams2000nystrom}.

Let $(\mathbf{X_\star}, \mathbf{Z_\star})$ be a size $v$ subset of the training set (which we refer to as the \textit{landmark points}) and $(\Psi_\star, \Phi_\star)$ be their corresponding kernel feature maps. The weighting coefficients $\mathbf{s}$ for each prior sample $z^* \sim Z$ derived in Alg. \ref{alg:gen_algo} can be approximated by
\begin{align}
    % \mathbf{s} = \Psi^\top \Psi' (\Psi'^\top \Psi')^\dagger \Psi'^\top \Psi(\Phi^\top \Phi' (\Phi'^\top \Phi')^\dagger \Phi'^\top \Phi)^\dagger
    \nonumber\mathbf{s} &= L (K + \lambda nI)^\dagger \Phi^\top k(z^\star, \cdot) \\\nonumber
    &\approx L_{\Psi_\star} W_{\Psi_\star}^\dagger L_{\Psi_\star}^\top (K_{\Phi_\star} W_{\Phi_\star}^\dagger K_{\Phi_\star}^\top + \lambda nI)^\dagger \Phi^\top k(z^\star, \cdot) \\
    &= L_{\Psi_\star} W_{\Psi_\star}^\dagger L_{\Psi_\star}^\top (\lambda n)^{-1}(I - K_{\Phi_\star}^\top(\lambda n W_{\Phi_\star}^\dagger + K_{\Phi_\star} K_{\Phi_\star}^\top)^\dagger K_{\Phi_\star}) \Phi^\top k(z^\star, \cdot)
\end{align}
where $L_{\Psi_\star} = \Psi^\top \Psi_\star \in \mathbb{R}^{n \times v}$, $W_{\Psi_\star} = \Psi_\star^\top \Psi_\star \in \mathbb{R}^{v \times v}$, $K_{\Phi_\star} = \Phi^\top \Phi_\star \in \mathbb{R}^{n \times v}$, $W_{\Phi_\star} = \Phi_\star ^\top \Phi_\star \in \mathbb{R}^{v \times v}$, and the last identity is due to applying the Woodbury formula on $(K_{\Phi_\star} W_{\Phi_\star}^\dagger K_{\Phi_\star}^\top + \lambda nI)^\dagger$. Assuming $v \ll n$, the memory complexity is reduced to $O(nv)$ and the computation complexity to $O(nv^2)$.

We empirically evaluated the Nystr\"om-approximated kPF on the CelebA experiment and present the result in Tab. \ref{tab:fid_nystrom}. It can be observed that when $v$ is sufficiently large, the performance of Nystr\"om approximated kPFs is as good as 
the ones using the full kernel matrices.

\begin{table}[h]
    \centering
    \begin{tabular}{c || c | c | c| c}
    \toprule
    \diaghead{NVNV}{$n$}{$v$} & 100 & 500 & 1000 & \thead{w/o Approximation}\\
    \midrule\midrule
         10,000 & 45.9 & 41.6 & 42.3 & 41.8\\
         \midrule
         30,000 & 46.3 & 40.5 & 42.4 & -\\
         \midrule
         50,000 & 45.2 & 44.1 & 42.0 & -\\
    \bottomrule
    \end{tabular}
    \caption{FIDs of samples generated by Nystr\"om-approximated NTK-kPF on CelebA. $n$ denotes the size of the training data subset we consider in computing kPF, while $v$ denotes the size of selected landmark points for Nystr\"om approximation. Without approximation, we cannot fit the kernel matrices onto a GPU with 11GB RAM when $n > 10,000$. It is worth noting that the approximated kPFs can perform similarly to the full kPF even with $v < 0.05n$, which indicates that Nystr\"om approximation does not sacrifice much in terms of performance while delivering significant efficiency gain.}
    \label{tab:fid_nystrom}
\end{table}

\clearpage
\section{Does kPF Memorize Training Data?}

Since in kPF, samples are generated by linearly interpolating between training samples, it is natural to wonder whether it `fools' the metrics by simply replicating training samples. For comparison, we consider an alternative scheme that generates data through direct manipulation of the training data, namely Kernel Density Estimation (KDE). 

We fit KDEs by varying noise levels $\sigma$ and compare their FIDs and nearest samples in the latent space to kPF in Fig. \ref{fig:overfitting}. We observe that, although KDE can reach very low FIDs when $\sigma$ is small, almost all new samples closely resemble some instance in the training set, which is a clear indication of memorization. In contrast, kPF can generate diverse samples that do not simply replicate the observed data.

\begin{figure}[ht]
    \begin{subfigure}[b]{0.49\linewidth}
    \centering
    \begin{tikzpicture}
    \node[inner sep=0pt] (img) at (0,0) {\includegraphics[width=\textwidth]{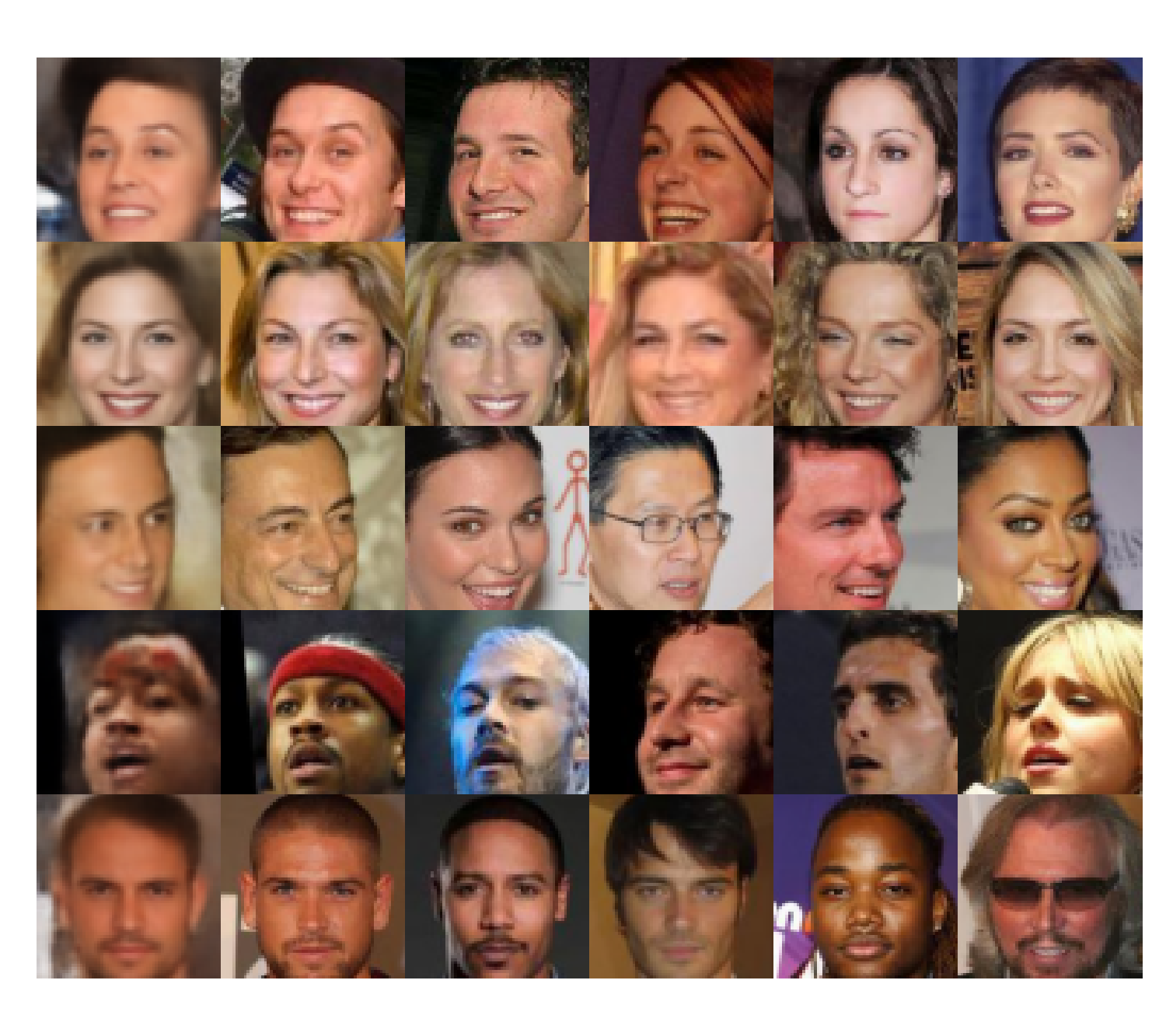}};
    \draw[draw=red, line width=2pt] (-3.2, -2.7) rectangle (-2.15, 2.7);
    \end{tikzpicture}
    \caption{KDE, $\sigma = 0.005$, FID: 30.9}
    \end{subfigure}
    \begin{subfigure}[b]{.49\linewidth}
    \centering
    \begin{tikzpicture}
    \node[inner sep=0pt] (img) at (0,0) {\includegraphics[width=\textwidth]{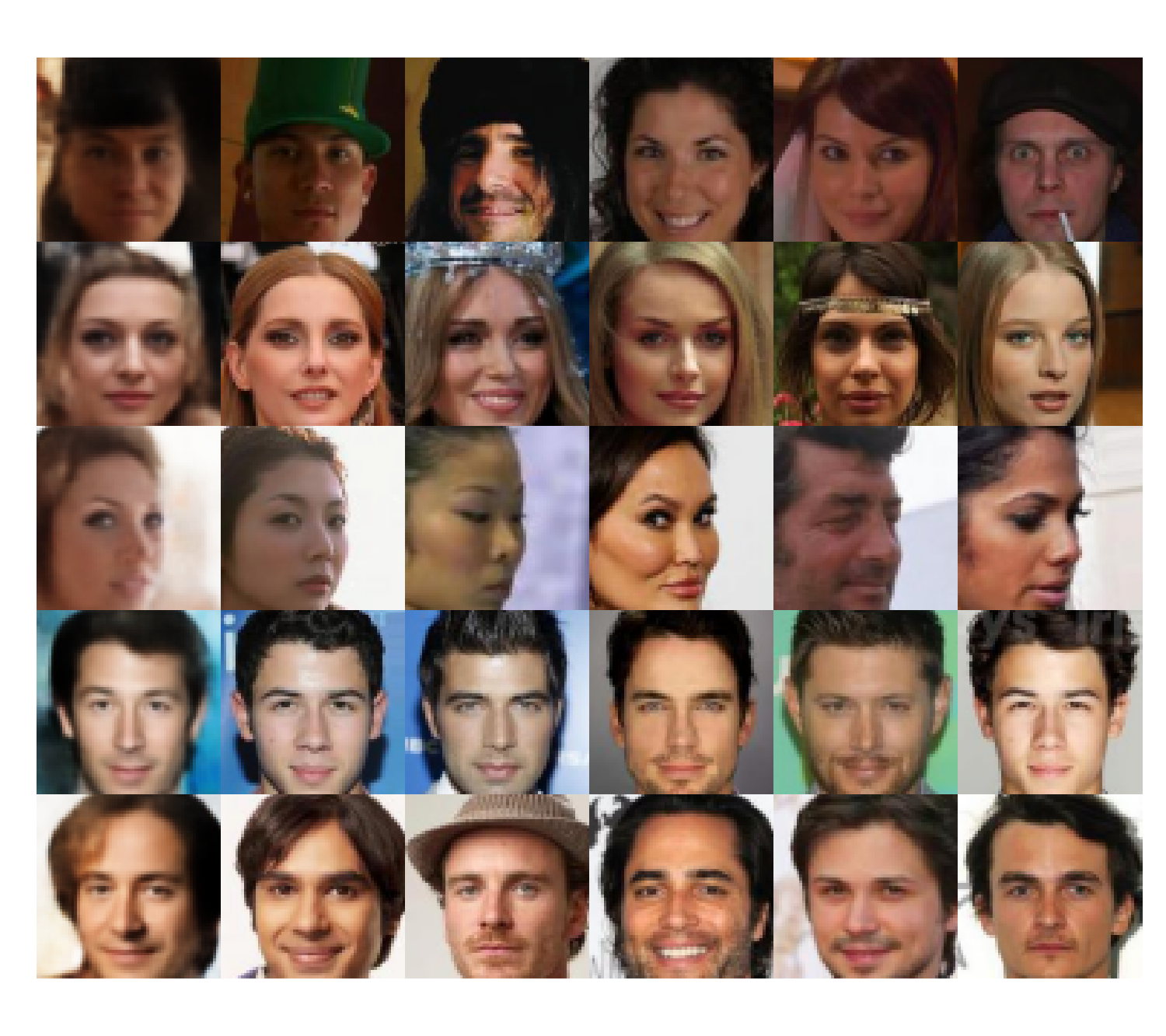}};
    \draw[draw=red, line width=2pt] (-3.2, -2.7) rectangle (-2.15, 2.7);
    \end{tikzpicture}
    \caption{KDE, $\sigma = 0.01$, FID: 36.3}
    \end{subfigure}\\
    \centering
    \begin{subfigure}[b]{.49\linewidth}
    \centering
    \begin{tikzpicture}
    \node[inner sep=0pt] (img) at (0,0) {\includegraphics[width=\textwidth]{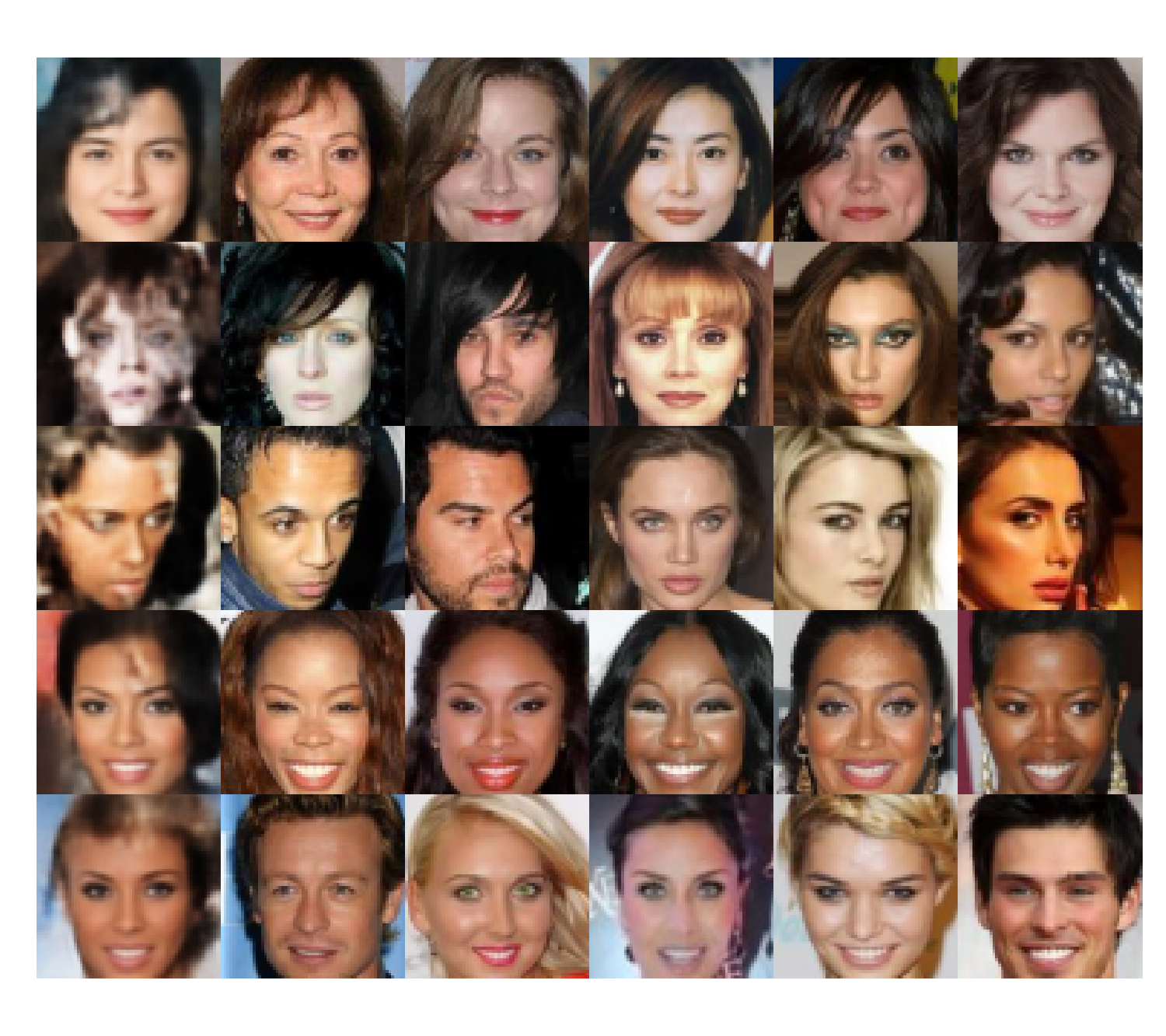}};
    \draw[draw=red, line width=2pt] (-3.2, -2.7) rectangle (-2.15, 2.7);
    \end{tikzpicture}
    \caption{KDE, $\sigma = 0.02$, FID: 48.0}
    \end{subfigure}
    \begin{subfigure}[b]{.49\linewidth}
    \centering
    \begin{tikzpicture}
    \node[inner sep=0pt] (img) at (0,0) {\includegraphics[width=\textwidth]{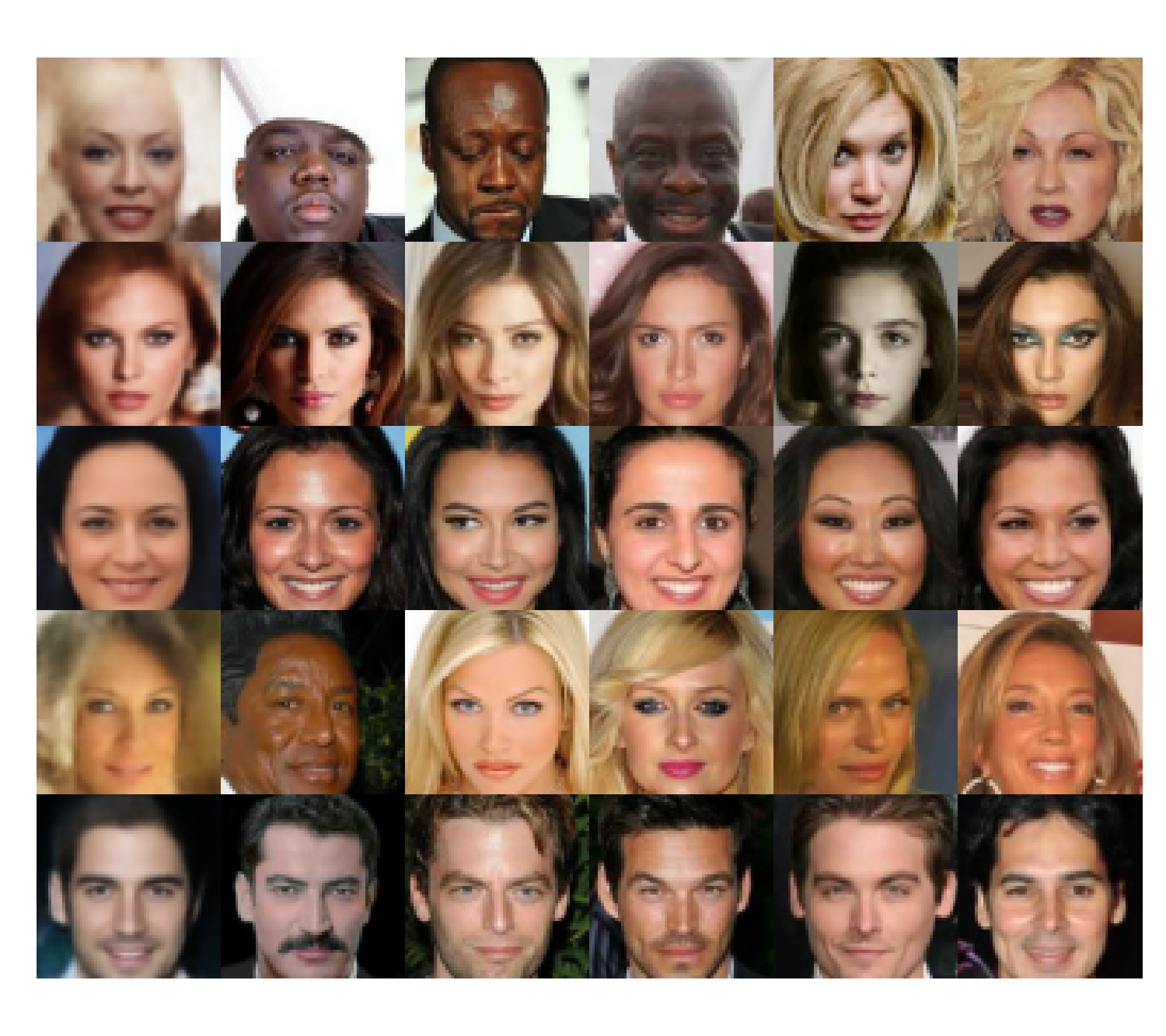}};
    \draw[draw=red, line width=2pt] (-3.2, -2.7) rectangle (-2.15, 2.7);
    \end{tikzpicture}
    \caption{kPF, FID: 41.0}
    \end{subfigure}
    \caption{Comparing KDE to kPF. Generated samples are presented in \textcolor{red}{$\msquare$}, followed by their 5 nearest neighbors in the latent space (ordered from closest to furthest)}
    \label{fig:overfitting}
\end{figure}

\clearpage
\section{Assessing Sample Diversity}

Although FID is one of the most frequently used measures for assessing sample quality of generative models, certain diversity considerations, such as mode collapse, may not be conveniently deduced from it \citep{sajjadi2018prd}. To enable explicit examination of generative models with respect to both accuracy (i.e., generating samples within the support of the data distribution) and diversity (i.e., covering the support of the data distribution as much as possible), \citet{sajjadi2018prd} proposed an approach to evaluate generative models with generalized definitions of \textit{precision} and \textit{recall} between distributions. Quality of generation can then be assessed by evaluating the PRD curve, which depicts the trade-offs between accuracy (precision) and diversity (recall). We present the PRD curves in Fig. \ref{fig:prd}. The observations align with our results in Tab. \ref{tab:fid_table} and kPF performs competitively in both accuracy and sample diversity.

\begin{figure}[h]
    \centering
    \begin{tabular}{c c c}
        \quad MNIST & \quad CIFAR-10 & \quad CelebA\\
        \includegraphics[width=0.3\textwidth]{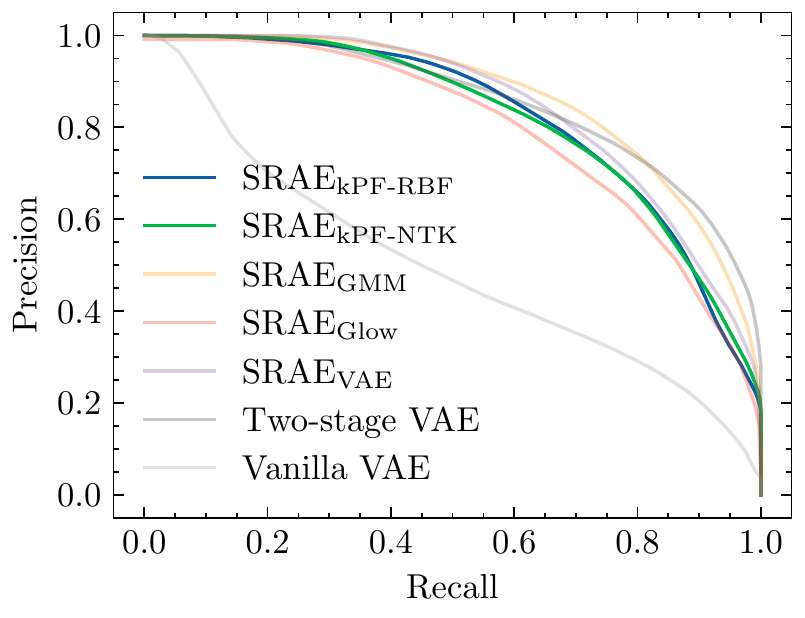}&
        \includegraphics[width=0.3\textwidth]{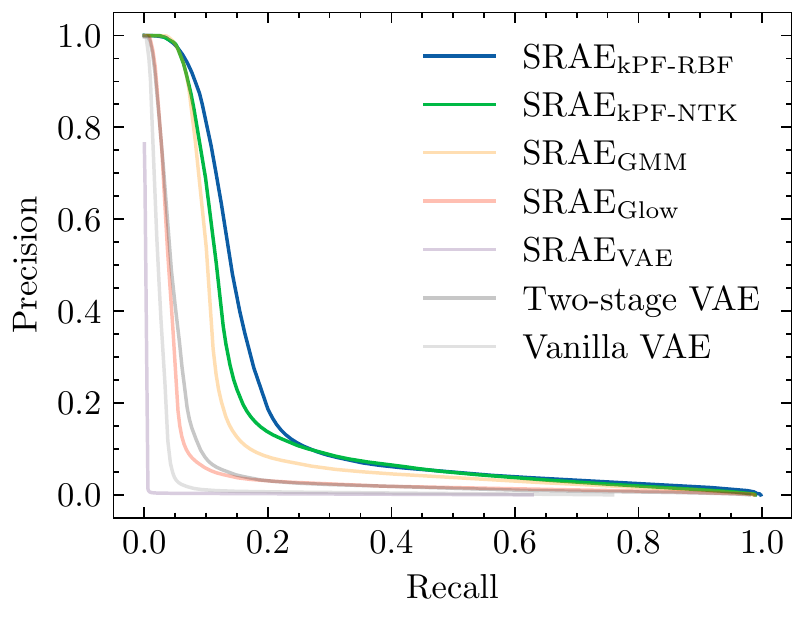}&
        \includegraphics[width=0.3\textwidth]{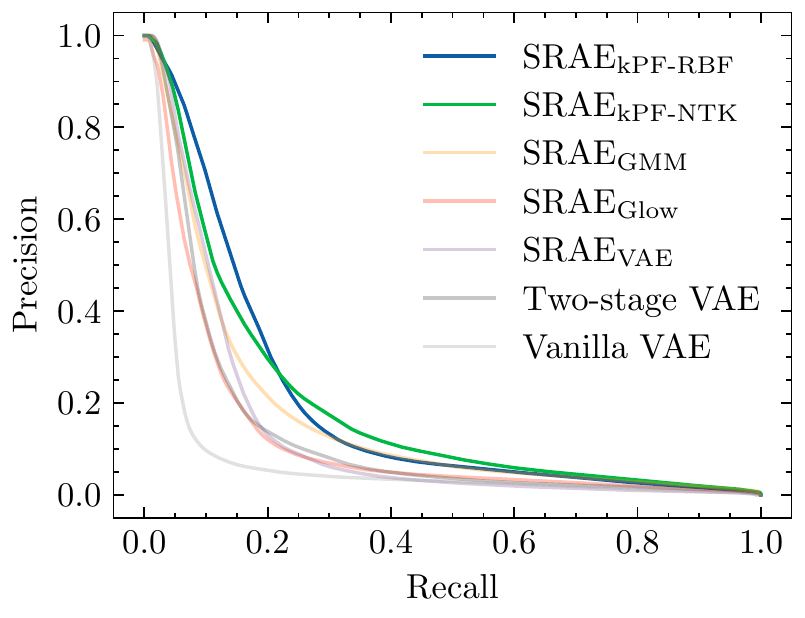}\\
    \end{tabular}
    \caption{PRD curves on all datasets. kPF is competitive to the other methods in terms of Area Under Curve (AUC)}
    \label{fig:prd}
\end{figure}

\clearpage
\section{Exploring Kernel Configurations}

To investigate the implication of kernel choices on generation quality, we tested 25 different kernel configurations for CelebA generation (results are presented in Tab. \ref{tab:kernel_config}). For RBF kernels used in the CelebA experiments of the main text, we use a bandwidth of $\sigma_{in} = {\sqrt{2 |\spc{Z}|}} / {8} \approx 0.71$ when used as input kernel and $\sigma_{out} \approx 0.34$, and we adopt the same notation here.

\begin{table}[h]
    \centering
    \begin{tabular}{ c | c | c | c | c | c}
        \toprule
         \diaghead{\theadfont Input kernel kernel}{Input \\ kernel}{Output \\ kernel}
         & \thead{RBF \\($\sigma = \sigma_{out} / 4$)} & \thead{RBF \\($\sigma = \sigma_{out} / 2$)}& \thead{RBF \\($\sigma = \sigma_{out}$)}& \thead{RBF \\($\sigma = 2\sigma_{out}$)}& \thead{RBF \\($\sigma = 4\sigma_{out}$)}\\
         \midrule\midrule
         \thead{RBF \\($\sigma = \sigma_{in} / 2$)} & 41.50 & 41.20 & 41.21 & 50.92 & 66.06 \\\midrule
         \thead{RBF \\($\sigma = \sigma_{in}$)} & 41.90 & 42.11 & 41.91 & 45.83 & 50.70 \\\midrule
         \thead{RBF \\($\sigma = 2\sigma_{in}$)} & 42.20 & 42.82 & 42.69 & 65.76 & 70.19 \\\midrule
         \thead{NTK \\ ($L=8$, $w=10,000$)} & 41.90 & 41.56 & 41.73 & \textbf{37.86} & 37.89  \\\midrule
         \thead{Arccos \\ ($L=1$, $\textrm{deg}=1$)} & 41.71 & 42.03 & 42.22 & 52.83 & 63.13\\
         \bottomrule
         
    \end{tabular}
    \caption{FID table for different kernel configurations.}
    \label{tab:kernel_config}
\end{table}

It can be seen that kernel configurations and parameters indeed has a non-trivial impact on the generation quality, with NTK-kPF being the most robust to the choice of parameters. This aligns with our previous observations and offers some support for using NTK as an input kernel despite the additional compute cost.

\clearpage

\section{Experimental Details}
\label{sec:experimental_details}

In this section, we provide the detailed specifications for all of our experiments. We have also provided our code in the supplemental material.
\subsection{Density Estimation on Toy Densities}

We generated $10000$ samples from each of the toy densities to learn the kPF operator. The input kernel $k$ is a ReLU-activated NTK corresponding to a fully-connected network with depth $L = 4$ and width $w = 10000$ at each layer, and the output kernel $l$ is a Gaussian kernel. Unless specified otherwise, we always uses a Gaussian kernel as the output kernel for the remainder of this appendix. The bandwidth of the output kernel was adjusted separately for density estimation and sampling for the purpose of demonstration. For comparison, we also fit/estimate a 10-component GMM and a Glow model with 50 coupling layers, where each of them were trained until convergence.

\subsection{Image Generation with Computer Vision Datasets}

To generate results in Tab. \ref{tab:fid_table} and Tab. \ref{tab:fid_table_limited}, we first trained an autoencoder for each dataset following the model setup in \citep{Ghosh2020From}, which uses a modified Wasserstein autoencoder \citep{tolstikhin2018wasserstein} architecture with batch normalization. Additionally, we applied spectral normalization on both the encoder and the decoder, following \citep{Ghosh2020From}, to obtain a regularized autoencoder. The latent representations were projected onto a hypersphere before decoding to image space. We trained the models on two NVIDIA GTX 1080TI GPUs. A detailed model specification is provided below in Table \ref{tab:vision_architecture}.

We used an NTK with $L = 8$ and $w = 10000$ as the input kernel $k$ (i.e. the embedding kernel of $p_Z$) for NTK-kPF, and a Gaussian kernel with bandwidth $\sigma_{in} = {\sqrt{2 |\spc{Z}|}} / {8}$ for RBF-kPF. The bandwidth for the output Gaussian kernel is selected by grid search over $\{2^{-i} * \sigma_{data}| i \in [8]\}$, where $\sigma_{data}$ is the empirical data standard deviation, based on cross-validation of a degree 3 polynomial kernel MMD between the sampled and the ground-truth latent points. Further, to mitigate the deterioration of performance of kernel methods in a high-dimensional setting due to the curse of dimensionality \citep{Evangelista2006Taming}, in practice, we model $\spc{Z}$ as a space with fewer dimensions than the input space $\spc{X}$. As a rule of the thumb, we choose $\spc{Z}$ such that $|\spc{Z}| = |\spc{X}| / 4$.

To generate images from kPF learned on NVAE latent space, we used the pre-trained checkpoints provided in \citep{vahdat2020NVAE} to obtain the latent embeddings for 2000 FFHQ images. We then construct the kPF from the concatenated latent space of the lowest resolution ($8 \times 8$). During sampling, prior samples at those resolutions are replace by the kPF samples, while for other resolutions samples remain generated from inferred Gaussian distributions. The batchnorm statistics were readjusted for $500$ iterations following \citep{vahdat2020NVAE}. We use rbf kernels as input and output kernels, with bandwidths $\sigma_k$, $\sigma_l$ chosen by the \textit{median heuristic} ($\sim100$ for input and $\sim70$ for output in our experiments).

\newcommand{\conv}[2]{\textrm{Conv}_{#1}^{#2}}
\newcommand{\convt}[2]{\textrm{ConvT}_{#1}^{#2}}
\newcommand{\resblock}[2]{\textrm{ResBlock}_{#1} \times #2}
\renewcommand{\arraystretch}{1.2}
\begin{figure}[h]
    \begin{minipage}{0.55\textwidth}
    \centering
    {\footnotesize    
    \begin{tabular}{c|c|c|c}
            \toprule
          & MNIST & CIFAR-10 & CelebA\\
          \midrule\midrule
          \multirow{4}{*}{Encoder}
                 & $\conv{128}{4\times4}$ & $\conv{128}{4\times4}$ & $\conv{128}{5\times5}$ \\
                 & $\conv{256}{4\times4}$ & $\conv{256}{4\times4}$ & $\conv{256}{5\times5}$ \\
                 & $\conv{512}{4\times4}$ & $\conv{512}{4\times4}$ & $\conv{512}{5\times5}$ \\
                 & $\conv{1024}{4\times4}$ & $\conv{1024}{4\times4}$ & $\conv{1024}{5\times5}$ \\
                 \midrule
          \multirow{4}{*}{Decoder}
                 & $\convt{512}{4\times4}$ & $\convt{512}{4\times4}$ & $\convt{512}{5\times5}$ \\
                 & $\convt{256}{4\times4}$ & $\convt{256}{4\times4}$ & $\convt{256}{5\times5}$ \\
                 & $\convt{1}{4\times4}$ & $\convt{3}{4\times4}$ & $\convt{128}{5\times5}$ \\
                 &                       &                       & $\convt{3}{5\times5}$ \\
                %  &                   &                   & \convt{32}{5\times5} \\\cline{2-4}
                %  &\multicolumn{3}{c|}{5x5 conv, stride 1} \\
                 \bottomrule
    \end{tabular}
    }
    \captionof{table}{Model architecture for computer vision experiments. Subscript denotes the number of output channels and superscript denotes the window size of the convolution kernel. Batch normalization and activation is applied between each pair of convolution layers}
    \label{tab:vision_architecture}
    \end{minipage}
    \hspace{0.02\textwidth}
    \begin{minipage}{0.4\textwidth}
    \centering
    {\footnotesize
    \begin{tabular}{p{0.25\linewidth}>{\centering\arraybackslash}p{0.65\linewidth}}
         \toprule
          \multirow{4}{*}[-0.5em]{Encoder}& 5x5 conv, stride 4 \\\cline{2-2}
                 & $\resblock{64}{2}$ \\
                 & $\resblock{64}{2}$ \\
                 & $\resblock{128}{2}$ \\
                 & $\resblock{256}{2}$ \\
                 \midrule
          \multirow{4}{*}[-0.5em]{Decoder}
                 & $\resblock{256}{2}$ \\
                 & $\resblock{128}{2}$ \\
                 & $\resblock{64}{2}$ \\
                 & $\resblock{64}{2}$ \\\cline{2-2}
                 &5x5 conv, stride 4 \\
                 \bottomrule
    \end{tabular}
    }
    \captionof{table}{Model architecture for experiments for image generation on brain imaging dataset. Subscript denotes the number of output channels. Upsampling and downsampling are performed using strided convolutions.}
    \label{tab:brain_imaging_architecture}
    \end{minipage}
\end{figure}

% \newcommand{\resblock}[2]{\textrm{ResBlock}_{#1} \times #2}
% \begin{wraptable}{r}{0.5\textwidth}
%     \centering
%     \begin{tabular}{p{0.1\textwidth}>{\centering\arraybackslash}p{0.3\textwidth}}
%          \toprule
%           \multirow{4}{*}[-0.5em]{Encoder}& 5x5 conv, stride 4 \\\cline{2-2}
%                  & $\resblock{64}{2}$ \\
%                  & $\resblock{64}{2}$ \\
%                  & $\resblock{128}{2}$ \\
%                  & $\resblock{256}{2}$ \\
%                  \midrule
%           \multirow{4}{*}[-0.5em]{Decoder}
%                  & $\resblock{256}{2}$ \\
%                  & $\resblock{128}{2}$ \\
%                  & $\resblock{64}{2}$ \\
%                  & $\resblock{64}{2}$ \\\cline{2-2}
%                  &5x5 conv, stride 4 \\
%                  \bottomrule
%     \end{tabular}
%     \caption{Model architecture for experiments for image generation on brain imaging dataset. Subscript denotes the number of output channels. Upsampling and downsampling are performed using strided convolutions.}
%     \label{tab:my_label}
%     \vspace*{-1.5em}
% \end{wraptable}

\subsection{Image Generation for Brain Images}

For the high-resolution brain imaging dataset, we used a custom version of ResNet \citep{he2016resnet} with 3D convolutions. The detailed architecture is shown in Fig. \ref{tab:brain_imaging_architecture}. Due to the large size of the data, we trained the model on 4 NVIDIA Tesla V100 GPUs. 

\paragraph{Mandatory ADNI statement regarding data use.} Data used in preparation of this article were obtained from the Alzheimer’s Disease Neuroimaging Initiative
(ADNI) database (adni.loni.usc.edu). As such, the investigators within the ADNI contributed to the design
and implementation of ADNI and/or provided data but did not participate in analysis or writing of this report.
A complete listing of ADNI investigators can be found in the \href{http://adni.loni.usc.edu/wp-content/uploads/how_to_apply/ADNI_Acknowledgement_List.pdf}{\text{ADNI Acknowledgement List}}.

The T1 MR brain dataset we utilize consists of images from $184$ subjects diagnosed with Alzheimers's disease and $292$ healthy controls/ normal subjects. Images were first coregistered to a MNI template and segmented to preserve only the white matter and grey matter. Then, all images were resliced and resized to $160 \times 196 \times 160$ and rescaled to the range of $[-1, 1]$. Voxel-based morphometry (VBM) was used to obtain the $p$-value map of data and generated images.

%\clearpage
% \clearpage

% \begin{figure}[!ht]
%     \centering
%     \scalebox{1}{    \begin{tabular}{c c c}
%             \includegraphics[width=0.32\textwidth]{imgs/mnist_gmm_generations.png} &
%             \includegraphics[width=0.32\textwidth]{images/mnist_glow_generations.png} &
%             \includegraphics[width=0.32\textwidth]{images/mnist_kpf_generations.png}\\
%             \includegraphics[width=0.32\textwidth]{images/cifar_gmm_generations.png} &
%             \includegraphics[width=0.32\textwidth]{images/cifar_glow_generations.png} &
%             \includegraphics[width=0.32\textwidth]{images/cifar_kpf_generations.png}
%             GMM & Glow & NTK-kPF\\
%         \end{tabular}
%     }
  
% \end{figure}

\newpage
\section{More Samples}

In this section we present additional uncurated set of samples on MNIST, CIFAR-10, CelebA based on pre-trained SRAE and FFHQ based on NVAE. From the figures, it can be seen that kPF produces consistent and diverse samples, often better in quality than the alternatives.
\vfill
\begin{figure}[h]
    \centering
    \includegraphics[width=0.495\textwidth]{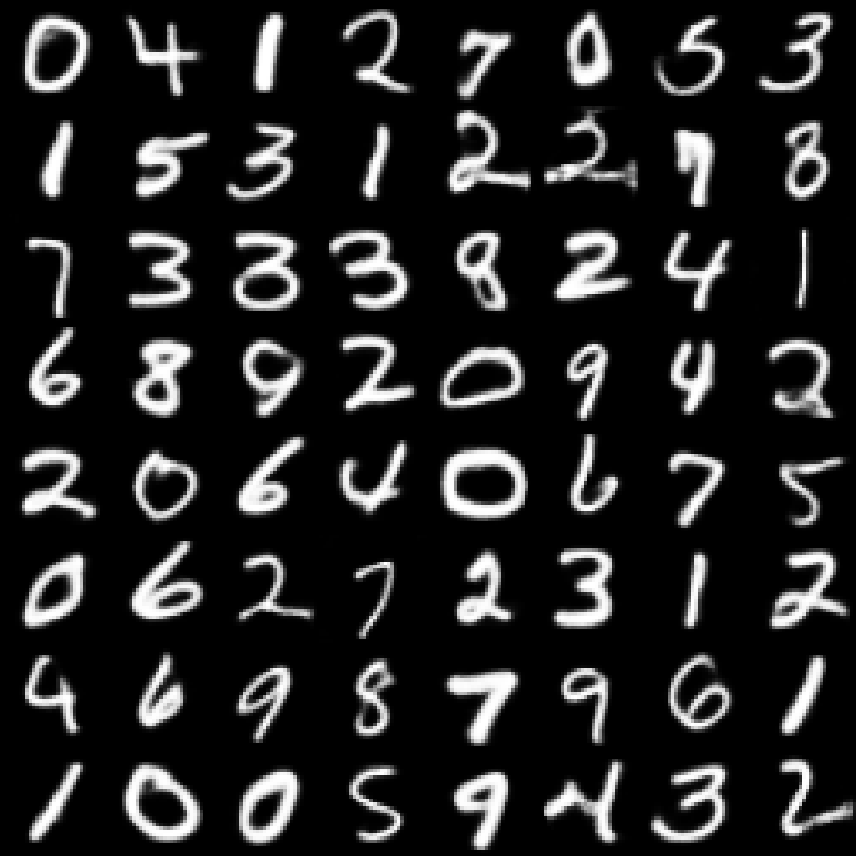}
    \includegraphics[width=0.495\textwidth]{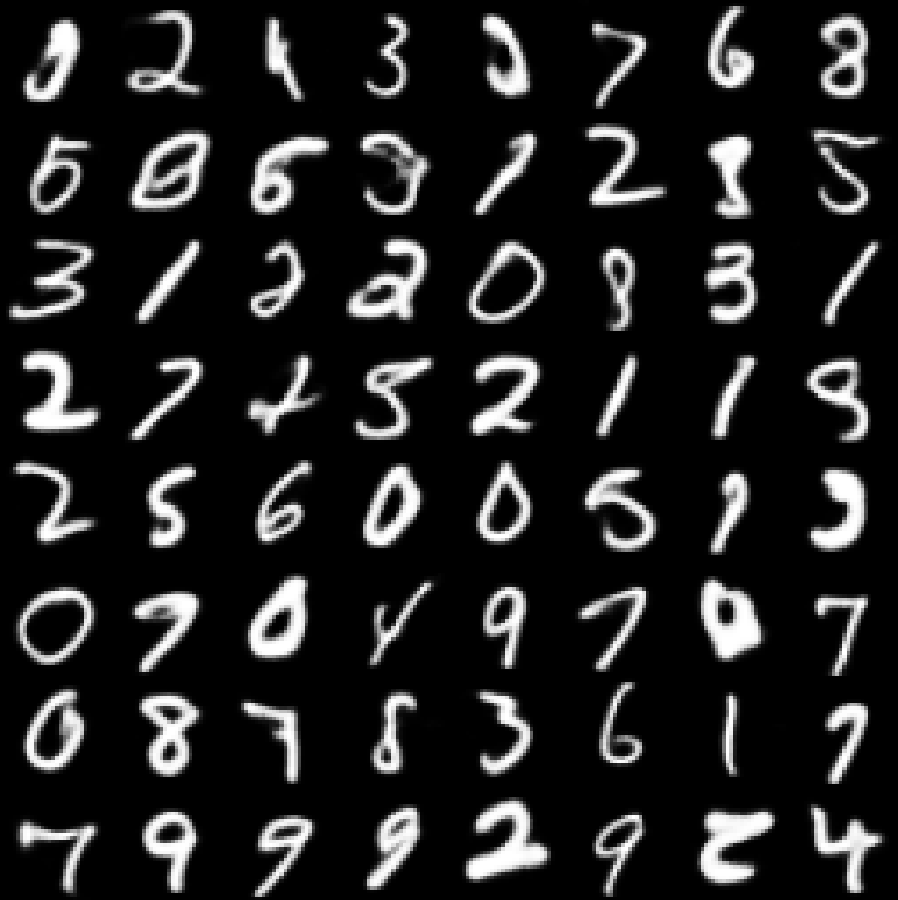}\\
    \includegraphics[width=0.495\textwidth]{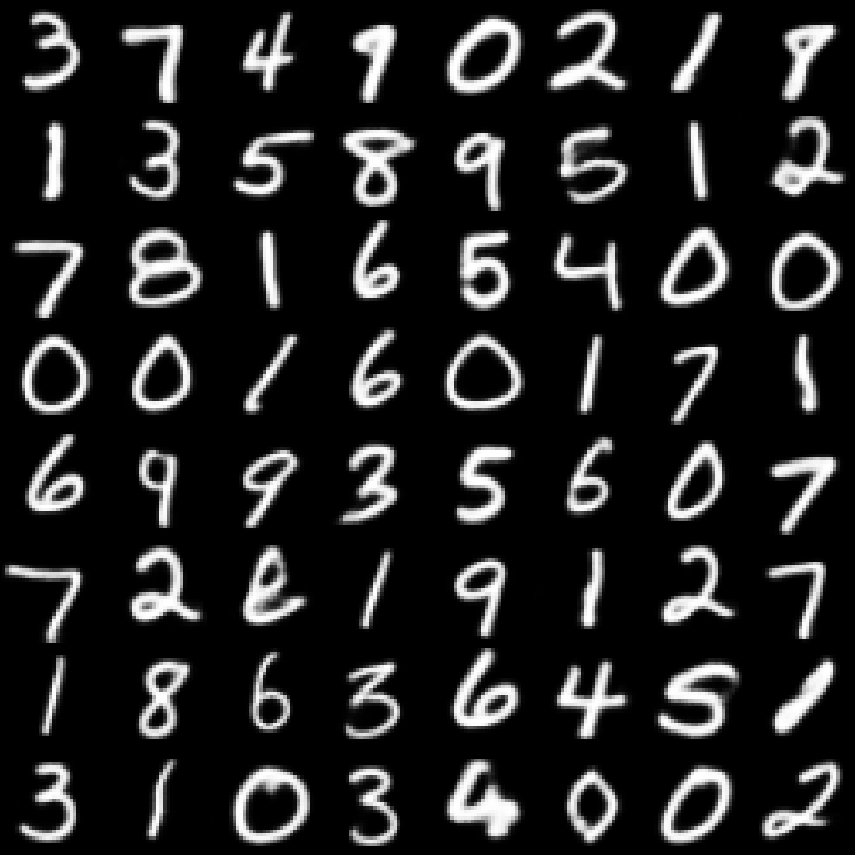}\\
    \caption{MNIST results from $\textrm{SRAE}_\textrm{GMM}$ (top left), $\textrm{SRAE}_{Glow}$ (top right) and our $\textrm{SRAE}_\textrm{NTK-kPF}$.}
    \label{fig:mnist-more-samples}
\end{figure}
\vfill

\newpage
\vfill
\begin{figure}
    \centering
    \includegraphics[width=0.495\textwidth]{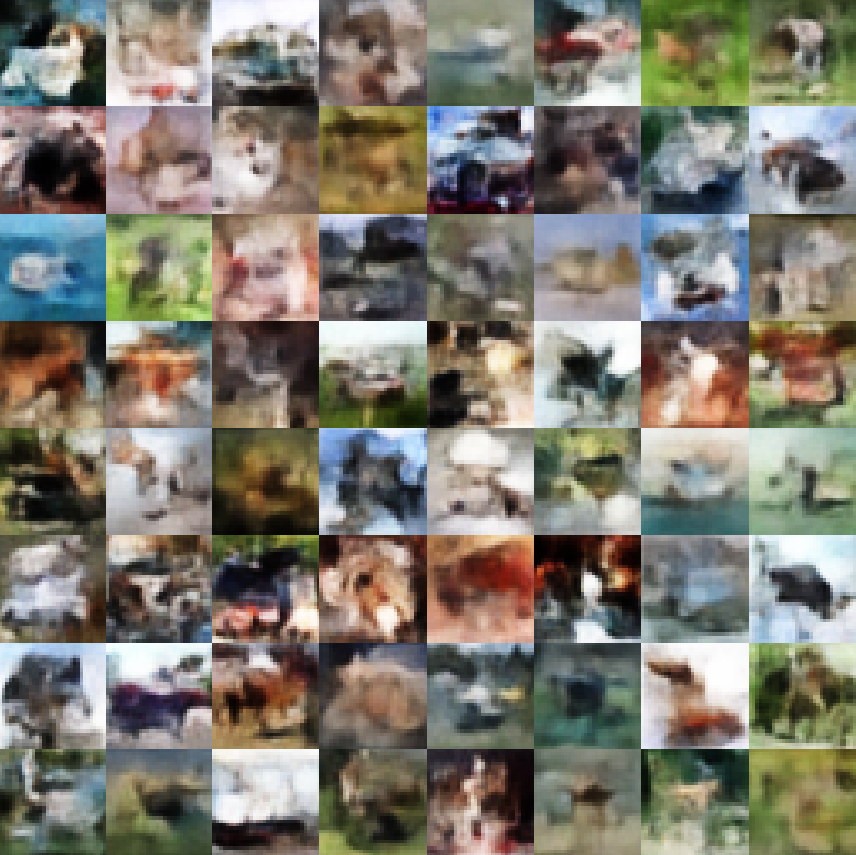}
    \includegraphics[width=0.495\textwidth]{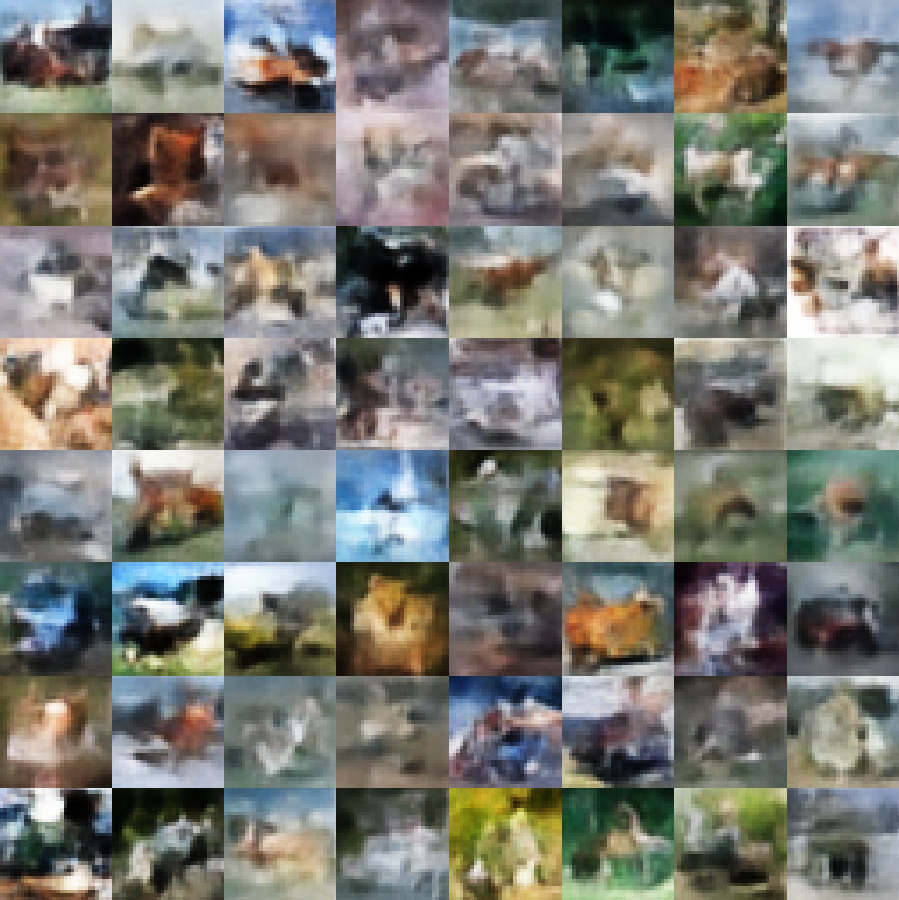}\\
    \includegraphics[width=0.495\textwidth]{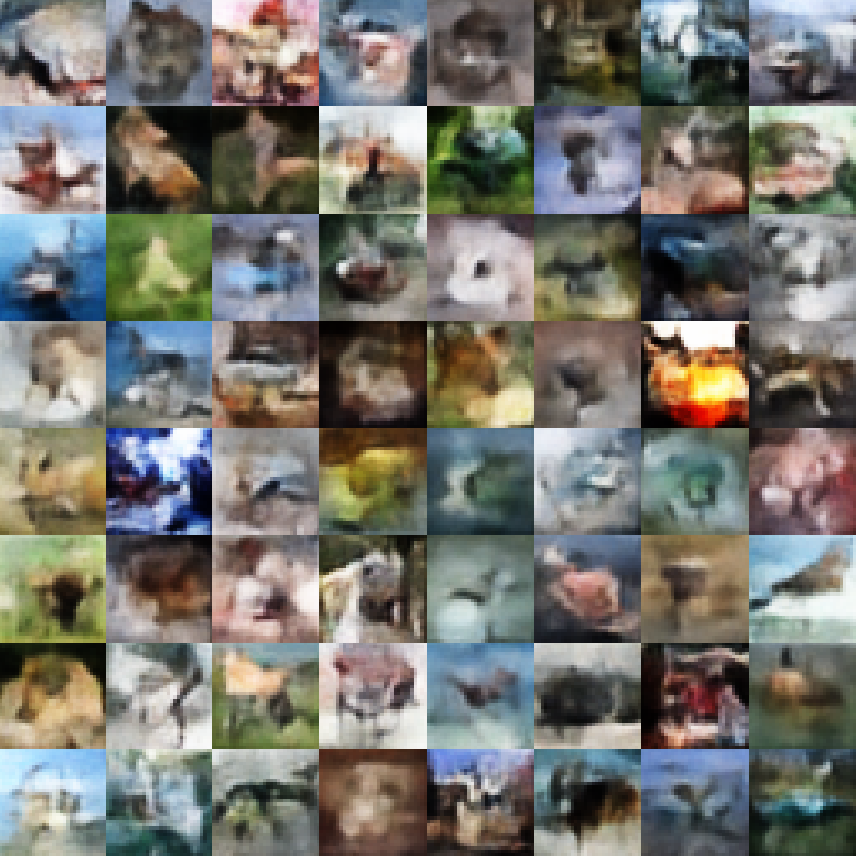}\\
    \caption{CIFAR results from $\textrm{SRAE}_\textrm{GMM}$ (top left), $\textrm{SRAE}_{Glow}$ (top right) and our $\textrm{SRAE}_\textrm{NTK-kPF}$.}
    \label{fig:cifar-more-samples}
\end{figure}
\vfill

\newpage
\vfill
\begin{figure}
    \centering
    \includegraphics[width=0.495\textwidth]{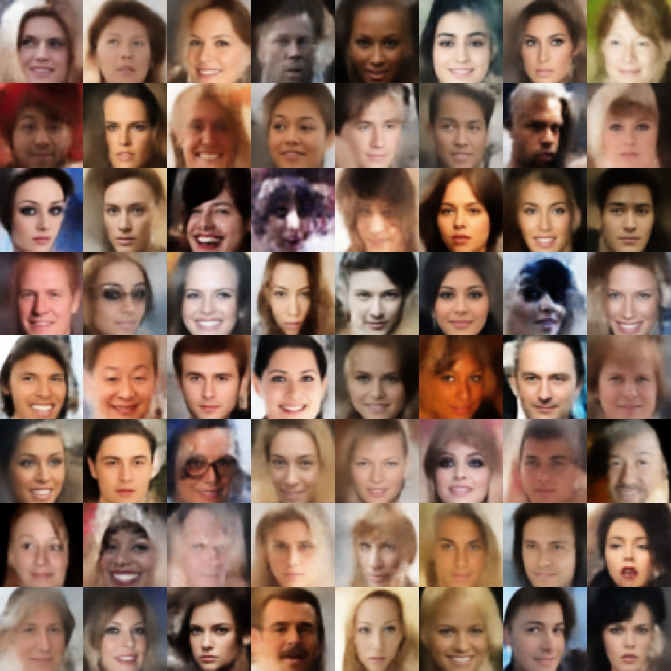}
    \includegraphics[width=0.495\textwidth]{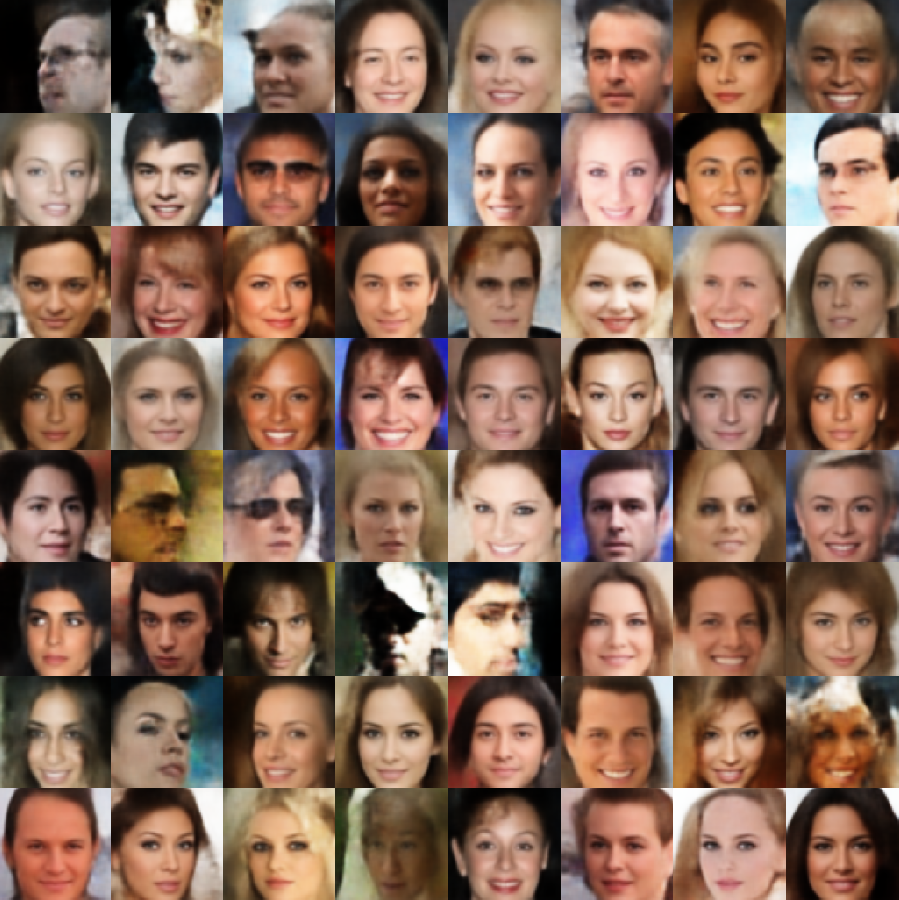}\\
    \includegraphics[width=0.495\textwidth]{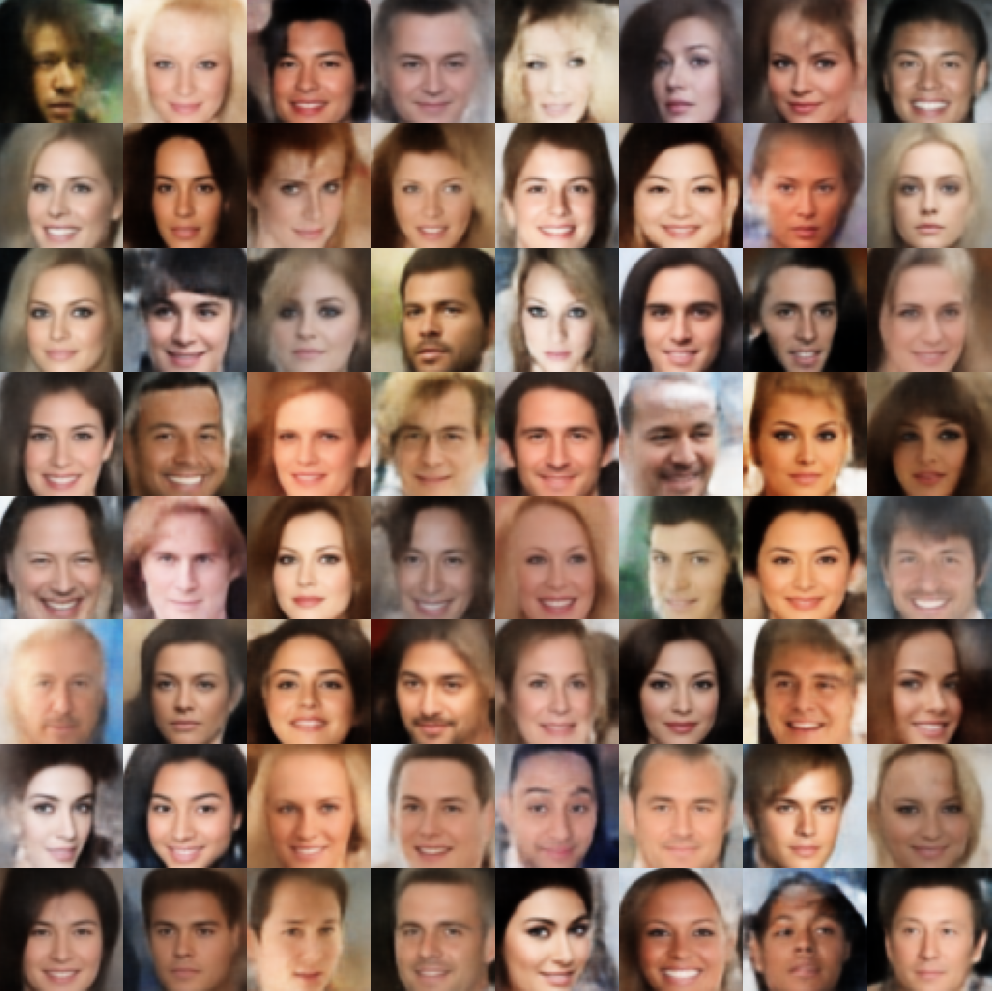}\\
    \caption{CelebA results from $\textrm{SRAE}_\textrm{GMM}$ (top left), $\textrm{SRAE}_{Glow}$ (top right) and our $\textrm{SRAE}_\textrm{NTK-kPF}$ (bottom).}
    \label{fig:celeba-more-samples}
\end{figure}
\vfill

\begin{figure}
    \centering
    \includegraphics[width = 0.95\textwidth]{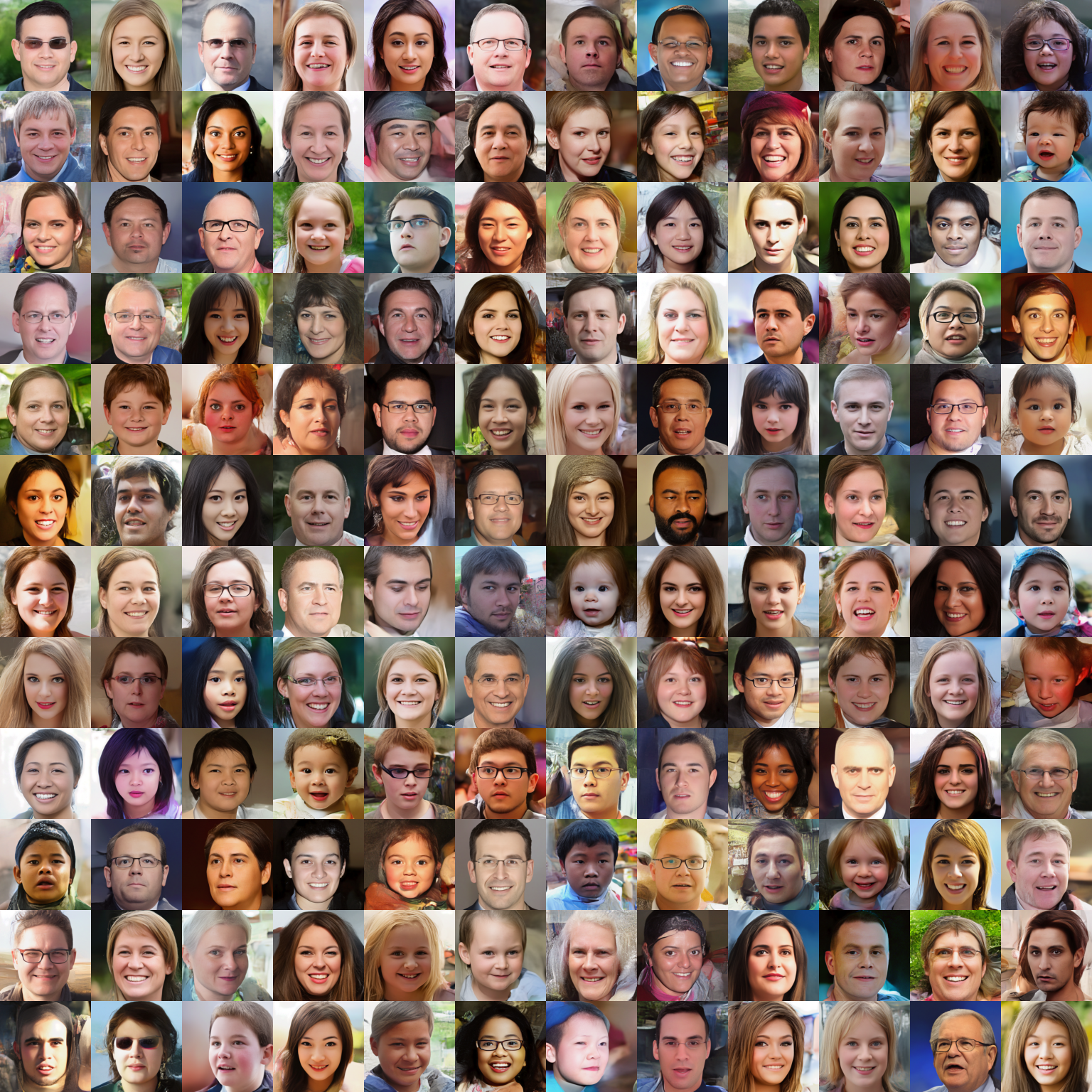}
    \caption{Additional samples from kPF+NVAE pre-trained on FFHQ}
    \label{fig:nvae-more-samples}
\end{figure}

\clearpage